\documentclass{article}

\usepackage[in]{fullpage}
\usepackage{natbib}
\usepackage{mathpazo}
\usepackage[table]{xcolor}
\usepackage[utf8]{inputenc} %
\usepackage[T1]{fontenc}    %
\usepackage{hyperref}       %
\usepackage{url}            %
\usepackage{booktabs}       %
\usepackage{amsfonts}       %
\usepackage{nicefrac}       %
\usepackage{microtype}      %
\usepackage{lipsum}         %
\usepackage{todonotes}
\usepackage{subcaption}
\usepackage{caption}

\usepackage{colortbl}

\usepackage{tcolorbox}
\usepackage{xcolor}
\usepackage{framed}
\colorlet{shadecolor}{orange!15}

\usepackage{url}            %
\usepackage{booktabs}       %
\usepackage{multirow}    
\usepackage{amsfonts}       %
\usepackage{nicefrac}       %
\usepackage{microtype}      %
\usepackage{natbib}
\usepackage{enumerate}
\usepackage{hhline}
\usepackage{makecell}
\usepackage{pifont}

\usepackage{graphicx} %
\usepackage{caption}
\usepackage{subcaption}
\usepackage{amsmath}
\usepackage{amsthm}
\usepackage{amssymb}
\usepackage{tikz}
\usepackage{xcolor}
\usetikzlibrary{arrows}

\allowdisplaybreaks

\usepackage{mathrsfs}

\usepackage{algorithm}
\usepackage{algorithmic}
\usepackage{hyperref}
\usepackage{bm}

\allowdisplaybreaks

\newtheorem{thm}{Theorem}
\newtheorem{lem}{Lemma}
\newtheorem{cor}{Corollary}
\newtheorem{prop}{Proposition}
\newtheorem{asmp}{Assumption}

\newtheorem{rem}{Remark}

\usepackage[capitalize,noabbrev]{cleveref}
\crefname{thm}{Theorem}{Theorems}
\crefname{lem}{Lemma}{Lemmas}
\crefname{cor}{Corollary}{Corollaries}
\crefname{prop}{Proposition}{Propositions}
\crefname{asmp}{Assumption}{Assumptions}
\crefname{defn}{Definition}{Definitions}
\crefname{oracle}{Oracle}{Oracles}
\crefname{fact}{Fact}{Facts}
\crefname{conj}{Conjecture}{Conjectures}
\crefname{rem}{Remark}{Remarks}
\crefname{example}{Example}{Examples}
\crefname{condition}{Condition}{Conditions}
\crefname{exercise}{Exercise}{Exercises}
\crefname{algorithm}{Algorithm}{Algorithms}
\crefname{table}{Table}{Tables}
\crefname{figure}{Figure}{Figures}
\crefname{section}{Section}{Sections}
\crefname{subsection}{Section}{Sections}
\crefname{appendix}{Appendix}{Appendices}
\crefname{message}{Message}{Messages}

\definecolor{red}{rgb}{1, 0, 0}

\definecolor{green}{rgb}{0, 1, 0}

\definecolor{blue}{rgb}{0, 0, 1}

\definecolor{orange}{rgb}{1, 0.4, 0.0}

\input{packages/math_commands}

\usepackage{pifont}

\usepackage{soul}

\usepackage{wrapfig}

\AtBeginDocument{\setlength\abovedisplayskip{5pt}}
\AtBeginDocument{\setlength\belowdisplayskip{5pt}}

\hypersetup{
    colorlinks=true,
    citecolor=blue,
    linkcolor=blue,
}

\usepackage{listings} %

\definecolor{codegreen}{rgb}{0,0.6,0}
\definecolor{codegray}{rgb}{0.5,0.5,0.5}
\definecolor{codepurple}{rgb}{0.58,0,0.82}
\definecolor{codeblue}{rgb}{0,0,1}
\definecolor{backcolour}{rgb}{0.95,0.95,0.92}
\definecolor{key-color}{rgb}{0.8, 0.47, 0.196}

\lstdefinestyle{mystyle}{
    backgroundcolor=\color{backcolour},
    commentstyle=\color{codegreen},
    numberstyle=\tiny\color{codegray},
    stringstyle=\color{codepurple},
    basicstyle=\ttfamily\footnotesize,
    breakatwhitespace=false,
    breaklines=true,
    captionpos=b,
    keepspaces=true,
    numbers=left,
    numbersep=5pt,
    showspaces=false,
    showstringspaces=false,
    showtabs=false,
    tabsize=2,
    language=Python,
    emph={lm},
    emphstyle={\color{blue}},
    classoffset=1, %
    otherkeywords={sum},
    morekeywords={rm, mean},
    keywordstyle=\color{codegreen},
    classoffset=0,
}
\title{PC Layer: Polynomial Weight Preconditioning for Improving LLM Pre-Training}

\makeatletter
\def\@fnsymbol#1{\ensuremath{\ifcase#1\or *\or \dagger\or \ddagger\or
   \mathsection\or \sharp\or \Diamond\or \mathparagraph\or \|\or
   \or \ddagger\ddagger \else\@ctrerr\fi}}
\makeatother

\usepackage{authblk}

\author[1]{Senmiao Wang\thanks{Equal contribution.}}  %
\author[2]{Tiantian Fang{$^*$}}
\author[1]{Haoran Zhang{$^*$}}
\author[1,4]{Yushun Zhang}
\author[1]{Kunxiang Zhao}
\author[3]{Alex Schwing} %
\author[1,4]{Ruoyu Sun\thanks{Corresponding author.}} %
\affil[1]{The Chinese University of Hong Kong, Shenzhen, China}
\affil[2]{Google LLC, Mountain View, CA, United States}
\affil[3]{University of Illinois at Urbana-Champaign, Urbana, IL, United States}
\affil[4]{Shenzhen International Center for Industrial and Applied Mathematics,  Shenzhen Research Institute of Big Data}
\affil[ ]{\texttt{\{senmiaowang1, haoranzhang2, yushunzhang, kunxiangzhao\}@link.cuhk.edu.cn}}
\affil[ ]{\texttt{\{tf6, aschwing\}@illinois.edu}}
\affil[ ]{\texttt{sunruoyu@cuhk.edu.cn}}

\date{}

\begin{document}

\maketitle
\begin{abstract}

We propose a preconditioning (PC) layer — a weight parameterization via polynomial preconditioner that ensures stable weight conditioning throughout LLM training.
The PC module reshapes the singular-value spectrum of weight matrices via low-degree polynomial preconditioning.
After training, the preconditioned weights can be merged back into the original architecture, incurring no inference overhead.
We demonstrate the advantage of the proposed PC layer
over standard transformers in Llama-1B pre-training, for both the AdamW and Muon optimizers.
Theoretically, we justify this spectrum-control principle by proving that uniformly bounding each layer's singular values ensures geometric convergence of gradient descent to global minima, for certain deep linear networks.
Our code is available at \url{https://github.com/Empath-aln/PC-layer}.

\end{abstract}

\section{Introduction}
\label{sec:intro}

Training modern neural networks, especially at scale, relies on normalization techniques.
These methods stabilize optimization by inserting layers that normalize intermediate quantities—such as features or parameters—within the network \citep{huang2023normalization}.
Representative examples include batch normalization \citep{ioffe2015batch}, layer normalization \citep{ba2016layer}, instance normalization \citep{ulyanov2016instance}, group normalization \citep{wu2018group}, etc.
In the era of large language models (LLMs), normalization methods remain critical. For instance, RMSNorm \citep{zhang2019root}, a variant of LayerNorm \citep{ba2016layer}, has become a de facto standard in most prevalent LLM architectures, and  Query–Key Normalization (QK-Norm) \citep{henry2020query, loshchilov2024ngpt} has been widely used to control attention logits and stabilize training.

RMSNorm and QK-Norm are designed to control the layer outputs of neural networks. Another family of normalization methods acts directly in \textit{weight space}: classical examples include weight normalization (WN) \citep{salimans2016weight} and spectral normalization (SN) \citep{miyato2018spectral}.
While such weight control methods have not become standard components of mainstream Transformer-based LLM pre-training, recent works have revived this line by explicitly normalizing or constraining weights, including row/column-wise normalizations on weights \citep{loshchilov2024ngpt, franke2025learning, fu2025nemotron}, hyperball/hypersphere constraints on weights
\citep{wen2025hyperball, ren2026rethinking}, spectral constraints
on weights \citep{newhouse2025training, xie2026controlled}, etc. These
developments suggest that weight geometry is becoming an increasingly
relevant design axis for stable LLM training.
This motivates two questions:

\begin{center}
\textit{
(i) Is there a theoretical principle relating weight properties
to the performance of the algorithm?}

\textit{
(ii) If so, can such a principle inform the design of weight-based control techniques?}
\end{center}

\begin{figure}[t]
    \centering
    \includegraphics[width=\linewidth]{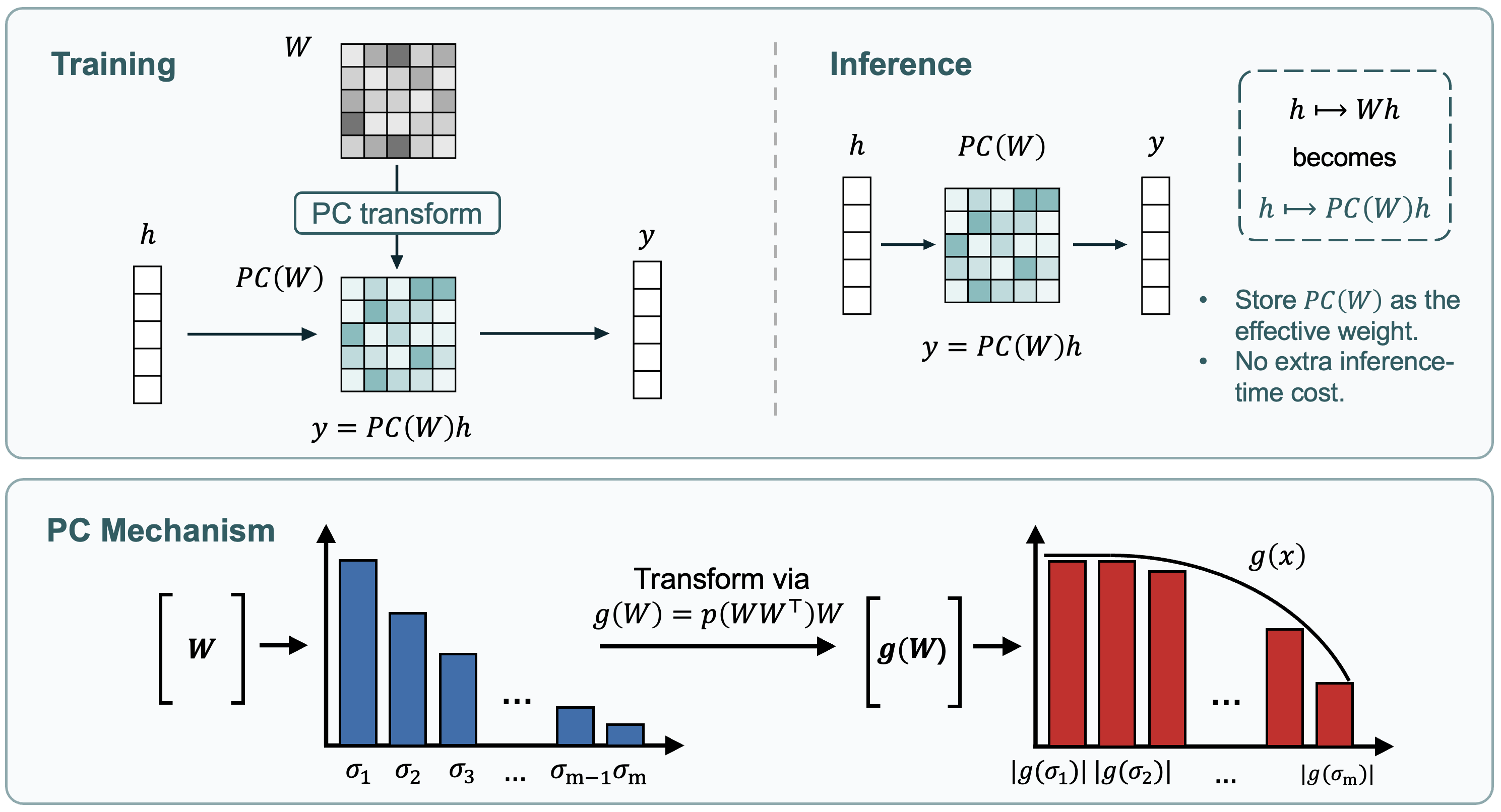}
    \caption{Illustration of the PC layer: a low-degree matrix polynomial reshapes the singular-value spectrum of a weight matrix, amplifying small singular values and saturating large ones, without explicit SVD. Unlike sign/polar-style maps (e.g., those used in Muon) that drive nearly all nonzero singular values toward one, PC performs \emph{soft} spectrum shaping rather than near-orthogonalization.}
    \label{fig:pc-intro}
\end{figure}

Regarding the first question, we note that the spectrum of weight matrices can be understood from the perspective of signal propagation. 
For a linear layer, the largest and smallest singular values control the upper and lower bounds of the propagated signal magnitude, so their ratio characterizes the conditioning of signal propagation. 
A collapsed lower tail of singular values can therefore degrade signal propagation across depth. This suggests that controlling the weight spectrum can promote better signal propagation, which in turn makes the network easier to optimize. 
Moreover, we prove that under certain assumptions, for deep linear networks, if all weights of a multi-layer fully connected network have upper bounded condition number, then gradient descent converges to global minimizer at a rate dependent on the weight condition numbers.
This gives a partial answer to the first question.

Regarding the second question on the algorithm design,
an immediate challenge is how to make the spectrum control
method compatible with the original backpropagation method\footnote{Background: Existing normalization layers
are embedded into the network so that auto-differentiation applies to normalization layers.}.
A possible way to reshape the spectrum of a weight matrix $W$ is to compute an explicit factorization (e.g., a singular value decomposition (SVD) $W=U\,\mathrm{diag}(\sigma_1,\ldots,\sigma_r)\, V^\top$) and then modify the singular values $\sigma_1,\ldots,\sigma_r$. However, such matrix-factorization–style approaches may be rather costly.

We propose to add \textbf{polynomial preconditioning} into the neural network.
The key insight is that matrix polynomials provide a convenient mechanism for acting \emph{directly} on the spectrum.
Concretely, let \(W = U \,\mathrm{diag}(\sigma_1,\ldots,\sigma_r)\, V^\top\) be the SVD of \(W\), and let \(g\) be a certain feasible preconditioning polynomial. Then the transformed matrix \(g(W)\) has singular values \(\{|g(\sigma_1)|,\ldots,|g(\sigma_r)|\}\), i.e., each singular value is reshaped by the same scalar function \(g\).
This yields an explicit and controllable modification of the weight spectrum \emph{without} resorting to computationally heavy matrix decompositions.
We choose the polynomial \(g\) to approximate a desired shaping function that amplifies small singular values relative to large ones, thereby narrowing the spread of the spectrum and promoting well-conditioned weights throughout training.

Our main contributions are threefold:

\begin{itemize}
    
    \item \textbf{Algorithm design}. We propose a built-in-layer module, termed \textit{preconditioning (PC) layer}, which encourages well-conditioned weight matrices via polynomial preconditioning. This method is lightweight during training and causes no additional inference cost after training.

    \item \textbf{Empirical validation in LLM pre-training}.
    We evaluate PC layers on Llama-271M and Llama-1B models with both the
    AdamW \citep{kingma2014adam, loshchilov2017decoupled} and Muon
    \citep{jordan2024muon} optimizers. On Llama-1B, PC delivers consistent
    token-efficiency speedups under both optimizers—i.e., it reaches the same
    loss with fewer training tokens (2$\times$ with AdamW
    and 1.13$\times$ with Muon)—and improves zero-shot downstream accuracy
    in both cases. We further verify that PC improves weight-spectrum
    conditioning.

    \item \textbf{Theory for controlling weight spectrum}. Our intuition is that a “good weight spectrum” helps training. To support this intuition, we prove that a bounded weight spectrum implies geometric convergence of the loss to global minimum on deep linear networks, providing theoretical backing for the proposed method.

\end{itemize}

A preliminary version of this work, focusing on generative adversarial nets (GANs), is available online \citep{fang2021precondition}. Compared to that version, the present version targets large language models and introduces several critical LLM-specific modifications to the method, including (i) magnitude adjustment operations (norm recovery and learnable scaling)
and (ii) a specific choice of layers to apply preconditioning. 
These changes are essential for achieving strong performance in the LLM setting, which presents distinct optimization challenges not found in the earlier GAN-based experiments.

\section{Motivation: From Signal Propagation to Weight-Spectrum Control}
\label{sec:motivation}

\paragraph{Signal propagation as a guiding lens.}
Signal propagation provides a classical lens for understanding trainability: it studies how activations and gradients are transformed as they pass through depth. Historically, this lens has been most influential in the design of weight initialization schemes. Xavier \citep{glorot2010understanding} and Kaiming \citep{he2015delving} initializations aim to keep forward and backward signal magnitudes stable at the start of training. The line of work on orthogonal initialization further refines this principle by providing a more fine-grained analysis of how depth affects signal propagation \citep{saxe2013exact, pennington2017resurrecting, pennington2018emergence, xiao2018dynamical, hu2020provable}. Recent LLM training methods have also begun to control weights or their updates beyond initialization, through normalized Transformers \citep{loshchilov2024ngpt}, hyperspherical or hyperball constraints \citep{wen2025hyperball, ren2026rethinking}, and spectral-sphere constraints \citep{xie2026controlled}. This viewpoint is also compatible with $\mu$P-style spectral analyses, which suggest that the spectral scale of weights and updates is central to stable feature learning under width scaling \citep{yang2021tuning, yang2023spectral}. Together, these observations motivate treating the weight spectrum itself as a direct control knob for signal propagation.

\paragraph{Weight spectrum as a propagation control knob.}
Consider a linear transformation
\[
    y = Wx .
\]
The singular values of $W$ determine how much the layer can amplify or suppress different input directions:
\[
    \sigma_{\min}(W)\|x\|_2
    \leq
    \|Wx\|_2
    \leq
    \sigma_{\max}(W)\|x\|_2 .
\]
The same singular values also govern the adjoint transformation $W^\top$ that appears in backpropagation. Thus, the spectrum of $W$ gives a layer-local proxy for forward and backward signal propagation. In the ideal square case, if $W$ is orthogonal, then $\|Wx\|_2=\|x\|_2$ for every input $x$: the layer is exactly norm-preserving. This connects weight-spectrum control to the classical dynamical-isometry and orthogonal-initialization literature, where near-isometric transformations are used to maintain well-conditioned propagation through depth \citep{saxe2013exact, pennington2017resurrecting, pennington2018emergence, xiao2018dynamical, hu2020provable}.

\paragraph{Why exact orthogonality may not be the right target.}
Although orthogonal weights provide a clean signal-propagation ideal, enforcing exact orthogonality throughout training is too restrictive. To see this, consider a depth-$L$ linear network
\[
    f_\theta(x) = W_L W_{L-1}\cdots W_1 x,
\]
where each $W_\ell \in \mathbb{R}^{d\times d}$ is orthogonal. Then the end-to-end map $W_L\cdots W_1$ is also orthogonal. Even if we add a global scalar $\alpha$ and consider $\alpha W_L\cdots W_1$, all singular values of the end-to-end map are still identical. Hence such a model cannot represent anisotropic maps that require direction-dependent amplification, such as
\[
    x \mapsto \mathrm{diag}(1,2)x .
\]
This example highlights a basic tension. Perfectly flat spectra are favorable for norm preservation, but they can remove useful spectral anisotropy and weaken representation power. Therefore, the goal should not be to collapse all singular values to one.

\paragraph{Soft spectrum conditioning.}
The preceding discussion suggests the following design principle: the weight spectrum should be well-conditioned, but not perfectly flat. Equivalently, the condition number
\[
    \kappa(W) = \frac{\sigma_{\max}(W)}{\sigma_{\min}(W)}
\]
should be kept moderate, while still allowing nontrivial variation among singular values. This principle balances two desiderata. On the optimization side, reducing spectral spread mitigates excessive amplification or attenuation of signals. On the representation side, preserving controlled spectral variation allows the model to implement anisotropic transformations.

\paragraph{Design implication.}
The discussion above leads to a training-time weight-spectrum control problem. We do not merely want an initialization whose spectrum is well-behaved at step zero. Instead, we want a mechanism that acts on the effective weight matrices used during training. Such a mechanism should reduce the relative spread between small and large singular values, while avoiding the over-restrictive regime where all singular values are forced to be identical. The next section develops a polynomial weight-preconditioning mechanism that satisfies these requirements.

\section{Polynomial Weight Preconditioning}
\label{sec:weight_prec}
We now instantiate the training-time weight-spectrum control desideratum from Section~\ref{sec:motivation}. In analogy to normalization methods, we embed the control mechanism---preconditioning (PC) layer---directly into the network.

\paragraph{Background on polynomial preconditioning.}
Polynomial preconditioning is a classical technique in numerical linear algebra, originally developed to accelerate iterative solvers for symmetric linear systems $Qx=b$ \citep{johnson1983polynomial}. The idea is to replace $Q$ by $g(Q) := p(Q)Q$, where $p$ is a low-degree polynomial chosen so that the spectrum of $g(Q)$ concentrates near $1$, dramatically reducing its condition number and speeding up iterative methods such as conjugate gradient. \citet{johnson1983polynomial} formulate the design of $p$ as a scalar approximation problem (minimax or least-squares) on the spectral interval, and observe that least-squares polynomials tend to yield better empirical convergence by improving the entire spectrum rather than just the extreme condition number. We will adapt this perspective to the deep-learning setting, where the matrices of interest are \emph{rectangular} weight matrices and the goal is to control their singular-value spectrum during training. See Appendix~\ref{subapp:poly-preconditioner-johnson} for a self-contained review of polynomial preconditioning and Appendix~\ref{subapp:precond-spectrum} for a discussion of the relation between preconditioning, the condition number, and the full spectrum.

\paragraph{Outline.}
We first extend polynomial preconditioning to rectangular matrices in deep nets (\S\ref{sec:rect}), describe how to \emph{determine} the preconditioning polynomials (\S\ref{sec:find-poly}), and finally propose the preconditioning (PC) layer method (\S\ref{sec:pc-layer-alg}).

\subsection{Preconditioning Rectangular Matrices in Deep Nets}
\label{sec:rect}
In deep networks, the objects we wish to control are typically \emph{rectangular} weight matrices $W\in\mathbb{R}^{n\times m}$.

\paragraph{Polynomial preconditioning via a Gram matrix.}
Polynomials are not directly defined on a rectangular $W$, so we work through a symmetric Gram matrix---either $WW^\top\in\mathbb{R}^{n\times n}$ or $W^\top W\in\mathbb{R}^{m\times m}$---whose nonzero eigenvalues are the squared singular values of $W$. Concretely, we embed a preconditioning polynomial by left-multiplying $W$ with a polynomial $p$ in \(WW^\top\):
\[
g(W) \triangleq p(WW^\top)W \in \mathbb{R}^{n\times m},
\]
which preserves the shape of $W$ and can be implemented efficiently using only repeated matrix multiplications.
(Equivalently, one can use the right-preconditioned form $g(W)=W p(W^\top W)$; in practice one can pick the smaller Gram matrix for efficiency.) The following proposition shows that controlling the conditioning of $g(W)$ reduces to shaping a one-dimensional map $g(\sigma) = p(\sigma^2) \sigma$ on the singular-value interval of $W$.\footnote{We slightly abuse notation by using \(g\) to denote both the matrix map \(W \mapsto p(WW^\top)W\) and the induced scalar mapping \( \sigma \mapsto p(\sigma^2)\sigma\).} See the proof in Appendix~\ref{app:proof-clm-svd}.

\begin{prop}[Singular-value mapping]
\label{clm:svd}
Let $W$ have singular values $\sigma_1\ge\cdots\ge\sigma_m\ge 0$ and define the scalar map
$g(\sigma)=p(\sigma^2)\,\sigma$.
Then the singular values of $g(W)=p(WW^\top)W$ are $\{|g(\sigma_i)|\}_{i=1}^m$.
\end{prop}

Proposition~\ref{clm:svd} shows that, despite $W$ being rectangular, polynomial preconditioning still reduces spectral control to designing a \emph{scalar} map $g(\sigma) = p(\sigma^2)\sigma$ acting on individual singular values. In particular, the spread (and hence the conditioning) of the singular values of $g(W)$ is determined entirely by how $g$ deforms the singular values of $W$: by choosing $g$ to expand small $\sigma_i$ relative to large ones, we directly reduce the relative gap between $\sigma_{\min}$ and $\sigma_{\max}$ and thereby improve conditioning.
This is exactly the mechanism we will exploit in Section~\ref{sec:find-poly} to \emph{determine} the polynomial $p$ and in Section~\ref{sec:pc-layer-alg} to build our preconditioning layer.

\subsection{Finding Preconditioning Polynomials}
\label{sec:find-poly}

Building on the singular-value mapping in Section~\ref{sec:rect},
this subsection describes the criteria and procedure for choosing the
preconditioning map \(g(\sigma)=p(\sigma^2)\sigma\).

\paragraph{Optimization formulation.}
Suppose we are given a domain $[\gamma_L,\gamma_U]$, a target function $f$, and an integer $k$, we seek the best approximation to $f(\sigma)$ on $[\gamma_L,\gamma_U]$
within the polynomial class
\[
G_k =\{ g(\sigma)=p(\sigma^2)\sigma \mid \deg(p)\le k \}.
\]
A standard approach in polynomial preconditioning is to determine $g$ by minimizing a \emph{continuous}
weighted least-squares objective \citep{johnson1983polynomial}. Concretely, writing
\[
g(\sigma)=\Big(a_0+a_1\sigma^2+\cdots+a_k\sigma^{2k}\Big)\sigma,
\qquad a=(a_0,\ldots,a_k)\in\mathbb{R}^{k+1},
\]
the continuous weighted least-squares objective is
\begin{equation}
\label{eq:wls_continuous}
\min_{a\in\mathbb{R}^{k+1}}\;\int_{\gamma_L}^{\gamma_U}\bigl|g(\sigma)-f(\sigma)\bigr|^2\,w(\sigma)\,d\sigma,
\end{equation}
where $w(\sigma)=\sigma^{\alpha}$ is the weight function used by \citet{johnson1983polynomial} and $\alpha$ is a numerical constant.
In practice, we approximate the objective by finite-sample approximation
\begin{equation}
\label{eq:wls_discrete}
\min_{a\in\mathbb{R}^{k+1}}\;\sum_{i=1}^{n}\bigl|g(\sigma_i)-f(\sigma_i)\bigr|^2\,w(\sigma_i),
\end{equation}
where $\sigma_1, \ldots, \sigma_n \in [\gamma_L,\gamma_U]$ are the sample points. Note that $g(\sigma_i)$ is linear in the coefficient vector $a$: $g(\sigma_i)=\phi(\sigma_i)^\top a$, where $\phi(\sigma_i)=(\sigma_i,\sigma_i^3,\ldots,\sigma_i^{2k+1})^\top$. Therefore, (\ref{eq:wls_discrete}) is a standard weighted linear least-squares problem in $a$ and can be solved efficiently. See Appendix~\ref{subapp: poly-fit-alg} for the formal polynomial fitting algorithm.

\paragraph{Choosing the fitting interval $[\gamma_L,\gamma_U]$ via spectral normalization.}
The singular spectrum of \(W\) can vary substantially across layers and
training steps. To make a single fitted polynomial applicable to matrices
with different spectral scales, we design it for normalized weights
\(\widetilde W = W/s(W)\), where \(s(W)\) is intended to approximate
\(\|W\|_2\). Ideally, taking \(s(W)=\|W\|_2\) would map the singular values
of \(\widetilde W\) into \([0,1]\).
In practice, during training, our method normalizes each selected weight
matrix by an estimate \(s(W)\approx\|W\|_2\), computed via streaming power
iteration; this PC-layer implementation is described in
Section~\ref{sec:pc-layer-alg}, and the estimator is detailed in
Appendix~\ref{app:streaming-pi}.
To make the
polynomial fit robust to the small error of this estimate, we use the
slightly enlarged fitting interval
\([\gamma_L,\gamma_U]=[0,1.1]\) instead of \([0,1]\). 
Empirically, the streaming power-iteration estimator is accurate enough for this margin to cover the observed approximation error during training, keeping the polynomial evaluation within its design domain; see Appendix~\ref{app:power_iter_quality} for the estimator validation.

\paragraph{A feasible desired mapping.}
Recall that driving \emph{all} singular values to a common value would yield a
perfectly conditioned matrix.
However, any map of the form $g(\sigma)=p(\sigma^2)\sigma$ necessarily satisfies $g(0)=0$, so it cannot send arbitrarily small singular values to a fixed positive constant.
We therefore adopt a more realistic target: \emph{amplify small} singular values
(to improve conditioning) while \emph{saturating large} ones near $1$.
This motivates the piecewise-linear target
\[
\mathrm{PL}_b(\sigma) \;=\;
\begin{cases}
\sigma/b, & \sigma<b,\\[2pt]
1, & \sigma\ge b,
\end{cases}
\]
with cutoff $b > 0$.
Intuitively, $\mathrm{PL}_b$ enlarges the low end of the spectrum by a factor
$1/b$ and caps the high end at $1$, thereby reducing the spread of singular values.

\paragraph{Choosing the cutoff: optimization vs.\ expressiveness.}
The cutoff $b$ operationalizes the optimization--expressiveness trade-off introduced in Section~\ref{sec:motivation} by controlling how aggressively singular values are driven toward $1$.
If $b\approx 1$, then $\mathrm{PL}_b(\sigma)\approx \sigma$ over most of $[0,1]$,
so the spectrum is barely changed and conditioning is not improved.
At the other extreme, if $b$ is so small that $b<\sigma_{\min}(W)$,
then $\mathrm{PL}_b$ maps \emph{all} singular values to $1$; this is undesirable
because it can substantially reduce the expressive capacity of a deep network.
To balance these effects, we consider a \emph{set} of targets with different strengths,
e.g., $b\in\{0.8,0.6,0.4,0.3\}$ (see Figure~\ref{fig:pc_target_poly}).
Additionally, Appendix \ref{app:overflatten} provides a complementary over-flattening stress test, showing that near-perfect spectrum flattening can hurt performance and supporting the use of moderate rather than overly aggressive targets.

\paragraph{Fitted polynomials.}
Solving the weighted least-squares problem above yields a sequence of preconditioning polynomials with increasing strength. Concretely, we associate each \texttt{pc\_level} with a cutoff \(b\) in the target \(\mathrm{PL}_b\), using
\(b\in\{0.8,\,0.6,\,0.4,\,0.3\}\) for \(k\in\{1,2,3,4\}\),
where smaller \(b\) corresponds to a more aggressive push of singular values toward \(1\).
Using the polynomial-fitting algorithm detailed in Appendix~\ref{subapp: poly-fit-alg}, we obtain one set of fitted polynomials \(g_k(\sigma)=p_k(\sigma^2)\sigma\) (with overall degrees \(3,5,7,9\)) are:
\begin{equation*}
\label{eq:fitted-polys}
\begin{aligned}
g_1(\sigma) &= 1.507\sigma - 0.507\sigma^{3},\\
g_2(\sigma) &= 2.083\sigma - 1.643\sigma^{3} + 0.560\sigma^{5},\\
g_3(\sigma) &= 2.909\sigma - 4.649\sigma^{3} + 4.023\sigma^{5} - 1.283\sigma^{7},\\
g_4(\sigma) &= 3.625\sigma - 9.261\sigma^{3} + 14.097\sigma^{5} - 10.351\sigma^{7} + 2.890\sigma^{9}.
\end{aligned}
\end{equation*}
Figure~\ref{fig:pc_fitted_poly} plots these four fitted mappings. Other valid polynomial sets can be obtained by re-solving \eqref{eq:wls_discrete}.

We refer to the index \(k\) as the \emph{PC level} (\texttt{pc\_level}). Concretely, \(\texttt{pc\_level}=k\) means that \(p_k\) has degree \(k\) in \(g_k(\sigma)=p_k(\sigma^2)\sigma\), so the induced scalar map \(g_k\) has overall degree \(2k+1\). A larger \texttt{pc\_level} corresponds to a smaller cutoff \(b\) in the target \(\mathrm{PL}_b\), and therefore a stronger spectrum-shaping operation: it more aggressively lifts small singular values and saturates large ones, while still avoiding exact spectrum flattening. Table~\ref{tab:pc_level} summarizes this correspondence.

\begin{table}[htbp]
\centering
\setlength{\tabcolsep}{10pt}
\renewcommand{\arraystretch}{1.15}
\begin{tabular}{ccccc}
\toprule
\texttt{pc\_level} \(k\) & cutoff \(b\) & \(\deg p_k\) & \(\deg g_k\) & intuition \\
\midrule
1 & 0.8 & 1 & 3 & mild shaping \\
2 & 0.6 & 2 & 5 & moderate shaping \\
3 & 0.4 & 3 & 7 & strong shaping \\
4 & 0.3 & 4 & 9 & strongest in this work \\
\bottomrule
\end{tabular}
\caption{The PC level \(k\) as a spectrum-shaping strength knob: larger \(k\) uses a higher-degree polynomial fitted to a smaller cutoff \(b\), giving a more aggressive map.}
\label{tab:pc_level}
\end{table}

\begin{figure}[htbp]
\centering
\begin{subfigure}[t]{0.46\textwidth}
\centering
\includegraphics[width=\linewidth]{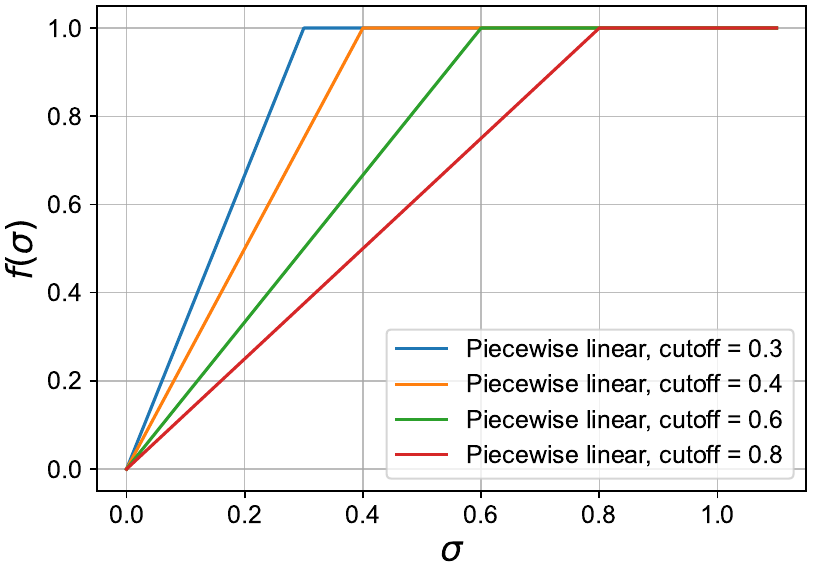}
\caption{\textbf{Piecewise-linear target mappings.} We show $\mathrm{PL}_b(\sigma)=\min(\sigma/b,1)$ for several cutoffs $b\in\{0.3,0.4,0.6,0.8\}$ in the fitting interval $[0,1.1]$. Smaller $b$ more aggressively enlarges small singular values and saturates large ones to $1$.}
\label{fig:pc_target_poly}
\end{subfigure}
\hfill
\begin{subfigure}[t]{0.46\textwidth}
\centering
\includegraphics[width=\linewidth]{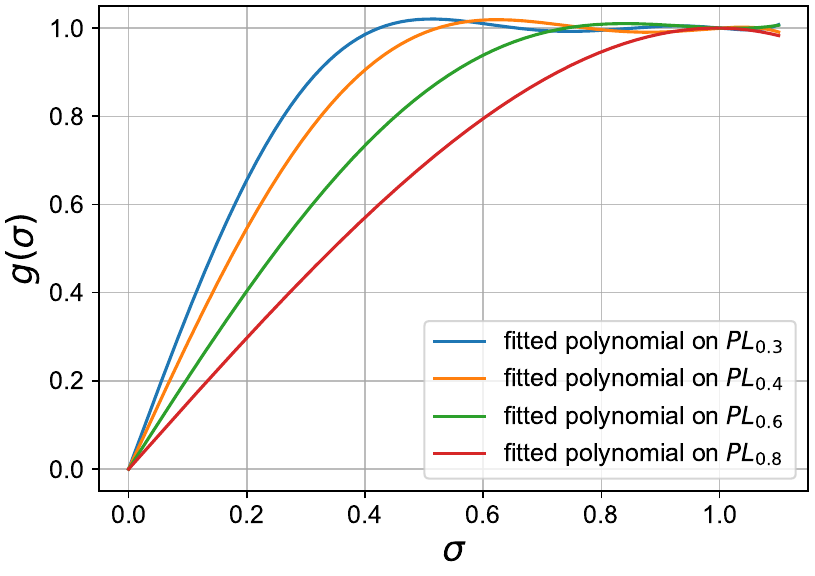}
\caption{\textbf{Fitted polynomial preconditioners.} Weighted least-squares solutions in $G_k=\{g(\sigma)=p(\sigma^2)\sigma:\deg p\le k\}$ for $k\in\{1,2,3,4\}$ (overall degrees $3,5,7,9$), each fitted to the corresponding $\mathrm{PL}_b$ target in (a) by Algorithm \ref{alg:poly-fitting} in Appendix \ref{subapp: poly-fit-alg}. The fitted map $g$ closely tracks the target while respecting the structural constraint $g(0)=0$.}
\label{fig:pc_fitted_poly}
\end{subfigure}
\caption{\textbf{Finding preconditioning polynomials on a fixed spectral domain.} (a) Target piecewise-linear mappings. 
(b) Low-degree polynomial preconditioners obtained by weighted least squares.
}
\label{fig:pc_poly}
\end{figure}

\subsection{Methodology: Preconditioning (PC) Layer}
\label{sec:pc-layer-alg}

In this subsection, we formally propose the Preconditioning (PC) layer,
a built-in weight-space module that reshapes the singular-value spectrum
of selected weight matrices. 
PC operates on a configurable set of blocks, acting inside the computation graph during training.
Algorithm~\ref{alg:pc} summarizes the procedure.

\begin{algorithm}[htbp]
\caption{PC Layer for Transformers}
\label{alg:pc}
\begin{algorithmic}[1]
\STATE{\textbf{Input:} Transformer architecture $\mathcal{A}$; subset of weight blocks \texttt{PC\_blocks} $\subseteq$ \texttt{param\_blocks}; preconditioning polynomial $g_k$ selected by \texttt{pc\_level}$=k$; a per-block learnable scalar $\gamma$ (initialized to $1$) for each $W \in$ \texttt{PC\_blocks}.}
\FOR{each weight matrix $W$ in \texttt{PC\_blocks}}
    \STATE{$s(W) \leftarrow \text{ estimation of } \Vert W \Vert_2$}
    \STATE{$\widetilde{W} = W / s(W)$ \hfill \textit{\# weight normalization}}
    \STATE{$g(\widetilde{W}) = p(\widetilde{W} \widetilde{W}^\top) \widetilde{W}$ \hfill \textit{\# polynomial preconditioning on singular values}}
    \STATE{${\mathrm{PC}}(W) \gets \gamma \cdot \left[s(W)\right]_{\rm stop\text{-}grad} \cdot g(\widetilde{W})$ \hfill \textit{\# norm recovery and learnable scaling}}
\ENDFOR
\STATE{\textbf{Return:} Transformer architecture $\mathcal{A}_{\rm PC} = \mathcal{A}[W \to {\mathrm{PC}}(W)]_{W \in \texttt{PC\_blocks}}$.}
\end{algorithmic}
\end{algorithm}

In general, our weight preconditioning method (Algorithm \ref{alg:pc}) replaces each weight matrix $W$ in selected blocks with its preconditioned form ${\mathrm{PC}}(W)$. Importantly, this replacement is \emph{architectural rather than numerical} — it modifies the computation graph itself, rather than simply altering the parameter values.

In this work, we apply PC to the attention output projection matrices
and the FFN projection matrices in Llama 2 architecture. We use $W_{\rm O}$ to denote
the attention output projection in each layer, and define
\[
    \texttt{ffn} := \{W_{\rm gate}, W_{\rm up}, W_{\rm down}\},
\]
where the three FFN matrices correspond to \texttt{gate\_proj},
\texttt{up\_proj}, and \texttt{down\_proj}, respectively. Thus, unless
otherwise specified,
\[
    \texttt{PC\_blocks} = \{W_{\rm O}, \texttt{ffn}\}.
\]

The following explains the key components of the weight preconditioning method:
\begin{itemize}

\item \textbf{Weight normalization (Line 4)}:
We scale the original weight matrix by
\(\widetilde{W} = W / s(W)\), where \(s(W)\approx\lVert W\rVert_{2}\) is an estimate of the spectral norm.
The goal is to confine the singular values of \(\widetilde{W}\) to (approximately) \([0,1]\), the interval over which the polynomial \(g\) is designed to be effective in preconditioning.
We compute \(s(W)\) by a streaming power-iteration procedure that avoids exact SVD; see Appendix~\ref{app:streaming-pi} for details.

\item \textbf{Polynomial preconditioning (Line 5):} As detailed in Section~\ref{sec:rect}, the odd polynomial of degree \(2k{+}1\),
\( g(\widetilde{W}) = p(\widetilde{W}\widetilde{W}^{\top})\,\widetilde{W} \),
maps each singular value \(\sigma\) of \(\widetilde{W}\) to \(g(\sigma)=\sigma\,p(\sigma^{2})\), thereby reshaping the spectrum for better conditioning. We implement this step using
compute-efficient techniques; see Appendix~\ref{app:pc_tricks} for details.

\item \textbf{Norm recovery and learnable scaling (Line 6)}:
After the polynomial preconditioning, 
we rescale the preconditioned matrix by the spectral-norm estimate 
$\left[s(W)\right]_{\rm stop\text{-}grad}$
(\emph{norm recovery})\footnote{Here \([\cdot]_{\mathrm{stop\text{-}grad}}\) indicates that the scalar
recovery factor is detached from the computation graph.} and by a learnable scalar $\gamma$ (initialized to $1$).
Intuitively, multiplying $s(W)$ back restores a comparable overall magnitude
after normalization and polynomial spectrum shaping.
This rescaling is crucial for the loss (see Section \ref{subsec:ab_study-rcv}).
The learnable $\gamma$ then allows the model 
to adaptively adjust this magnitude during training, 
leaving the final loss nearly unchanged but stabilizing signal-propagation metrics
(see Section \ref{subsec:ab_study-gamma}). 
Because both are \emph{scalar} multipliers,
they only rescale the singular values uniformly and do not alter their relative distribution.
Hence the spectrum shape tailored by the PC polynomial $g$ is preserved.

\end{itemize}

\begin{rem}
\label{rem:no-inference-overhead}
   Weight preconditioning is implemented as a reparameterization during training, thus it introduces \textbf{no} inference-time overhead.
    To be more specific, after training \(\mathcal{A}_{\mathrm{PC}}\), we simply store \(\mathrm{PC}(W)\) as the new weight and use it as the parameter of the original architecture \(\mathcal{A}\) during inference.
    Thus, there is no additional PC-polynomial computation at inference, and the FLOPs remain the same.
\end{rem}

\section{Experimental Results}

In this section, we present an empirical evaluation of the proposed PC layer. We first describe the experimental settings (\S\ref{subsec:exp_setting}), then report pre-training results (\S\ref{subsec:main_res}), and finally evaluate the computation and memory cost (\S\ref{subsec:compute-memory-cost}). 

\subsection{Experimental Settings}
\label{subsec:exp_setting}

\paragraph{Training setup.}
We evaluate PC by pre-training Llama-2-style \citep{touvron2023llama} models at two scales, 271M and 1B parameters, on the FineWeb dataset \citep{penedo2024the} using the TorchTitan codebase \citep{liang2025torchtitan}.\footnote{\url{https://github.com/pytorch/torchtitan}. Our implementation is built on an earlier revision of TorchTitan.}
Our main results are reported under two optimizers: AdamW \citep{kingma2014adam, loshchilov2017decoupled} and Muon \citep{jordan2024muon}. Detailed optimizer configurations are deferred to Appendix~\ref{app:exp_detail}.

Guided by recent pre-training practice in data-rich regimes \citep{gadre2024language,sardana2023beyond,wen2025fantastic}, we use token-to-parameter ratios well above the Chinchilla compute-optimal guideline ($\approx$20 tokens/parameter) \citep{hoffmann2022training}: the 271M model is trained on 57B tokens ($\approx$210 tokens/parameter, 10.5$\times$Chinchilla), and the 1B model on 160B tokens ($\approx$160 tokens/parameter, 8$\times$Chinchilla).
Both scales use sequence length 8192 and a fixed global batch of $2.62$M tokens per optimization step.
All runs are conducted on 8$\times$NVIDIA H100 GPUs.

\paragraph{Learning-rate schedule and grid search.}
We use a cosine learning-rate schedule with linear warmup: the LR is linearly warmed up over the first $1\%$ of training steps to the peak value, and then decayed to $10\%$ of the peak by the end of training following a cosine schedule.

We tune the \emph{peak LR only on the baseline configuration}. For each model size and optimizer, we sweep the peak LR over a $\sqrt{2}$-spaced geometric grid (e.g., $a/2,\, a/\sqrt{2},\, a,\, a\sqrt{2},\, 2a$). Among the runs that train stably (no spikes in the training loss curve), we pick the peak LR with the lowest final validation loss (see Table~\ref{tab:peak_lr_grid}).
The PC run then directly adopts this \emph{baseline-tuned} peak LR (together with all other non-PC hyperparameters), without any further tuning.
This procedure is conservative for PC, since the LR is tuned only on the baseline and is not re-optimized after adding PC.

\begin{table}[htbp]
\centering
\setlength{\tabcolsep}{10pt}
\renewcommand{\arraystretch}{1.15}
\begin{tabular}{llcc}
\toprule
\textbf{Model} & \textbf{Optimizer} & \textbf{Sweep range} & \textbf{Selected peak LR} \\
\midrule
\textbf{Llama-271M} & AdamW & $[3.0\times10^{-4},\, 3.394\times10^{-3}]$ & $8.486\times10^{-4}$ \\
             & Muon  & $[1.061\times10^{-3},\, 1.2\times10^{-2}]$   & $4.243\times10^{-3}$ \\
\midrule
\textbf{Llama-1B}   & AdamW & $[5.303\times10^{-5},\, 4.243\times10^{-4}]$ & $3\times10^{-4}$ \\
             & Muon  & $[7.5\times10^{-4},\, 4.243\times10^{-3}]$    & $1.5\times10^{-3}$ \\
\bottomrule
\end{tabular}
\caption{Peak learning-rate grid search on the baseline configuration.}
\label{tab:peak_lr_grid}
\end{table}

\paragraph{PC configuration.}
Unless otherwise specified, PC is applied to the FFN projections and the attention output projection in each transformer layer, i.e., $\texttt{PC\_blocks}=\{\texttt{ffn}, W_{\rm O}\}$ with $\texttt{ffn}=\{W_{\rm gate}, W_{\rm up}, W_{\rm down}\}$. For spectral normalization, we estimate $\|W\|_2$ using a 10-step streaming power iteration during each training step. All PC runs use norm recovery and a per-block learnable scalar $\gamma$, as described in Algorithm~\ref{alg:pc}. The only optimizer-dependent PC hyperparameter in our main experiments is \texttt{pc\_level}, the degree \(k\) of \(p_k\) in \(g_k(\sigma)=p_k(\sigma^2)\sigma\) and hence a knob controlling the strength of spectrum shaping (Table~\ref{tab:pc_level}); a larger \texttt{pc\_level} applies a higher-degree, more aggressive PC map. We use $\texttt{pc\_level}=4$ with AdamW and $\texttt{pc\_level}=2$ with Muon, following the ablation in Section~\ref{subsec:ab_study}.

Additional hyperparameters, model configurations, and implementation details are provided in Appendix~\ref{app:exp_detail}.

\subsection{Main Results}
\label{subsec:main_res}
We present pre-training results of PC under two optimizers, AdamW and Muon, on Llama-271M and Llama-1B models. For each optimizer, we compare a transformer equipped with PC layers against an optimizer-matched standard Llama-style transformer baseline without PC.

\paragraph{PC improves optimization under AdamW.}
We first evaluate PC under AdamW on Llama-271M and Llama-1B. On Llama-271M (Figure~\ref{fig:271m-adamw-loss-vs-token}), under the same token budget, the Transformer with PC layer reaches a 0.055 lower final validation loss than the standard baseline. It attains the baseline loss with about 39\% fewer training tokens, a $1.63\times$ speedup. The advantage carries over to the 1B scale (Figure~\ref{fig:1b-adamw-loss-vs-token}). Here PC lowers the final validation loss by 0.070 and reaches the baseline loss with 50\% fewer tokens, a $2\times$ speedup. The optimization gain does not diminish at 1B scale and is even slightly larger.

\begin{figure}[htbp]
  \centering
  \begin{subfigure}[t]{0.48\textwidth}
    \centering
    \includegraphics[width=\linewidth]{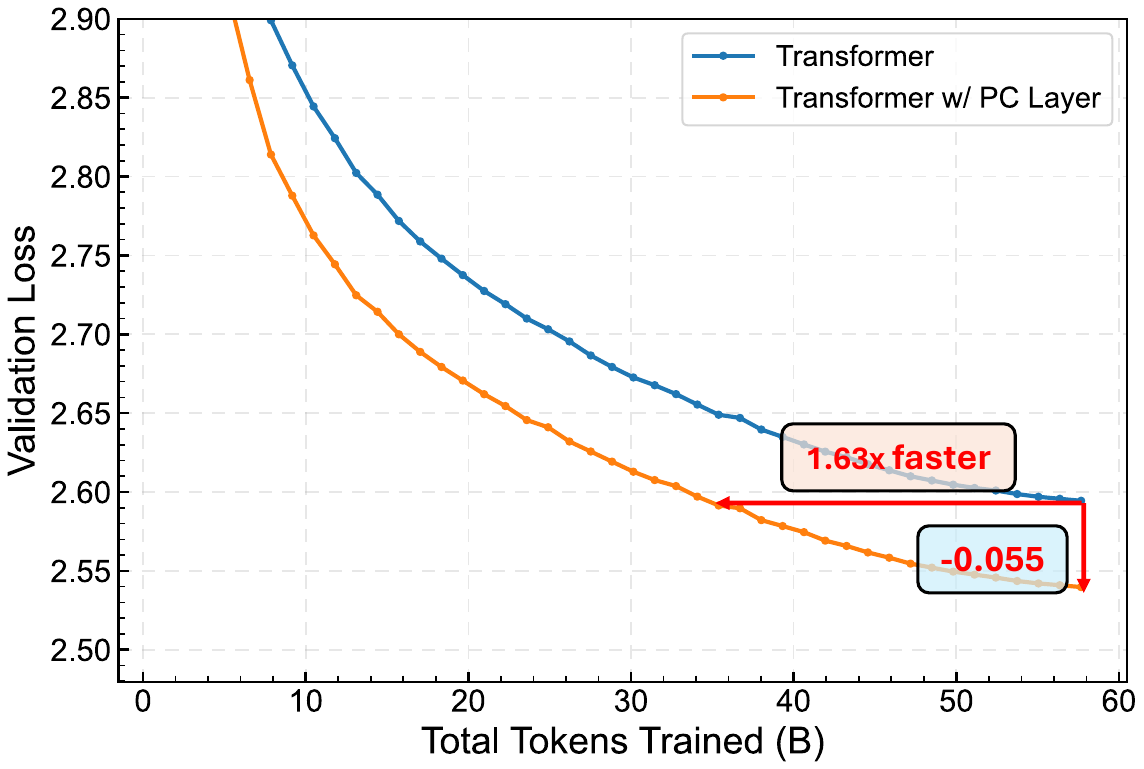}
    \caption{\textbf{Llama-271M.} Validation loss vs.\ training tokens.}
    \label{fig:271m-adamw-loss-vs-token}
  \end{subfigure}\hfill
  \begin{subfigure}[t]{0.48\textwidth}
    \centering
    \includegraphics[width=\linewidth]{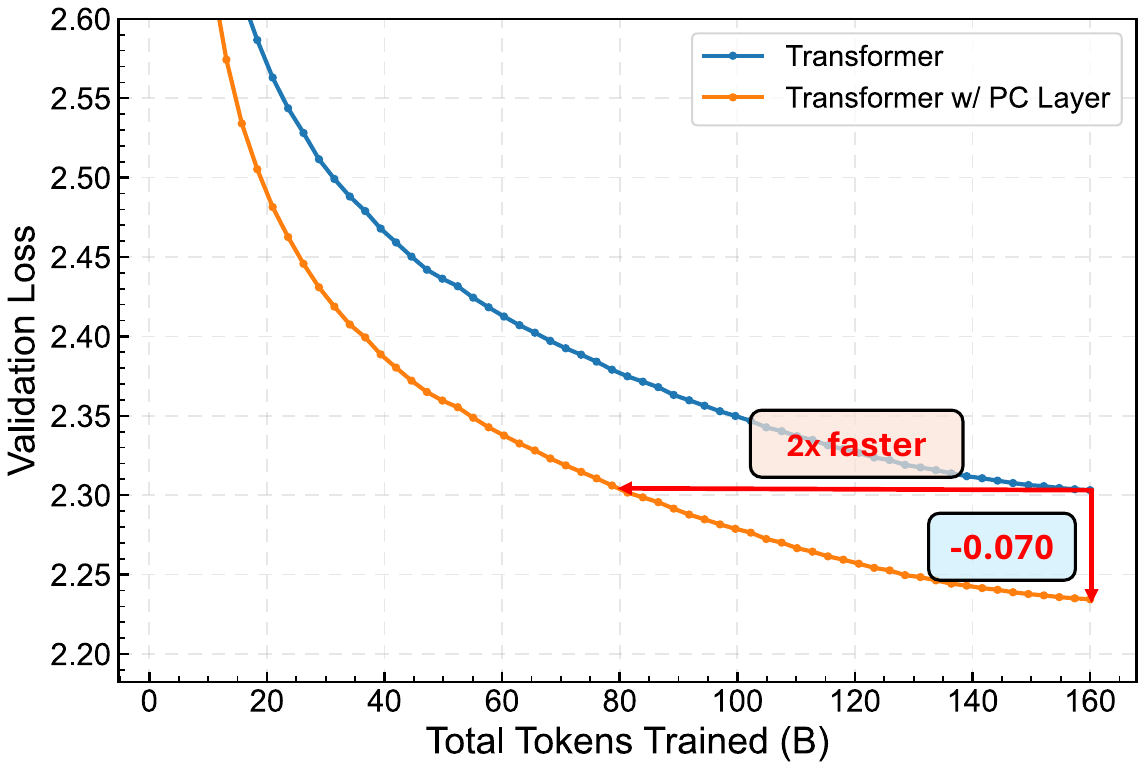}
    \caption{\textbf{Llama-1B.} Validation loss vs.\ training tokens.}
    \label{fig:1b-adamw-loss-vs-token}
  \end{subfigure}
    \caption{\textbf{PC performance under AdamW.}
    Validation loss vs.\ training tokens on (a) Llama-271M and (b) Llama-1B.
    Under the same total token budget, adding the PC layer reduces the final validation loss by \emph{0.055} on 271M (a $1.63\times$ token-efficiency speedup) and by \emph{0.070} on 1B (a $2\times$ speedup). The gain does not diminish at the larger scale.}
  \label{fig:pc-scaling-adamw}
\end{figure}

\paragraph{PC also helps under Muon.}
We further evaluate PC under Muon, a second widely used optimizer for LLM pre-training. Adding the PC layer again lowers the final validation loss on both scales (Figure~\ref{fig:271m-muon-loss-vs-token},~\ref{fig:1b-muon-loss-vs-token}). The reduction is 0.006 on Llama-271M, a $1.07\times$ speedup, and 0.012 on Llama-1B, a $1.13\times$ speedup. These suggest that the improvement brought by PC layer stays consistent across scales under Muon.

\begin{figure}[htbp]
  \centering
  \begin{subfigure}[t]{0.48\textwidth}
    \centering
    \includegraphics[width=\linewidth]{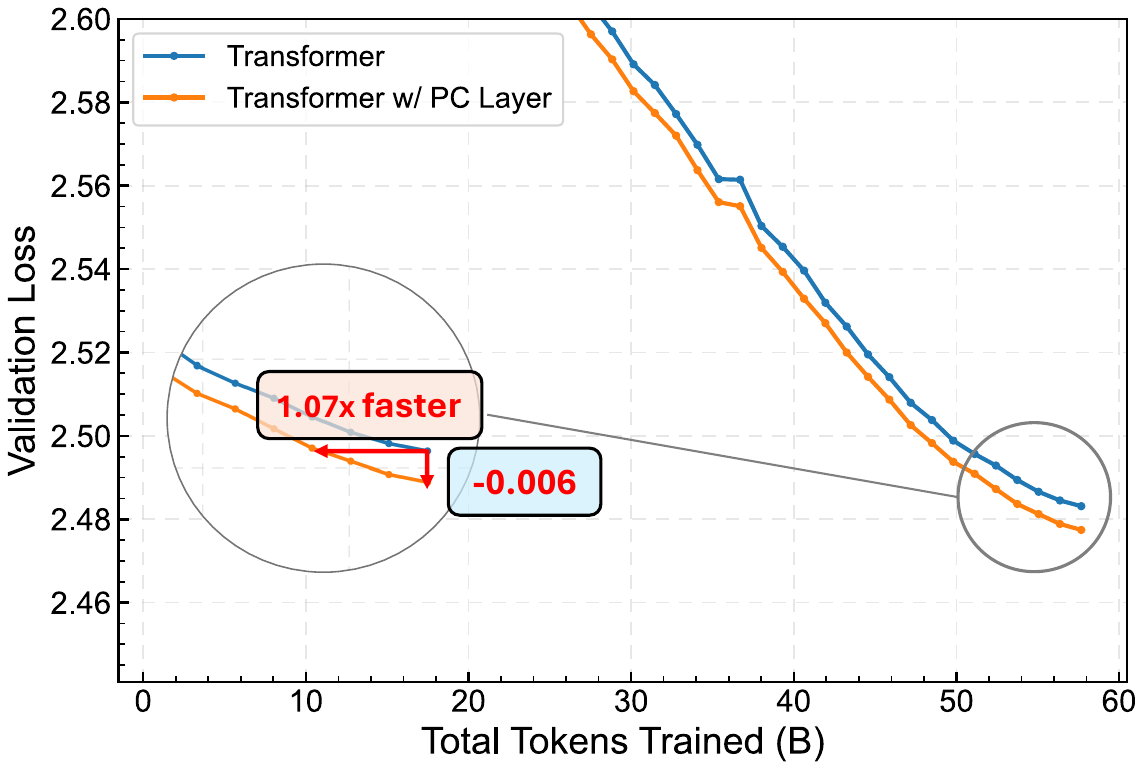}
    \caption{\textbf{Llama-271M.} Validation loss vs.\ training tokens.}
    \label{fig:271m-muon-loss-vs-token}
  \end{subfigure}\hfill
  \begin{subfigure}[t]{0.48\textwidth}
    \centering
    \includegraphics[width=\linewidth]{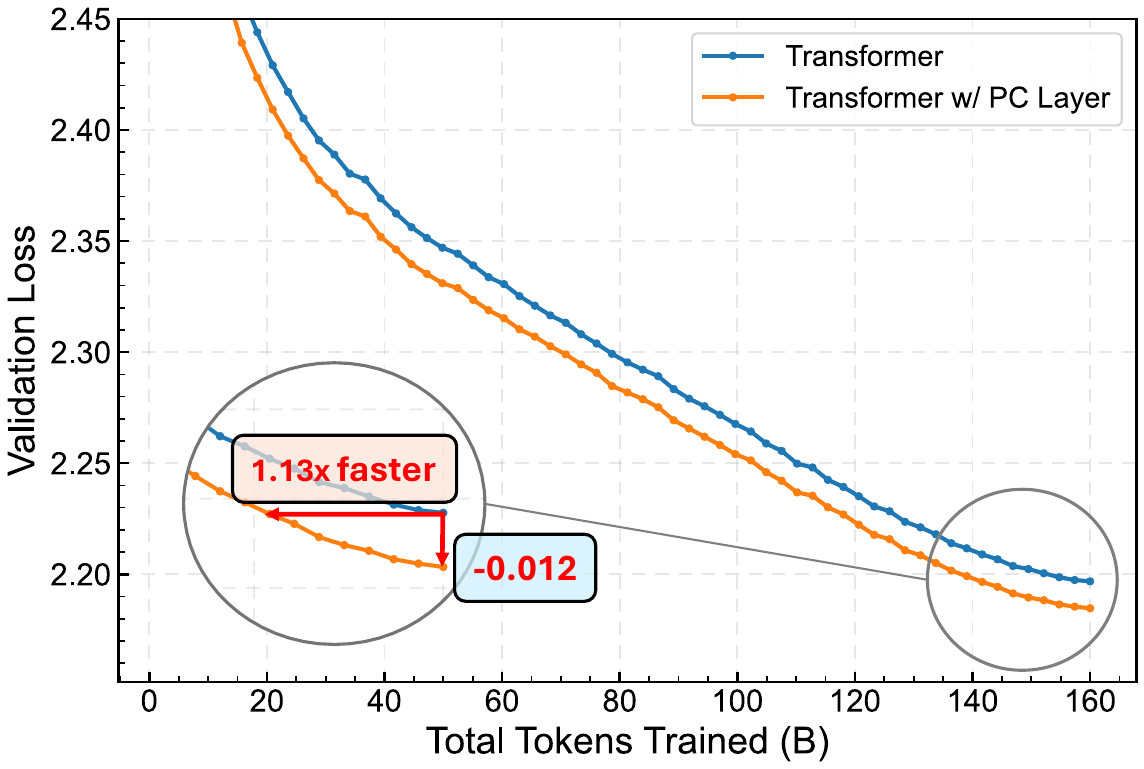}
    \caption{\textbf{Llama-1B.} Validation loss vs.\ training tokens.}
    \label{fig:1b-muon-loss-vs-token}
  \end{subfigure}
    \caption{\textbf{PC performance under Muon.}
    Validation loss vs.\ training tokens on (a) Llama-271M and (b) Llama-1B.
    Adding the PC layer consistently reduces the final validation loss across scales. The reduction is \emph{0.006} on 271M (a $1.07\times$ token-efficiency speedup) and \emph{0.012} on 1B (a $1.13\times$ speedup).
    }
  \label{fig:pc-scaling-muon}
\end{figure}

\paragraph{Downstream task evaluation.}
\label{para:downstream-main}
We evaluate the pre-trained 1B-parameter model on a suite of standard zero-shot downstream tasks using the Language Model Evaluation Harness \citep{eval-harness}. Specifically, we report results on LAMBADA \citep{paperno2016lambada},
HellaSwag \citep{zellers2019hellaswag}, WinoGrande \citep{sakaguchi2021winogrande}, PIQA \citep{bisk2020piqa}, BoolQ \citep{clark2019boolq}, ARC-Easy/ARC-Challenge (ARC-E/ARC-C) \citep{clark2018think}, CommonsenseQA (CSQA) \citep{talmor2019commonsenseqa}, and OpenBookQA (OBQA) \citep{mihaylov2018can}. Table~\ref{tab:downstream-1b} shows that adding the PC layer improves average downstream accuracy under both optimizers: by 0.0206 points under AdamW ($0.4539\to0.4745$) and by 0.0125 points under Muon ($0.4880\to0.5005$), winning on 8 out of 9 tasks in each case.

\begin{table}[htbp]
\centering
\renewcommand{\arraystretch}{1.15}
\setlength{\tabcolsep}{6pt}
\footnotesize
\begin{tabular}{l|ccccccccc|c}
\toprule
\textbf{Method} & \textbf{LMB.} & \textbf{Hella.} & \textbf{Wino.} & \textbf{PIQA} & \textbf{BoolQ} & \textbf{ARC-E} & \textbf{ARC-C} & \textbf{CSQA} & \textbf{OBQA} & \textbf{Avg.} \\
& {\scriptsize(OpenAI)} & & & & & & & & & \\
& Acc $\uparrow$ & Acc\textsubscript{n} $\uparrow$ & Acc $\uparrow$ & Acc $\uparrow$ & Acc $\uparrow$ & Acc\textsubscript{n} $\uparrow$ & Acc\textsubscript{n} $\uparrow$ & Acc $\uparrow$ & Acc\textsubscript{n} $\uparrow$ & $\uparrow$ \\
\midrule
AdamW baseline & 0.5026 & 0.5137 & 0.5438 & 0.7122 & \textbf{0.5645} & 0.4769 & 0.2645 & 0.2088 & 0.2980 & 0.4539 \\
AdamW + PC     & \textbf{0.5366} & \textbf{0.5642} & \textbf{0.5627} & \textbf{0.7274} & 0.5554 & \textbf{0.5034} & \textbf{0.2807} & \textbf{0.2105} & \textbf{0.3300} & \textbf{0.4745} \\
\midrule
Muon baseline  & 0.5717 & 0.6004 & 0.5809 & 0.7470 & 0.5550 & 0.5076 & 0.2892 & 0.1859 & \textbf{0.3540} & 0.4880 \\
Muon + PC      & \textbf{0.5812} & \textbf{0.6071} & \textbf{0.5967} & \textbf{0.7486} & \textbf{0.5847} & \textbf{0.5370} & \textbf{0.2986} & \textbf{0.2048} & 0.3460 & \textbf{0.5005} \\
\bottomrule
\end{tabular}
\caption{\textbf{Zero-shot downstream evaluation on Llama-1B.} Adding the PC layer improves average accuracy under both optimizers: from $0.4539$ to $0.4745$ under AdamW and from $0.4880$ to $0.5005$ under Muon, with consistent per-task improvements on 8 out of 9 tasks in each case. Best results within each optimizer block are in \textbf{bold}. To reduce length bias in candidate scoring, we report length-normalized accuracy (Acc\textsubscript{n}, i.e.\ \texttt{acc\_norm}) when available, namely for HellaSwag, ARC-E, ARC-C, and OBQA; the remaining tasks use plain accuracy (Acc). The average is computed over all nine reported task metrics.}
\label{tab:downstream-1b}
\end{table}

\subsection{Computational and Memory Cost Analysis}
\label{subsec:compute-memory-cost}

The polynomial preconditioner can be computed efficiently via Horner’s method (see Appendix~\ref{app:pc_tricks} for the implementation details). We provide a theoretical FLOPs analysis in Appendix~\ref{subapp:pc_cost} based on matrix-multiplication counts. For Llama-1B training, the estimated relative FLOPs overhead is at most 0.39\% under the AdamW default configuration ($\texttt{pc\_level}=4$) and 0.24\% when PC is combined with Muon ($\texttt{pc\_level}=2$).
In terms of memory, under the Llama-1B setting, the PC layer increases the per-GPU peak active memory by approximately 9.56\% under AdamW and 8.73\% under Muon (see Appendix~\ref{subapp:pc_memory}).

As discussed in Remark~\ref{rem:no-inference-overhead}, the learned preconditioned weights can be absorbed into the original architecture after training; hence inference incurs no additional computation or memory overhead.

\section{PC Improves the Weight Spectrum}
\label{sec:pc_improve_spec}

In this section, we quantify how PC improves the spectral conditioning of transformer weight matrices. We first propose a robust spectral metric, the modified condition number, and then extend it to a global model-level measure that we track during training.

\paragraph{Modified condition number.}
First, we introduce a \emph{Modified Condition Number} metric, denoted as $\tilde{\kappa}(W)$, to robustly evaluate the conditioning of a weight matrix $W$. The traditional condition number $\sigma_{\text{max}} / \sigma_{\text{min}}$ is numerically fragile, because the smallest singular value of a trained network is often extremely close to zero, which makes the ratio blow up.

To mitigate this and capture the effective spectral spread, we define $\tilde{\kappa}(W)$ as the ratio of the top singular value to the average of the bottom $10\%$ singular values. Let $n = \min(m, d)$ be the number of singular values of $W \in \mathbb{R}^{m \times d}$, ordered such that $\sigma_1 \geq \sigma_2 \geq \dots \geq \sigma_n$. The modified condition number is formally defined as:

\begin{equation*}
\tilde{\kappa}(W) = \frac{\sigma_1}{\bar{\sigma}_{\text{bottom-10\%}}}
\end{equation*}
where $\sigma_1$ denotes the largest singular value and $\bar{\sigma}_{\text{bottom-10\%}}$ denotes the average of the smallest $\lceil 0.1n \rceil$ singular values. Averaging over the bottom $10\%$ rather than taking the single $\sigma_{\min}$ smooths out the near-zero tail, giving a stable measure of how widely the spectrum is spread.

\paragraph{\emph{Global} modified condition number.}
To quantify the overall spectral health of the model, we further aggregate $\tilde{\kappa}(W)$ computed on each critical weight block into a \emph{Global Modified Condition Number} (GMCN, $\tilde{\kappa}$). We employ the geometric mean of the modified condition numbers across all critical weight matrices. The geometric mean is preferred over a simple product for numerical stability when aggregating across multiple blocks:
\begin{equation*}
\tilde{\kappa} = \left( \prod_{l \in L} \prod_{b \in B_l} \tilde{\kappa}(W'_{l, b}) \right)^{1/N}
\end{equation*}
where $L$ represents the set of transformer layers, $B_l = \{W_{\rm Q}, W_{\rm K}, W_{\rm V}, W_{\rm O}, W_{\rm gate}, W_{\rm up}, W_{\rm down}\}$ denotes the set of key weight blocks in layer $l$, and $N = |L| \cdot |B_l|$ is the total number of blocks.\footnote{In the Llama~2 architecture, $W_{\rm gate}$, $W_{\rm up}$, and $W_{\rm down}$ refer to the three linear projections in the feed-forward network (FFN): the gate (\texttt{gate\_proj}), up (\texttt{up\_proj}), and down (\texttt{down\_proj}) matrices, respectively.}
Crucially, the evaluated matrix $W'_{l, b}$ depends on the optimization strategy: for blocks targeted by our preconditioner ($W_{\rm O}, W_{\rm gate}, W_{\rm up}, W_{\rm down}$), we compute $\tilde{\kappa}$ on the preconditioned matrix $W'_{l, b} = \text{PC}(W_{l, b})$, which we refer to throughout as the \emph{effective weight}; for non-preconditioned blocks ($W_{\rm Q}, W_{\rm K}, W_{\rm V}$), we use the original weights, $W'_{l, b} = W_{l, b}$.

\paragraph{Experimental analysis of $\tilde{\kappa}$.}
The same geometric mean can be restricted to any subset of blocks: we report it over the PC-targeted blocks (\texttt{ffn} and $W_{\rm O}$), over the non-preconditioned attention blocks ($W_{\rm Q}, W_{\rm K}, W_{\rm V}$), and over all blocks (the global GMCN).
We monitor the modified condition number throughout training (Figure~\ref{fig:1b-adamw-mcn}). Here the baseline curve uses the original weights from the baseline run, while the PC curve uses the effective weights $\mathrm{PC}(W)$ from a separate PC run.
On the blocks directly targeted by PC, namely \texttt{ffn} and $W_{\rm O}$, PC shows a short early transient and then steadily lowers $\tilde{\kappa}$. At the final checkpoint, this aggregate decreases from about $17.3$ for the baseline to about $7.7$ for PC. We also examine $W_{\rm Q}$, $W_{\rm K}$, and $W_{\rm V}$, which are not directly preconditioned. Even though these blocks are never preconditioned, in the PC run their aggregate $\tilde{\kappa}$ ends up clearly below the baseline near the end of training. This indicates that applying PC to \texttt{ffn} and $W_{\rm O}$ not only avoids harming the untargeted weights but also improves their conditioning indirectly.

For the global aggregate, the baseline GMCN grows to about $42.4$, while PC stabilizes around $25.0$, giving a roughly $41\%$ reduction. This verifies that PC substantially improves the effective conditioning of the model weights during training.

\begin{figure}[htbp]
    \centering
    \begin{subfigure}[t]{0.32\textwidth}
        \centering
        \includegraphics[width=\linewidth]{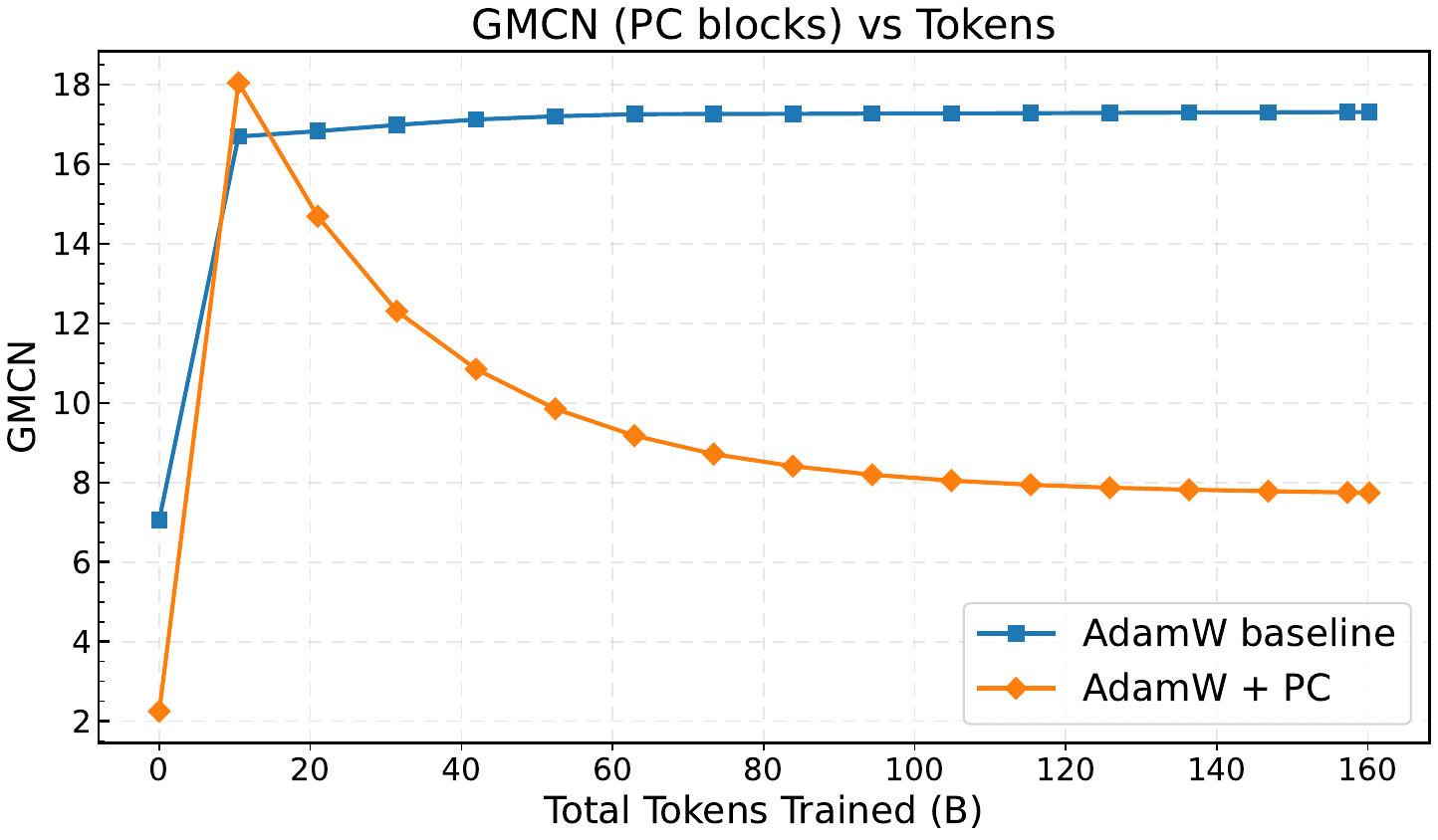}
        \caption{FFN + $W_{\rm O}$}
        \label{fig:1b-adamw-mcn-o-ffn}
    \end{subfigure}\hfill
    \begin{subfigure}[t]{0.32\textwidth}
        \centering
        \includegraphics[width=\linewidth]{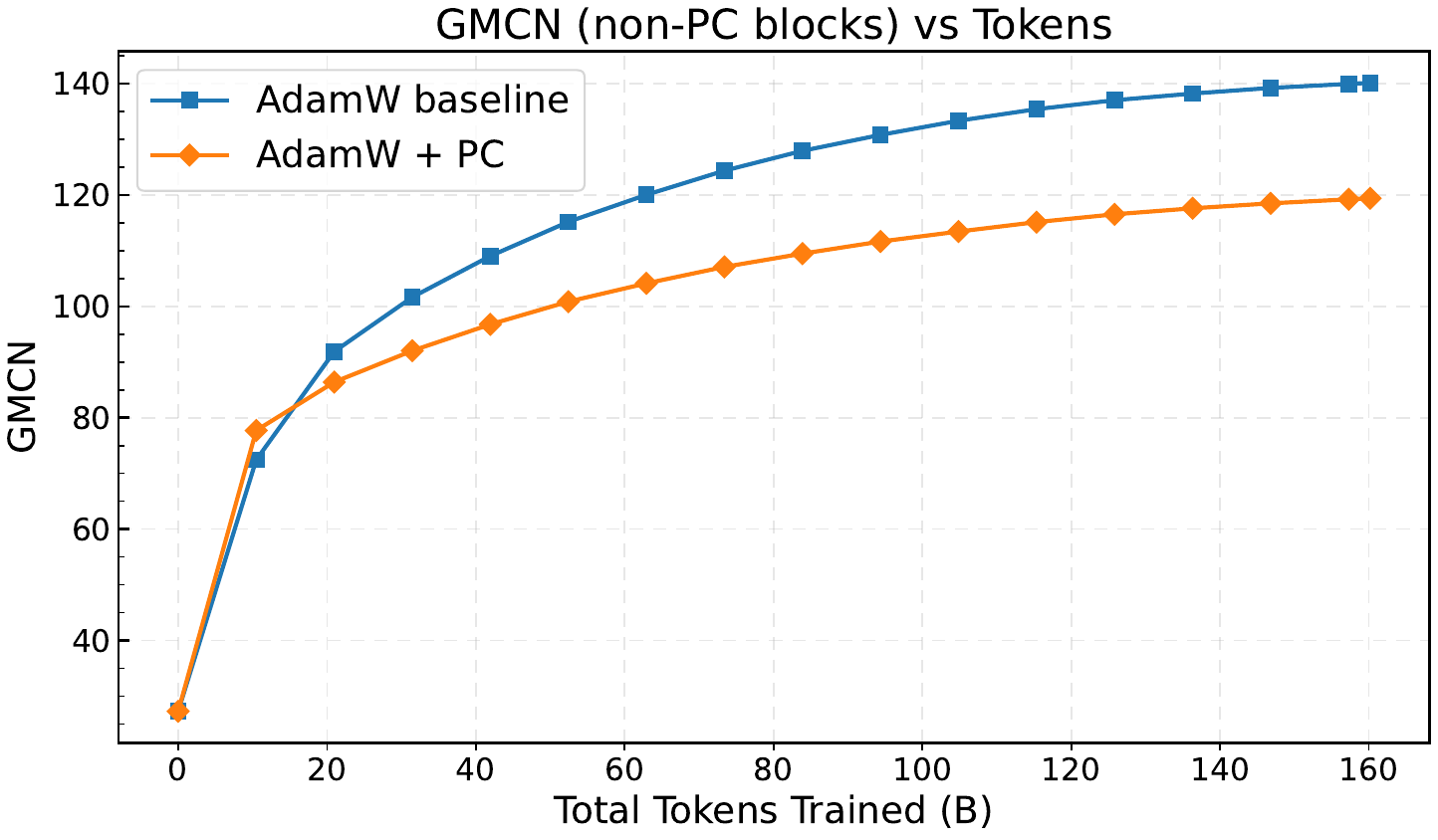}
        \caption{$W_{\rm Q}, W_{\rm K}, W_{\rm V}$}
        \label{fig:1b-adamw-mcn-qkv}
    \end{subfigure}\hfill
    \begin{subfigure}[t]{0.32\textwidth}
        \centering
        \includegraphics[width=\linewidth]{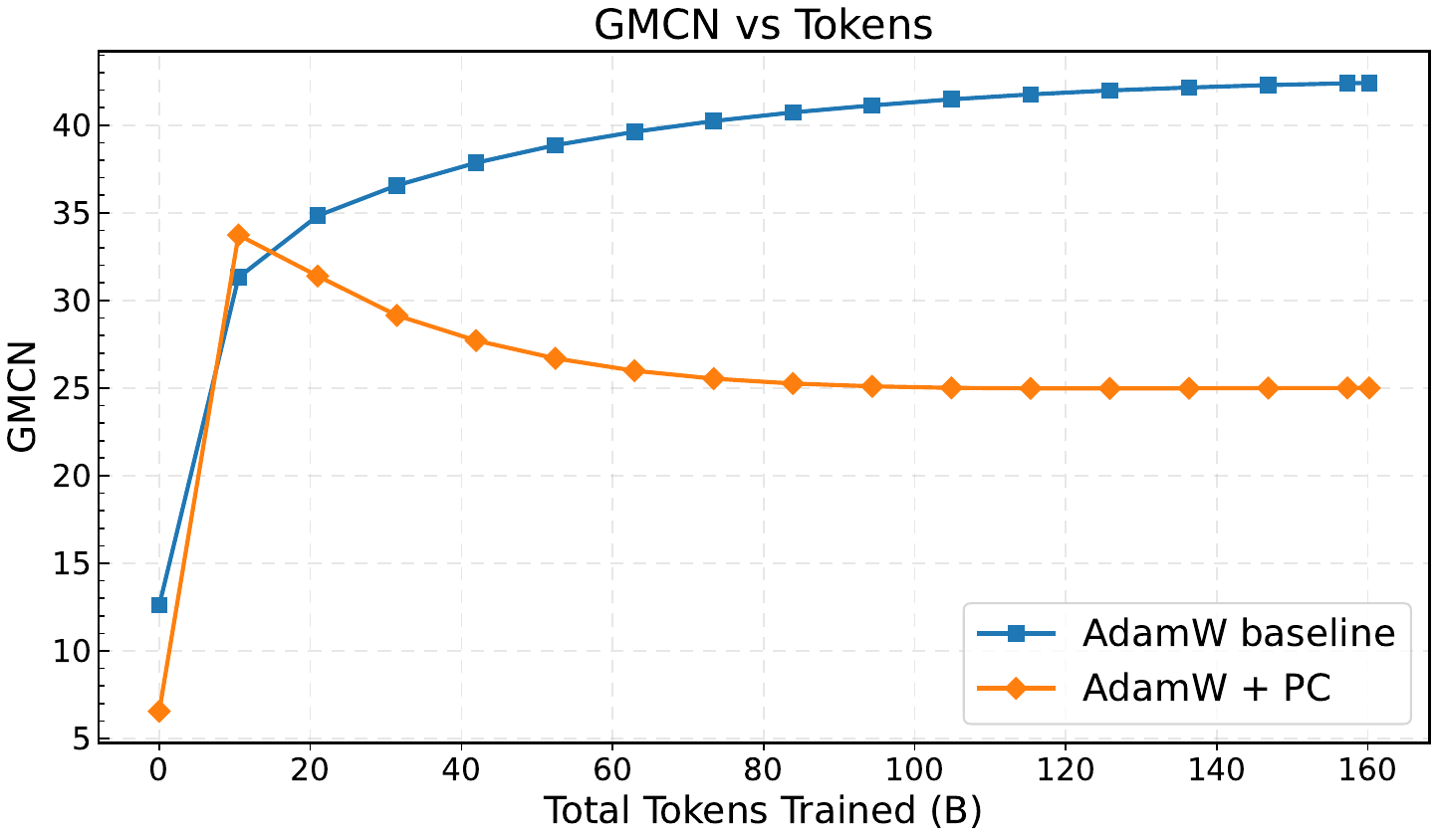}
        \caption{Global}
        \label{fig:1b-adamw-mcn-global}
    \end{subfigure}
    \caption{\textbf{Evolution of modified condition number under AdamW.}
    We report $\tilde{\kappa}$ for the preconditioned blocks (\texttt{ffn} and $W_{\rm O}$), the non-preconditioned attention-input blocks ($W_{\rm Q}, W_{\rm K}, W_{\rm V}$), and their global aggregate. The baseline curves are computed on the original weights, while the PC curves are computed on the effective weights used during the PC trajectory. PC improves the conditioning of the blocks to which it is applied, while the non-preconditioned blocks also show a clear reduction in $\tilde{\kappa}$ rather than any deterioration. Together, these reductions give a lower model-level GMCN.}
    \label{fig:1b-adamw-mcn}
\end{figure}

\begin{figure}[t!]
    \centering
    \begin{subfigure}[t]{0.24\textwidth}
        \centering
        \includegraphics[width=\linewidth]{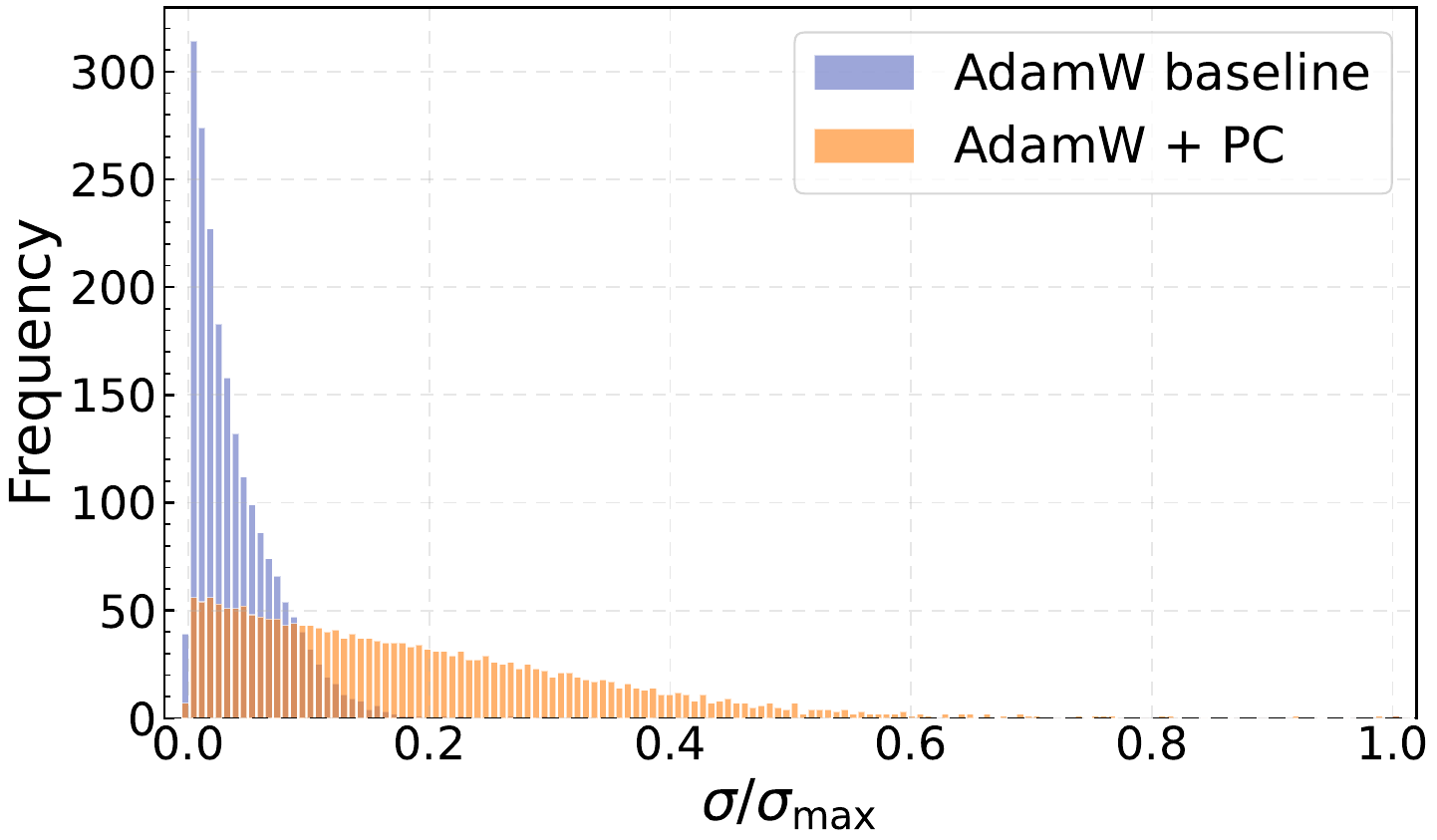}
        \caption{Layer 2: $W_{\rm O}$}
        \label{fig:spec-l2-o}
    \end{subfigure}\hfill
    \begin{subfigure}[t]{0.24\textwidth}
        \centering
        \includegraphics[width=\linewidth]{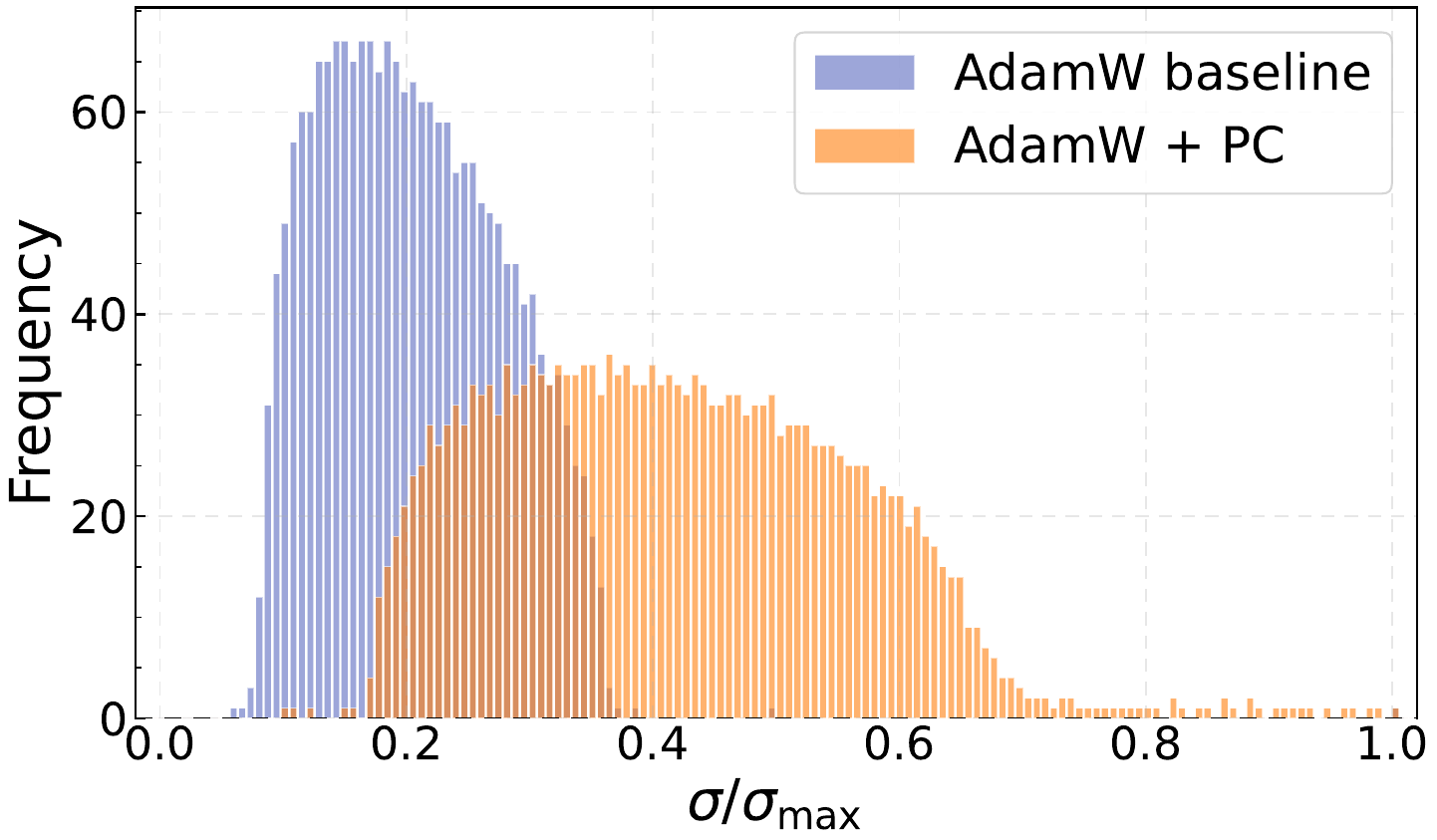}
        \caption{Layer 2: $W_{\rm gate}$}
        \label{fig:spec-l2-w1}
    \end{subfigure}\hfill
    \begin{subfigure}[t]{0.24\textwidth}
        \centering
        \includegraphics[width=\linewidth]{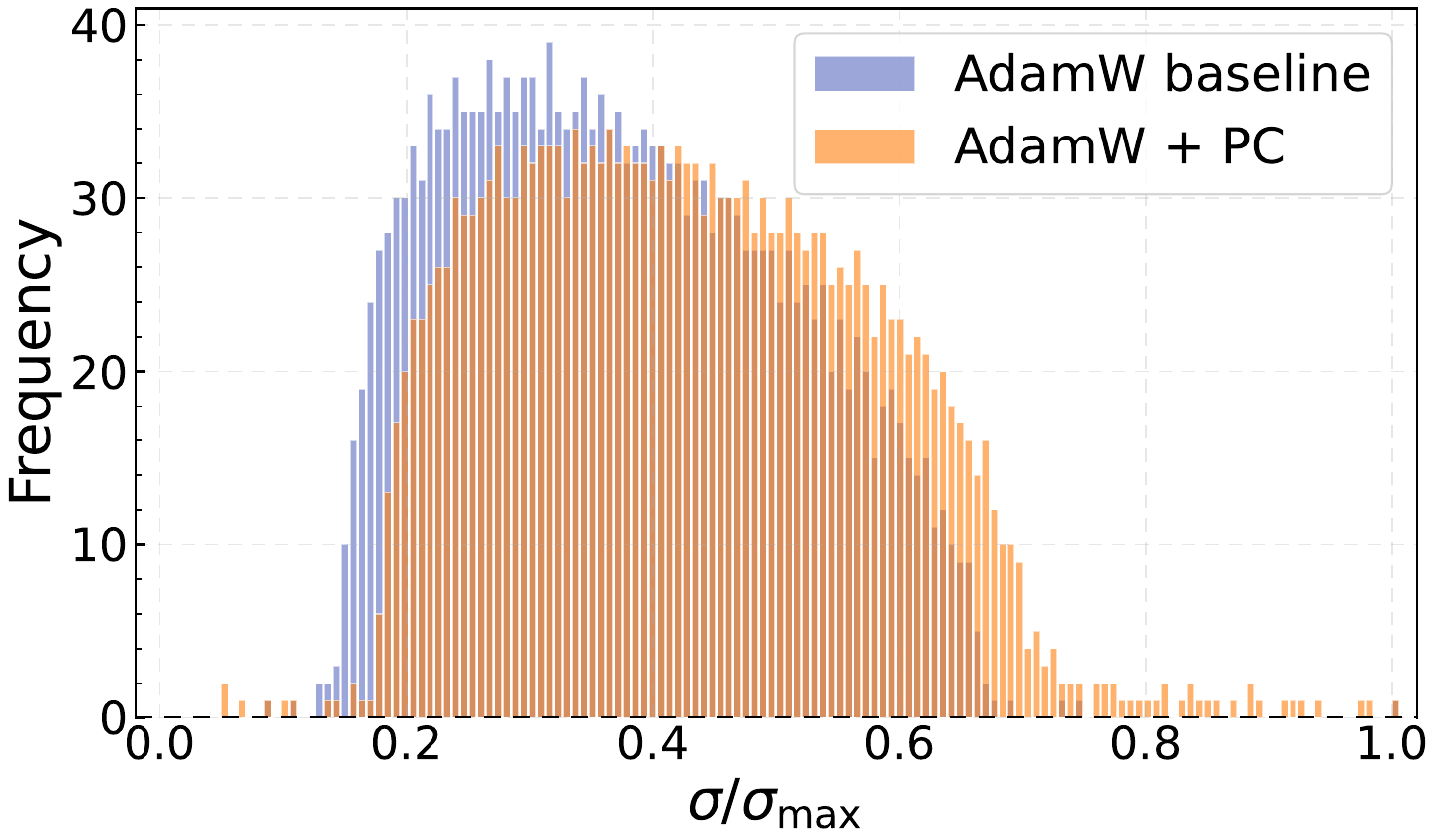}
        \caption{Layer 2: $W_{\rm up}$}
        \label{fig:spec-l2-w3}
    \end{subfigure}\hfill
    \begin{subfigure}[t]{0.24\textwidth}
        \centering
        \includegraphics[width=\linewidth]{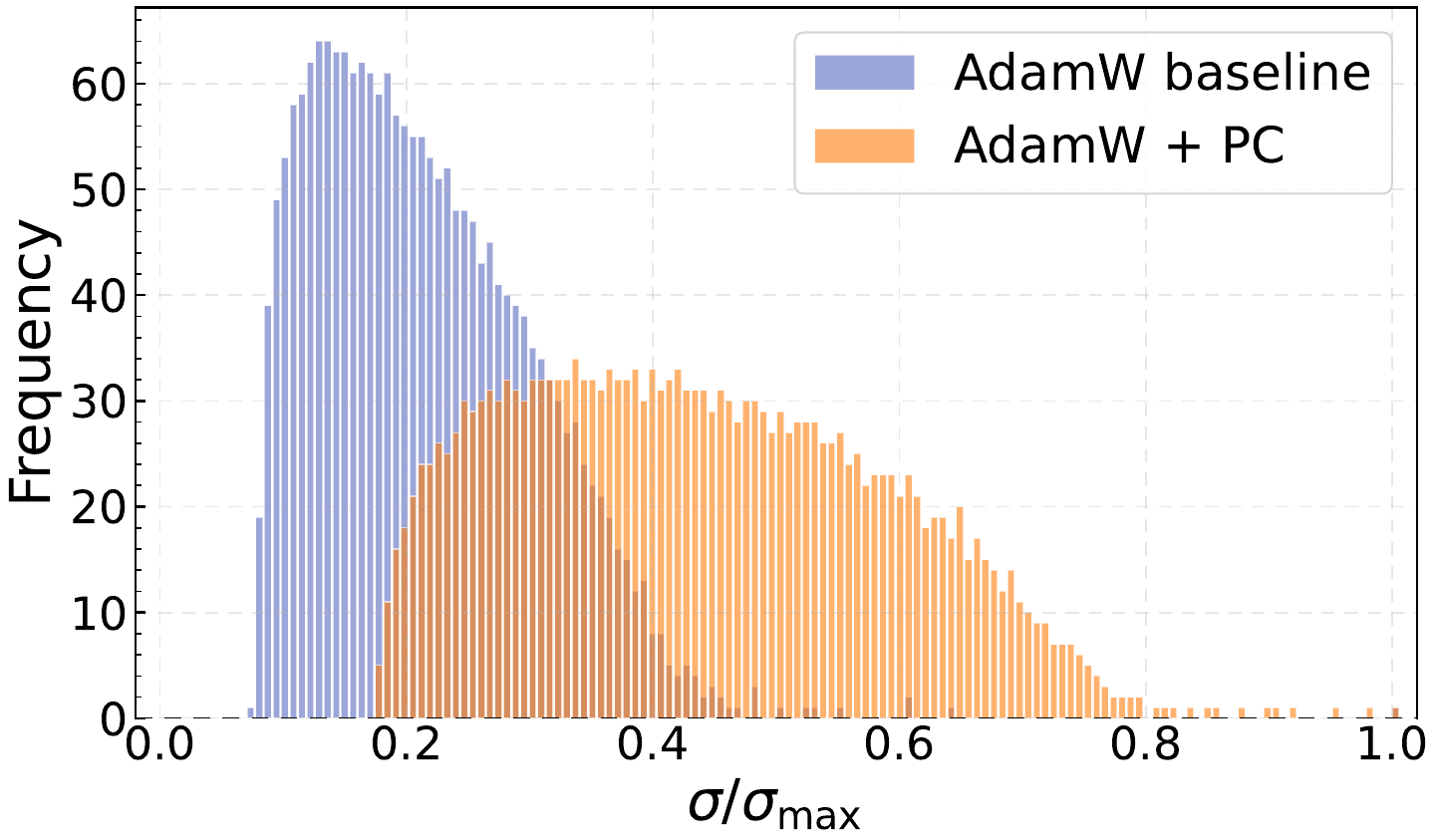}
        \caption{Layer 2: $W_{\rm down}$}
        \label{fig:spec-l2-w2}
    \end{subfigure}

    \vspace{0.6em}

    \begin{subfigure}[t]{0.24\textwidth}
        \centering
        \includegraphics[width=\linewidth]{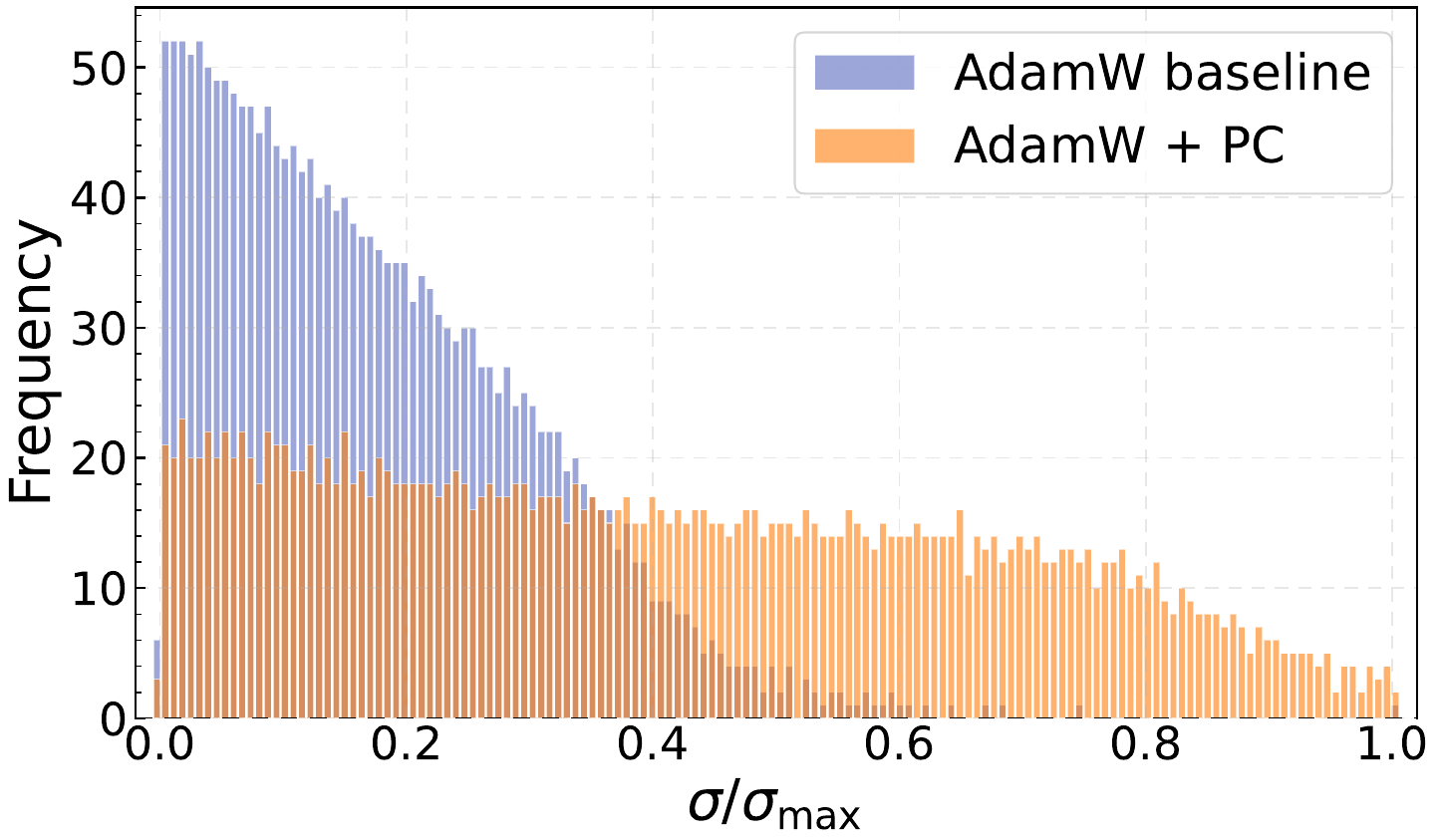}
        \caption{Layer 10: $W_{\rm O}$}
        \label{fig:spec-l10-o}
    \end{subfigure}\hfill
    \begin{subfigure}[t]{0.24\textwidth}
        \centering
        \includegraphics[width=\linewidth]{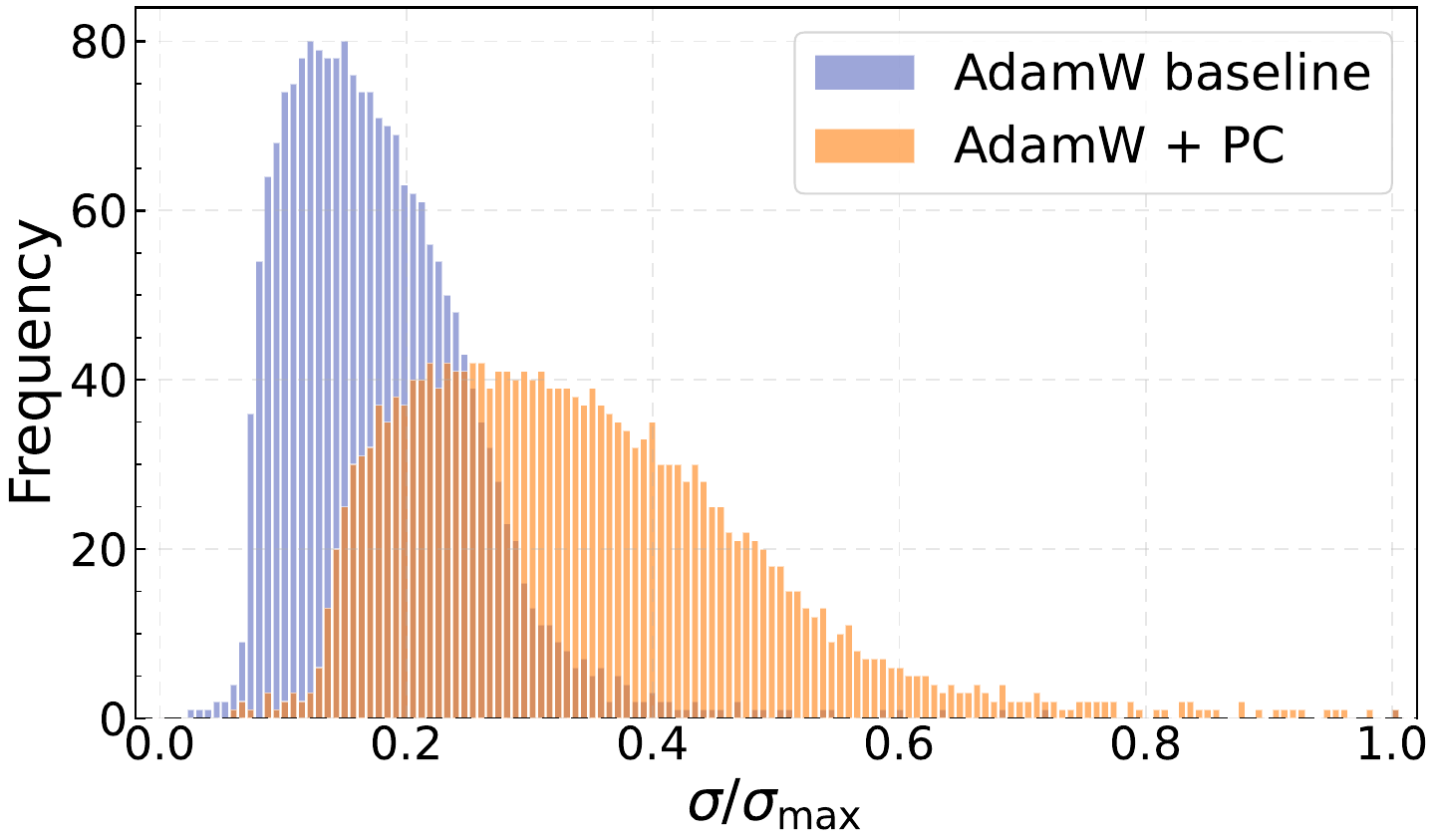}
        \caption{Layer 10: $W_{\rm gate}$}
        \label{fig:spec-l10-w1}
    \end{subfigure}\hfill
    \begin{subfigure}[t]{0.24\textwidth}
        \centering
        \includegraphics[width=\linewidth]{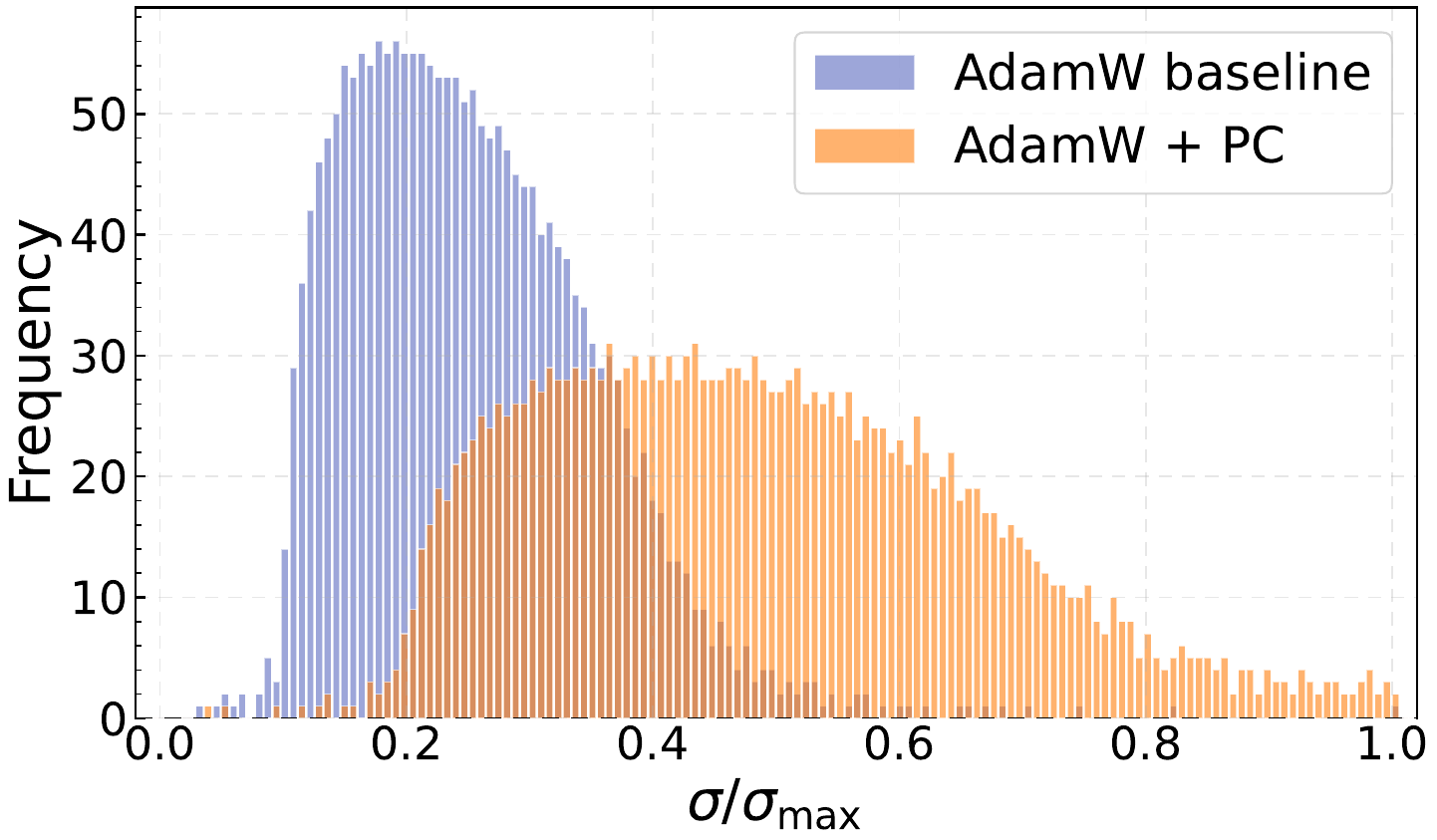}
        \caption{Layer 10: $W_{\rm up}$}
        \label{fig:spec-l10-w3}
    \end{subfigure}\hfill
    \begin{subfigure}[t]{0.24\textwidth}
        \centering
        \includegraphics[width=\linewidth]{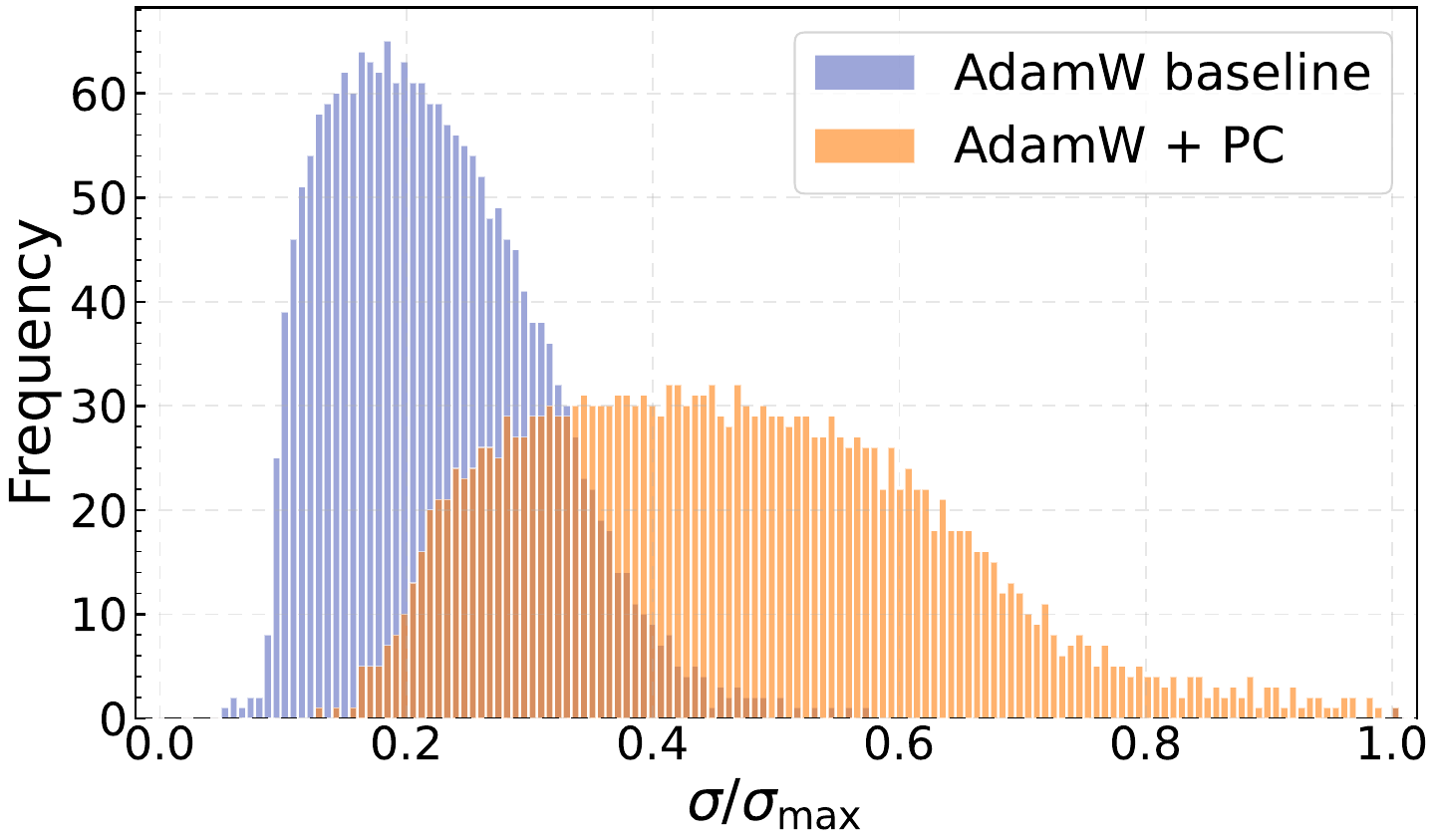}
        \caption{Layer 10: $W_{\rm down}$}
        \label{fig:spec-l10-w2}
    \end{subfigure}

    \vspace{0.6em}

    \begin{subfigure}[t]{0.24\textwidth}
        \centering
        \includegraphics[width=\linewidth]{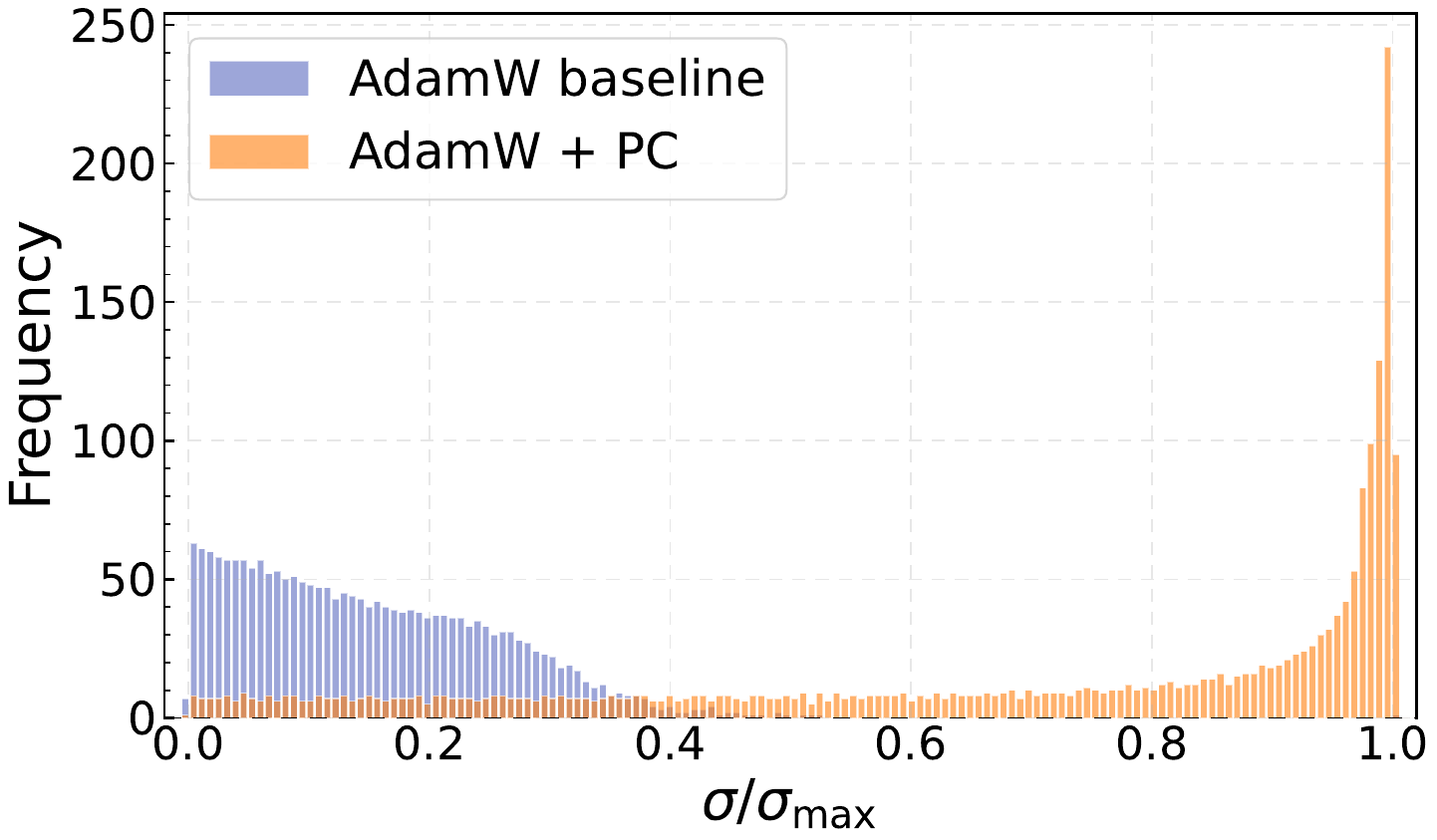}
        \caption{Layer 18: $W_{\rm O}$}
        \label{fig:spec-l18-o}
    \end{subfigure}\hfill
    \begin{subfigure}[t]{0.24\textwidth}
        \centering
        \includegraphics[width=\linewidth]{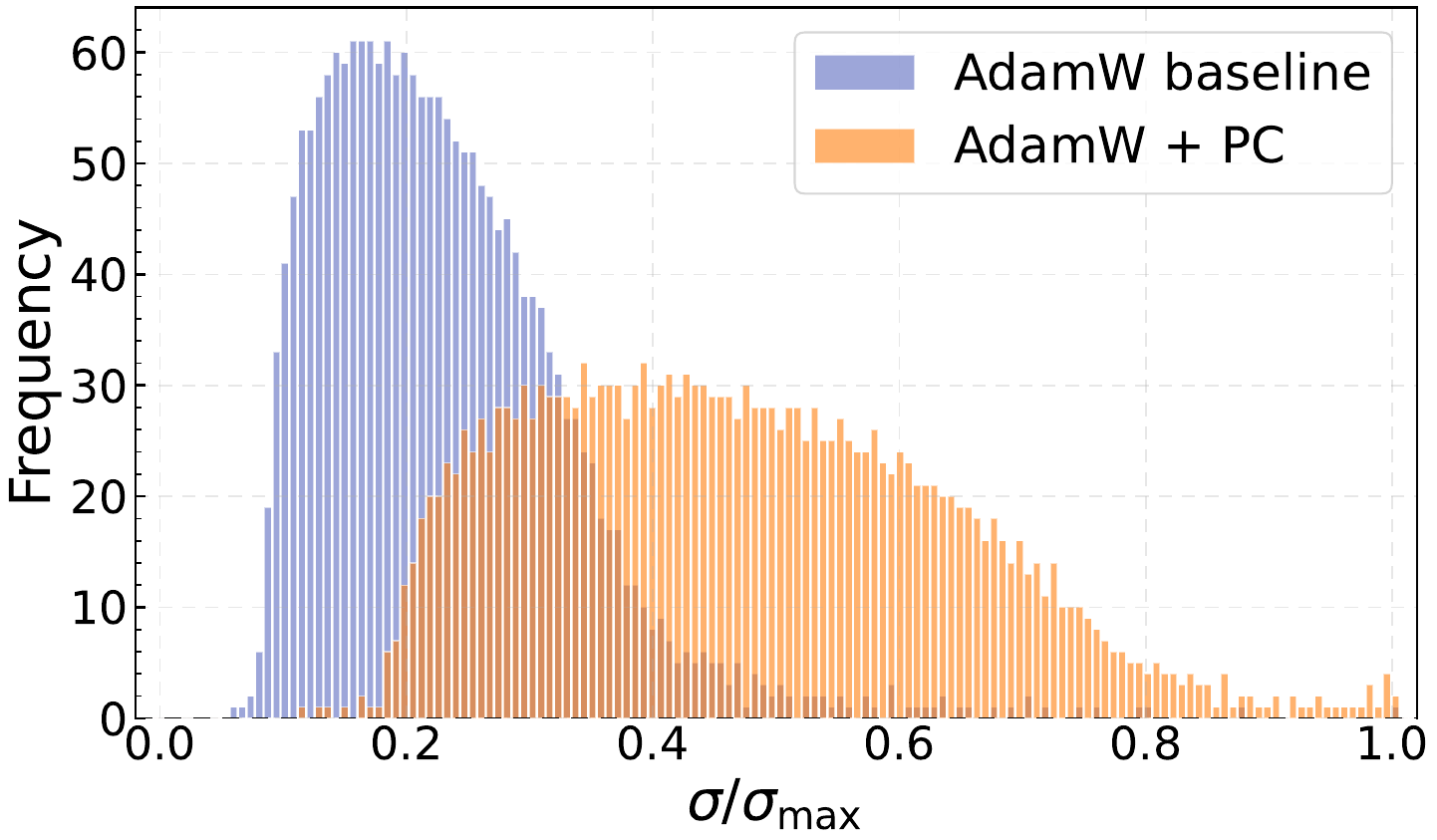}
        \caption{Layer 18: $W_{\rm gate}$}
        \label{fig:spec-l18-w1}
    \end{subfigure}\hfill
    \begin{subfigure}[t]{0.24\textwidth}
        \centering
        \includegraphics[width=\linewidth]{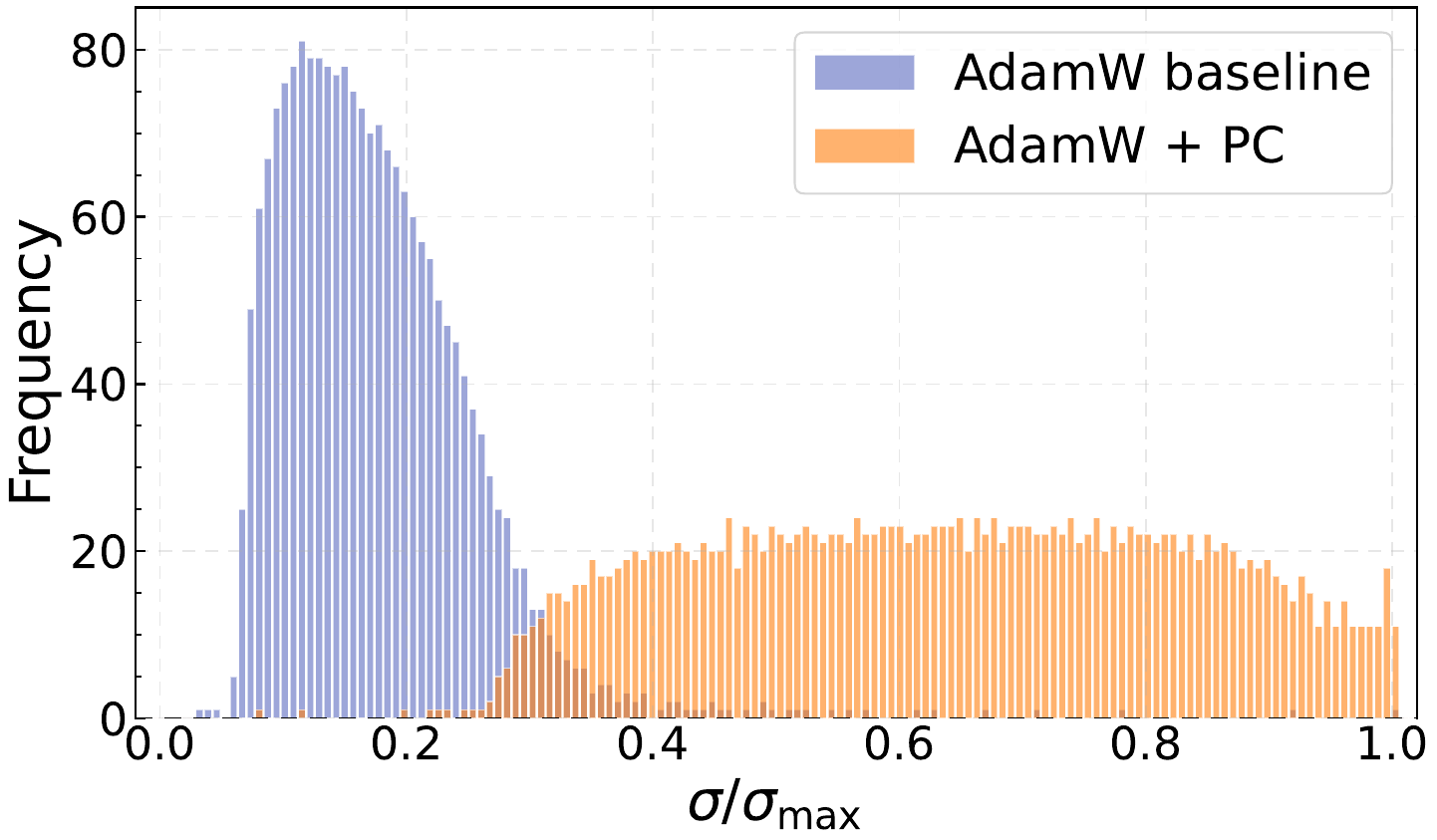}
        \caption{Layer 18: $W_{\rm up}$}
        \label{fig:spec-l18-w3}
    \end{subfigure}\hfill
    \begin{subfigure}[t]{0.24\textwidth}
        \centering
        \includegraphics[width=\linewidth]{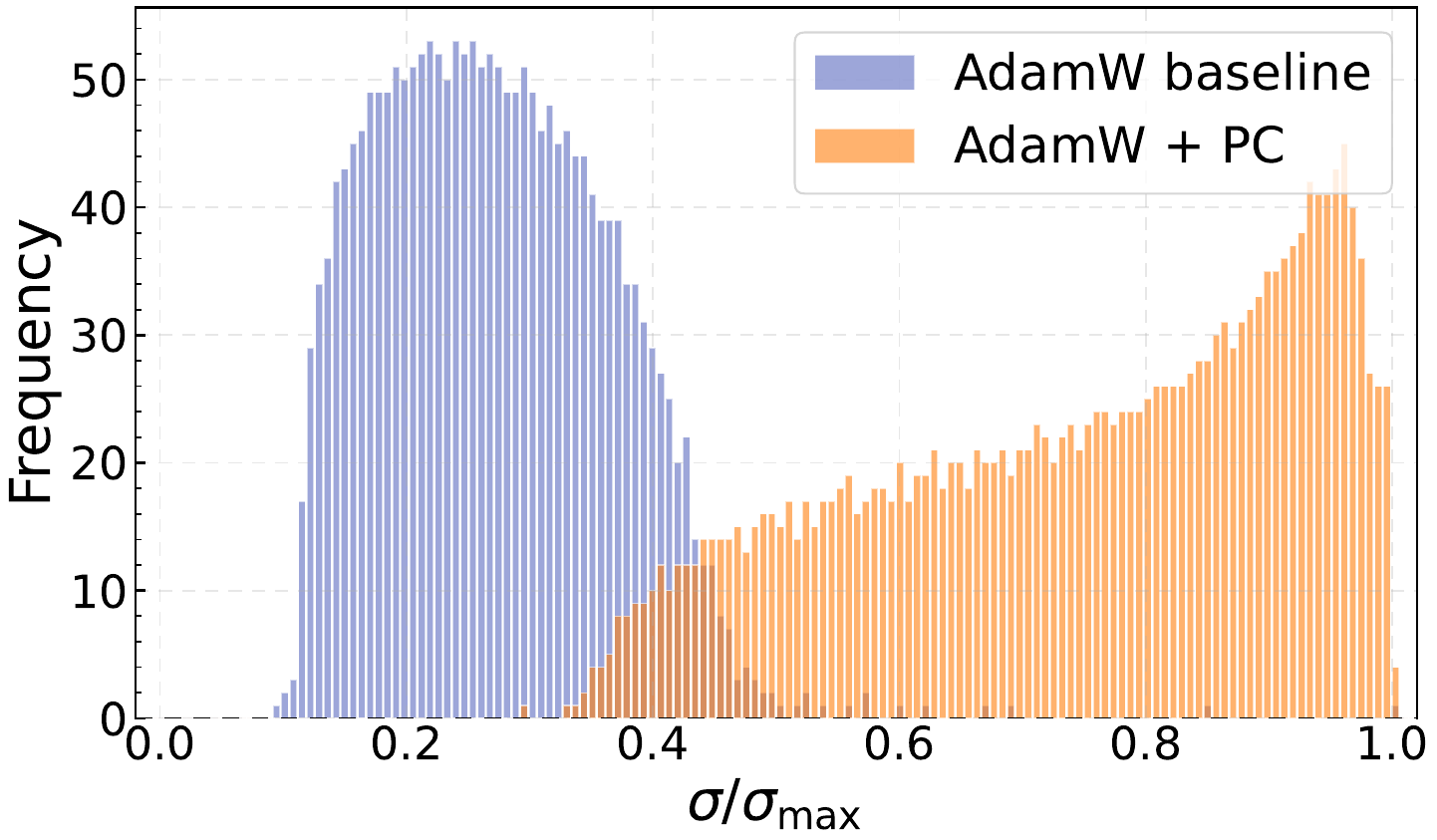}
        \caption{Layer 18: $W_{\rm down}$}
        \label{fig:spec-l18-w2}
    \end{subfigure}

    \caption{\textbf{Singular-value histograms at the final-step checkpoint (AdamW, Llama-1B).}
    We visualize the singular-value spectra for representative layers (2, 10, and 18 out of 18) and \texttt{PC\_blocks} ($W_{\rm O}$, $W_{\rm gate}$, $W_{\rm up}$, $W_{\rm down}$).
    For the baseline, spectra are computed on the original weights $W$ of the baseline-trained model; for PC, spectra are computed on the effective preconditioned matrices $\mathrm{PC}(W)$ from the PC-trained model.
    Within each subplot, singular values are rescaled by the largest singular value of that matrix so that all spectra lie in $[0, 1]$, which makes the spectral \emph{shape} directly comparable across blocks and depths; the $y$-axis reports the bin frequency (raw counts per bin).
    Across depths and blocks, PC moves more singular-value mass away from the lower end of the normalized spectrum and into a more regular middle range.}
    \label{fig:pc-sv-hist-adamw}
\end{figure}

\paragraph{Singular-value spectrum comparison.}
To further investigate how preconditioning reshapes the full singular-value spectrum across individual blocks and depths, we examine singular-value distributions at the final checkpoint. For the transformer baseline, we compute singular values from the original weights $W$; for the PC run, we compute singular values of the effective (preconditioned) weight matrices,
i.e., $\mathrm{PC}(W)$. To probe depth-wise behavior while avoiding boundary effects, we report three representative layers in our 18-layer Llama-1B model: layer 2 (shallow), layer 10 (middle), and layer 18 (deep). For each selected layer, we plot singular-value histograms (with bin frequencies on the $y$-axis) for all \texttt{PC\_blocks}: $W_{\rm O}$, $W_{\rm gate}$, $W_{\rm up}$ and $W_{\rm down}$, where each block's singular values are first rescaled by their per-matrix maximum so that all spectra share the support $[0,1]$, making the spectral shapes directly comparable (see Figure~\ref{fig:pc-sv-hist-adamw}).

Across layers and blocks, the AdamW baseline places substantial histogram mass near the lower end of the normalized singular-value spectrum. PC moves more of this mass into the middle range. The spectra are therefore less concentrated at the lower end and have a narrower relative spread. The dominant visible effect is that PC lifts the lower part of the normalized spectrum and narrows the relative spread, consistent with the intended soft spectrum-shaping behavior of the PC polynomial and with the lower modified condition numbers reported above.

\section{Ablation Studies}
\label{subsec:ab_study}
To better understand the contributions of individual components within the PC module, we conduct ablation studies. Unless otherwise specified, ablations are conducted on the Llama-271M model. The \texttt{pc\_level} and PC-block ablations report results under both AdamW and Muon, while the norm-recovery and learnable-$\gamma$ ablations focus on the AdamW setting. We examine four aspects in order: the polynomial degree (\texttt{pc\_level}), the choice of PC blocks, the effect of norm recovery, and the use of the learnable scalar $\gamma$.

Unless otherwise specified, each ablation varies only the factor under study and keeps all remaining settings at their default values described in Section~\ref{subsec:exp_setting}. The default PC configuration is $\texttt{PC\_blocks}=\{\texttt{ffn}, W_{\rm O}\}$ and spectral-norm normalization with $10$ power-iteration steps, with $\texttt{pc\_level}=4$ for AdamW and $\texttt{pc\_level}=2$ for Muon. The no-PC baseline has a final validation loss of \textbf{2.5944} under AdamW and \textbf{2.4831} under Muon, which serve as the references in this section.

\subsection{PC Level Selection}
Recall that \texttt{pc\_level} is the degree \(k\) of \(p_k\), equivalently the strength knob of the PC map: larger values use a higher-degree, more aggressive spectral shaping (Table~\ref{tab:pc_level}). We evaluate different PC levels by sweeping $\texttt{pc\_level} \in \{1, 2, 3, 4\}$ under the spectral-norm normalizer using $10$ power-iteration steps. Figure~\ref{fig:ab-pclevel-adamw} reports the final validation loss against the AdamW baseline ($2.5944$). $\texttt{pc\_level}=1$ is slightly worse than the baseline, and increasing the degree consistently improves validation performance, with $\texttt{pc\_level}=4$ achieving the lowest loss ($2.5397$). Within the tested range \(k\in\{1,2,3,4\}\), this monotone trend indicates that stronger polynomial preconditioning is more effective at reshaping the weight spectrum under AdamW.

Under Muon (with the same spectral-norm normalizer), this monotone trend reverses: sweeping the same range, $\texttt{pc\_level}=2$ achieves the lowest final validation loss and higher degrees degrade performance (Figure~\ref{fig:ab-pclevel-muon}). Intuitively, Muon already exerts an implicit form of spectral control on the update matrices, so the residual spectral shaping required from PC is smaller and a lower-degree polynomial suffices; pushing $\texttt{pc\_level}$ higher risks over-constraining the weight spectrum. We therefore adopt $\texttt{pc\_level}=2$ as the default under Muon, and discuss the mechanism behind this difference in Appendix~\ref{subapp:supp_disc}.

\begin{figure}[htbp]
  \centering
  \begin{subfigure}[t]{0.48\textwidth}
    \centering
    \includegraphics[width=0.85\linewidth]{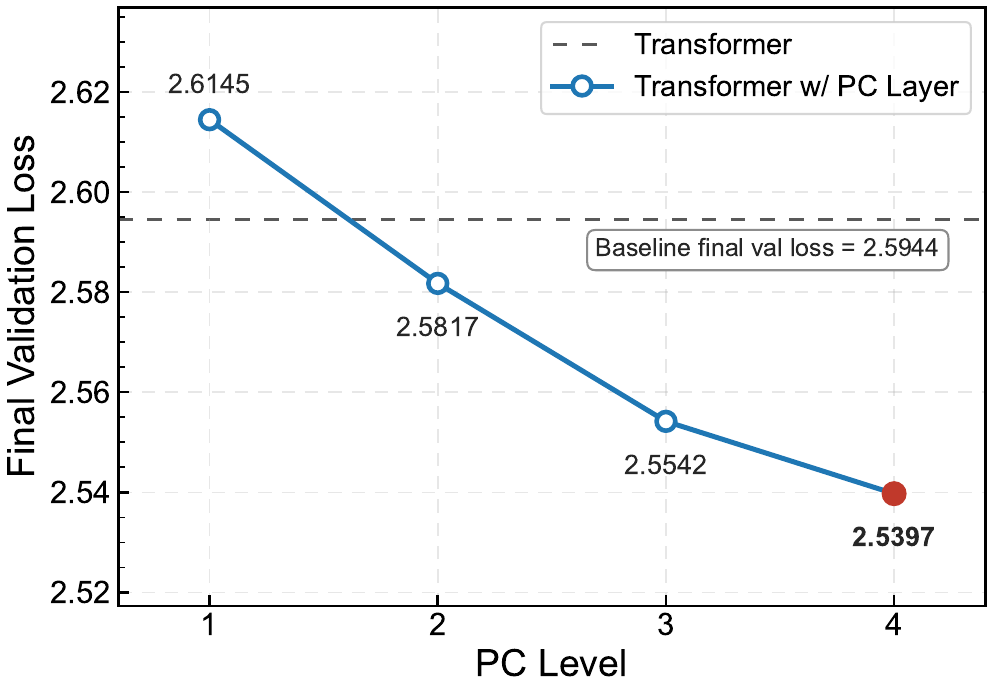}
    \caption{\textbf{AdamW.} $\texttt{pc\_level}=4$ is best.}
    \label{fig:ab-pclevel-adamw}
  \end{subfigure}\hfill
  \begin{subfigure}[t]{0.48\textwidth}
    \centering
    \includegraphics[width=0.85\linewidth]{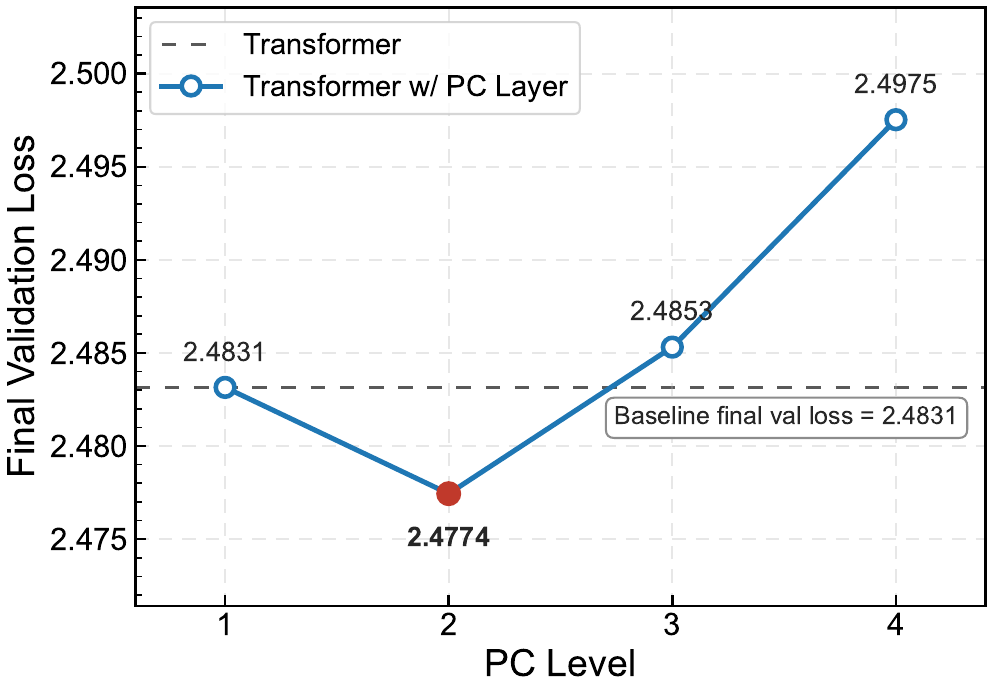}
    \caption{\textbf{Muon.} Loss is non-monotone; $\texttt{pc\_level}=2$ is best.}
    \label{fig:ab-pclevel-muon}
  \end{subfigure}
  \caption{\textbf{\texttt{pc\_level} ablation.}
  Sweeping $\texttt{pc\_level}\in\{1,2,3,4\}$ with $\{\texttt{ffn},\ W_{\rm O}\}$ fixed. (a)~AdamW: larger degrees help. (b)~Muon: non-monotone, $\texttt{pc\_level}=2$ optimal.}
  \label{fig:ab-pclevel}
\end{figure}

\subsection{Choice of PC Blocks}
\label{para:ab_pc_blocks}

Figure~\ref{fig:ab-blocks} studies which transformer weight blocks should be equipped with PC.
We compare configurations obtained by adding different attention-side projections to \texttt{ffn}, with candidates
$\{(W_{\rm Q},W_{\rm K}), W_{\rm O}, W_{\rm V}\}$, where
$(W_{\rm Q},W_{\rm K})$ denotes applying PC jointly to the query and key projections.
We keep \texttt{ffn} enabled in all variants.

Under AdamW, all tested PC variants improve over the no-PC baseline, and our default choice $\texttt{PC\_blocks}=\{\texttt{ffn}, W_{\rm O}\}$ is among the best-performing configurations, reducing the baseline loss by $5.47\times 10^{-2}$.
Although several configurations containing $(W_{\rm Q},W_{\rm K})$ achieve slightly lower validation loss,
their advantage over our default is relatively small, within $6.6 \times 10^{-3}$ in final validation loss, while requiring PC on two
additional attention projections in each transformer block.
Under Muon, the picture is more selective (Figure~\ref{fig:ab-blocks-muon}): several configurations still improve and $\{\texttt{ffn},\ W_{\rm O}\}$ ties for the best-performing block choice, but adding certain attention-side projections can degrade performance relative to the Muon baseline. Since $\{\texttt{ffn},\ W_{\rm O}\}$ is among the best configurations under AdamW and tied for best under Muon, we adopt this simpler and more efficient configuration as the default across optimizers.

\begin{figure}[htbp]
  \centering
  \begin{subfigure}[t]{0.48\textwidth}
    \centering
    \includegraphics[width=\linewidth]{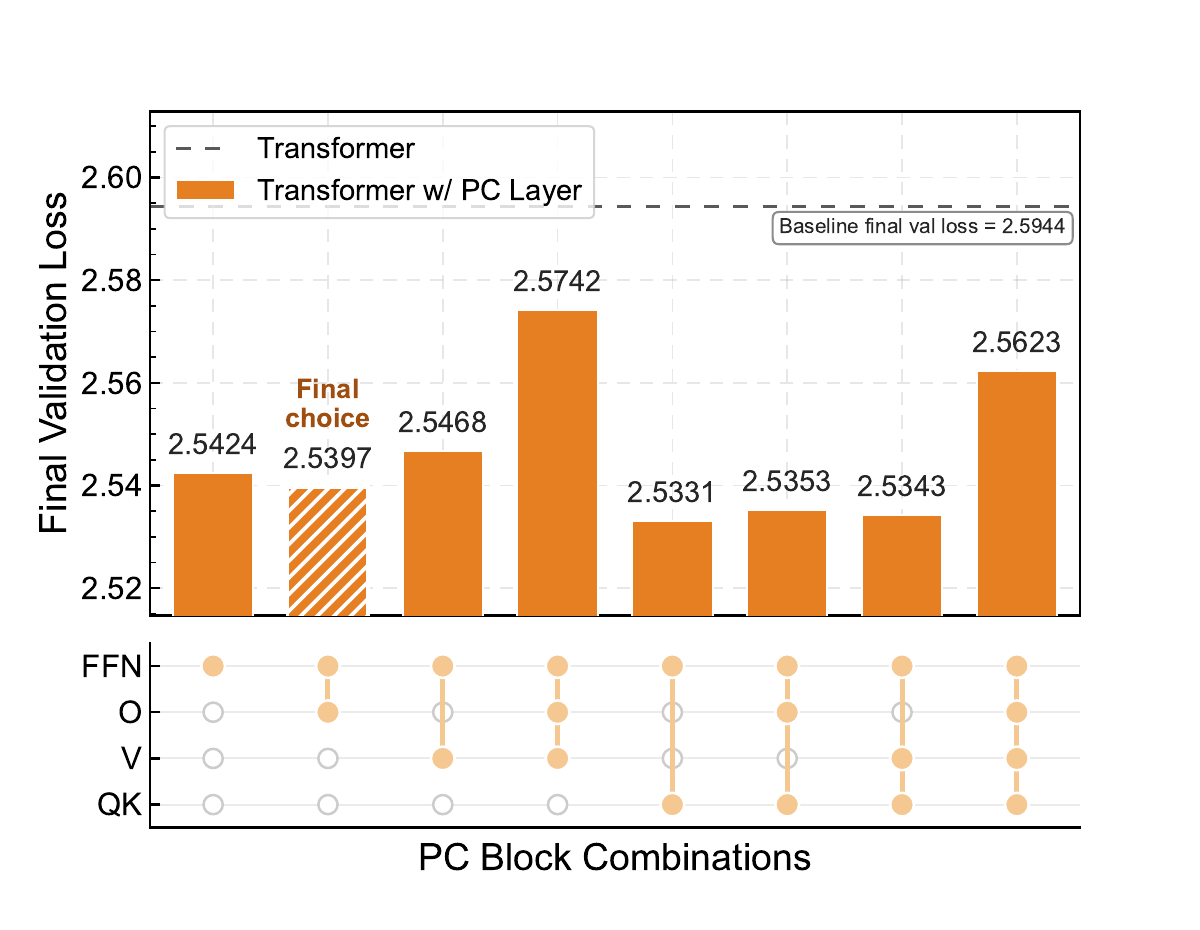}
    \caption{\textbf{AdamW.}}
    \label{fig:ab-blocks-adamw}
  \end{subfigure}\hfill
  \begin{subfigure}[t]{0.48\textwidth}
    \centering
    \includegraphics[width=\linewidth]{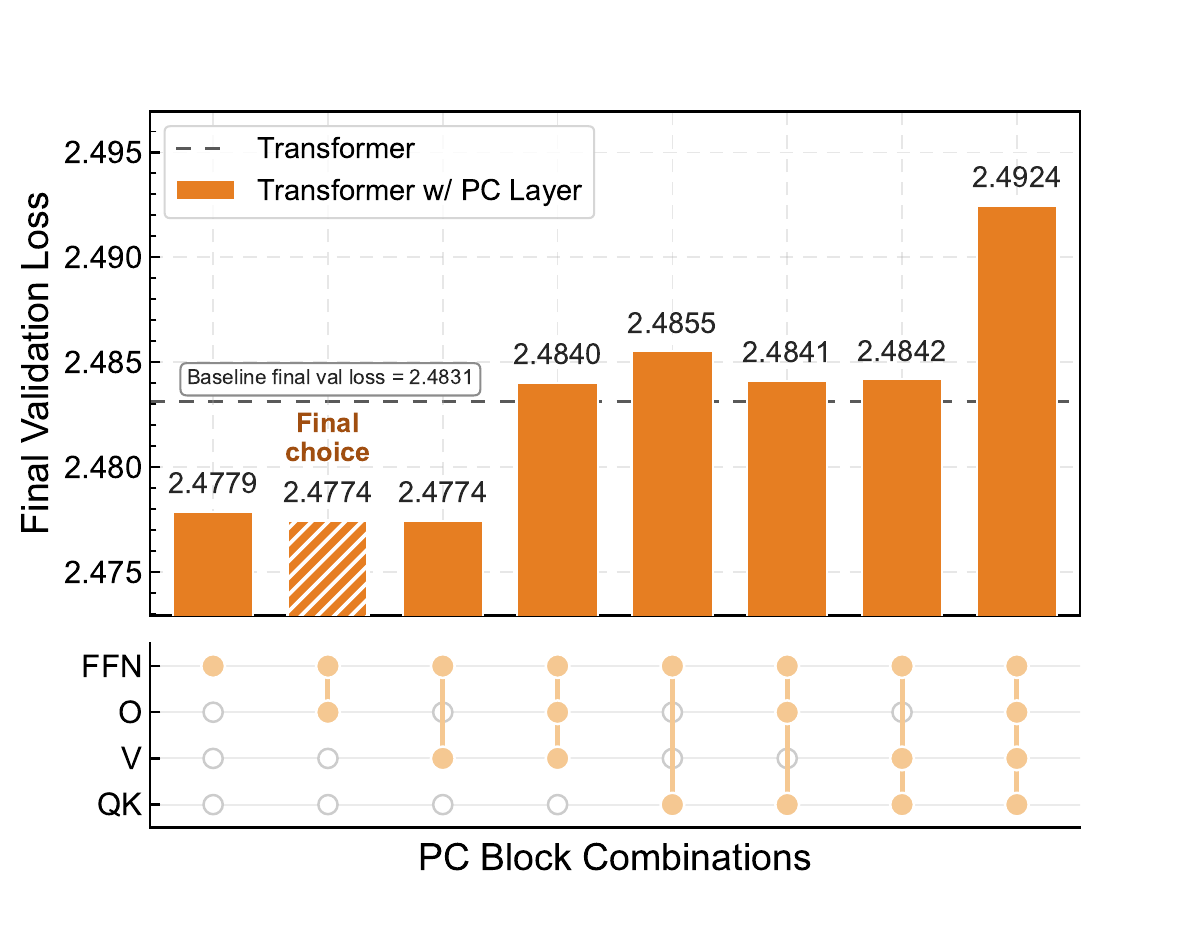}
    \caption{\textbf{Muon.}}
    \label{fig:ab-blocks-muon}
  \end{subfigure}
  \caption{\textbf{PC block selection ablation.}
  All variants include \texttt{ffn}; the hatched bar (``Final choice'') marks our default $\{\texttt{ffn},\ W_{\rm O}\}$, which is among the best under both (a)~AdamW and (b)~Muon.}
  \label{fig:ab-blocks}
\end{figure}

\subsection{Effect of Weight Norm Recovery After Preconditioning}
\label{subsec:ab_study-rcv}
Recall from Algorithm~\ref{alg:pc} that PC first applies the polynomial to a spectrally normalized weight matrix. After this spectrum-shaping step, the PC module can further rescale the result in two ways: by recovering the spectral scale removed during normalization, and by applying a learnable scalar $\gamma$. We ablate these two components jointly, giving four variants with or without norm recovery and with or without $\gamma$.

Table~\ref{tab:ab-rescale} reports the final-loss difference of these variants relative to the baseline, and Figure~\ref{fig:ab-rescale} shows the corresponding validation-loss curves.

The dominant factor is norm recovery. Without norm recovery, both PC variants perform worse than the transformer baseline, regardless of whether $\gamma$ is used. With norm recovery, the same polynomial spectrum shaping consistently improves the final validation loss. Thus, applying the polynomial to the normalized weight alone is not sufficient; restoring the spectral scale after shaping is essential for turning PC into a consistent gain. The learnable scalar $\gamma$ has a smaller effect on the final validation loss in this ablation, and its main role in stabilizing training dynamics is examined in the next subsection.

\begin{figure*}[htbp]
  \centering

  \begin{minipage}[b]{0.48\textwidth}
    \centering
    \renewcommand{\arraystretch}{1.3}
    \setlength{\tabcolsep}{10pt}
    \begin{tabular}{lcc}
      \toprule
      \textbf{Setting} & \textbf{w/o $\gamma$} & \textbf{w/ $\gamma$} \\
      \midrule
      w/o norm recovery & $+0.0447$  \textcolor{red}{$\uparrow$} & $+0.0430$  \textcolor{red}{$\uparrow$} \\
      w/ norm recovery & $-0.0480$  \textcolor{blue}{$\downarrow$} & \bm{$-0.0547$}  \textcolor{blue}{$\downarrow$} \\
      \bottomrule
    \end{tabular}

    \vspace{4.5em}

    \captionof{table}{\textbf{Ablation on norm recovery and learnable $\gamma$.} Loss difference denotes final validation loss minus the baseline final validation loss ($2.5944$). \textcolor{blue}{$\downarrow$} indicates an improvement over the baseline, and \textcolor{red}{$\uparrow$} indicates worse performance than the baseline. The best result is \textbf{bolded}.}
    \label{tab:ab-rescale}
  \end{minipage}
  \hfill
  \begin{minipage}[b]{0.48\textwidth}
    \centering
    \includegraphics[width=\linewidth]{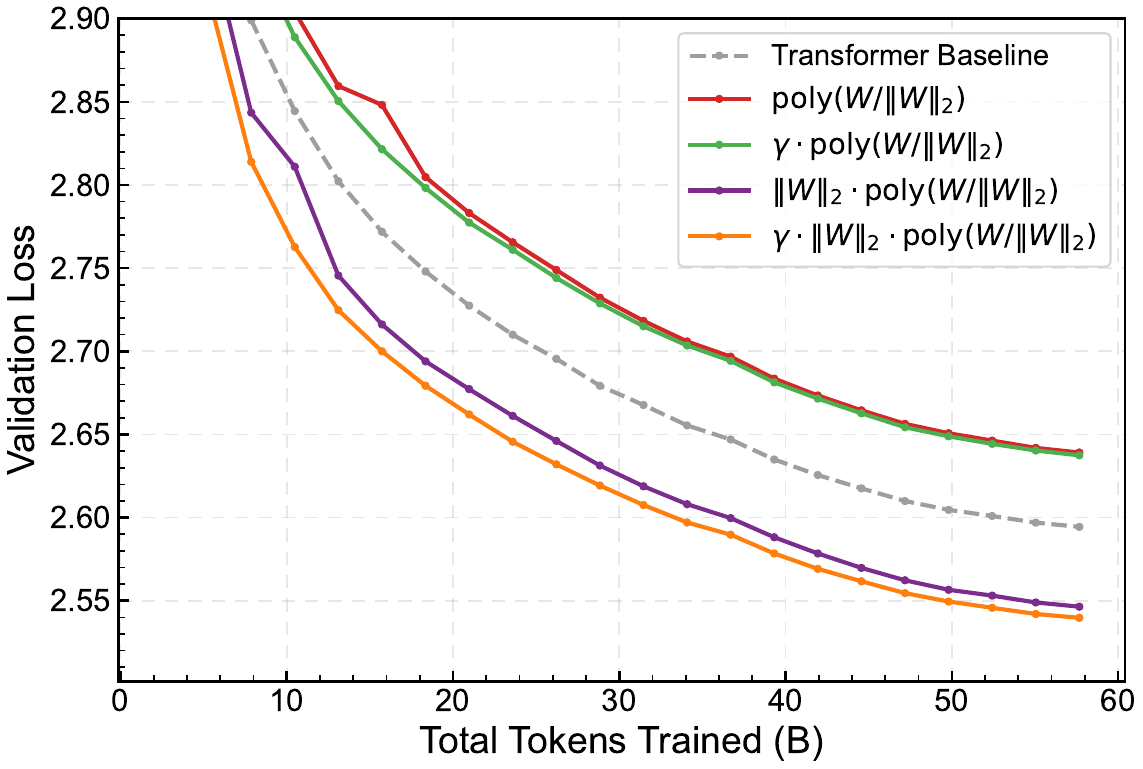}
    \captionof{figure}{\textbf{Validation-loss curves for the norm-recovery ablation.} The curves correspond to the four variants in Table~\ref{tab:ab-rescale}. Norm recovery consistently moves PC from worse-than-baseline performance to clear improvement.}
    \label{fig:ab-rescale}
  \end{minipage}

\end{figure*}

\subsection{The Role of Learnable $\gamma$ in the PC Module}
\label{subsec:ab_study-gamma}

We next isolate the role of the learnable scalar $\gamma$. Switching $\gamma$ on or off changes the final validation loss only mildly compared with the change caused by norm recovery. Instead, its main effect appears in the training dynamics. Here we track the activation root mean square (RMS), computed on the outputs of the full Attention and FFN submodules, after the output projections and before residual addition. As shown by the RMS curves in Figure~\ref{fig:ab-gamma-stable-sp}, $\gamma$ stabilizes signal propagation.

At the 1B scale, as shown in Figure~\ref{fig:ab-gamma-stable-sp}, the run without $\gamma$ exhibits sizable spikes in attention RMS and FFN RMS during training, whereas the $\gamma$-enabled run remains much smoother. We therefore adopt $\gamma$ to stabilize training signal propagation.

\begin{figure*}[htbp]
  \centering
  \begin{subfigure}[t]{0.49\textwidth}
    \centering
    \includegraphics[width=\linewidth]{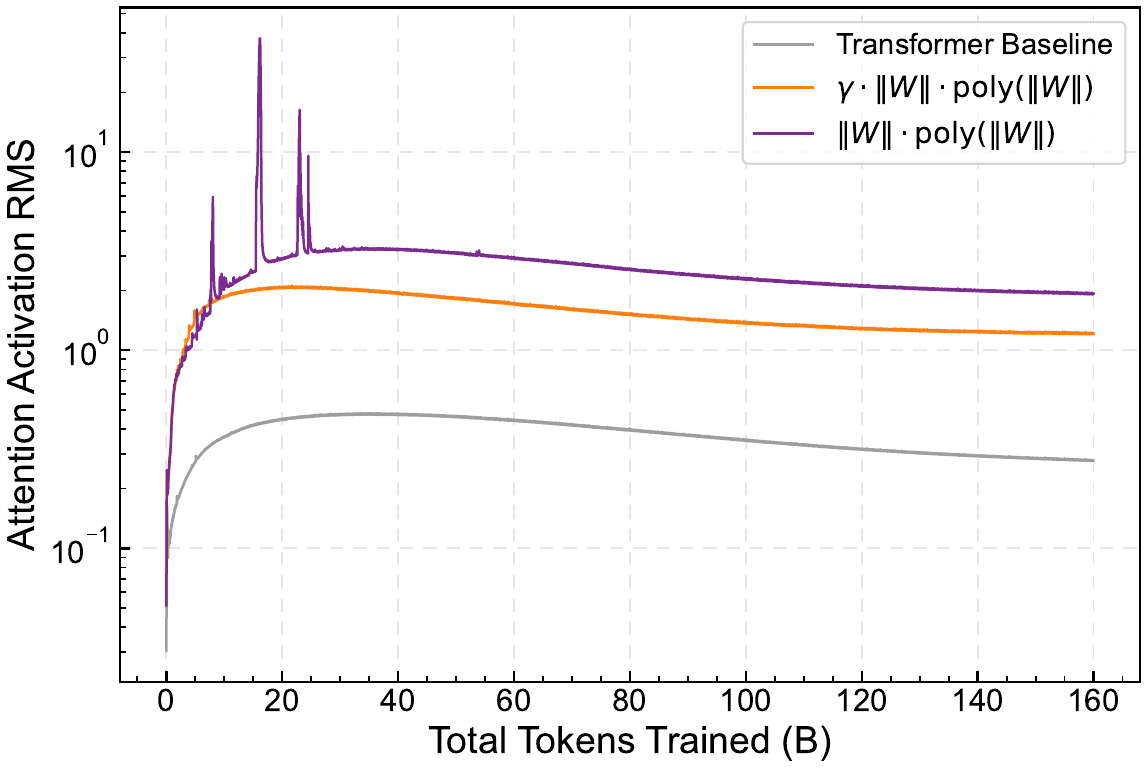}
    \caption{\textbf{Cross-layer average attention activation RMS.}}
    \label{fig:ab-gamma-stable-attnrms}
  \end{subfigure}
  \hfill
  \begin{subfigure}[t]{0.49\textwidth}
    \centering
    \includegraphics[width=\linewidth]{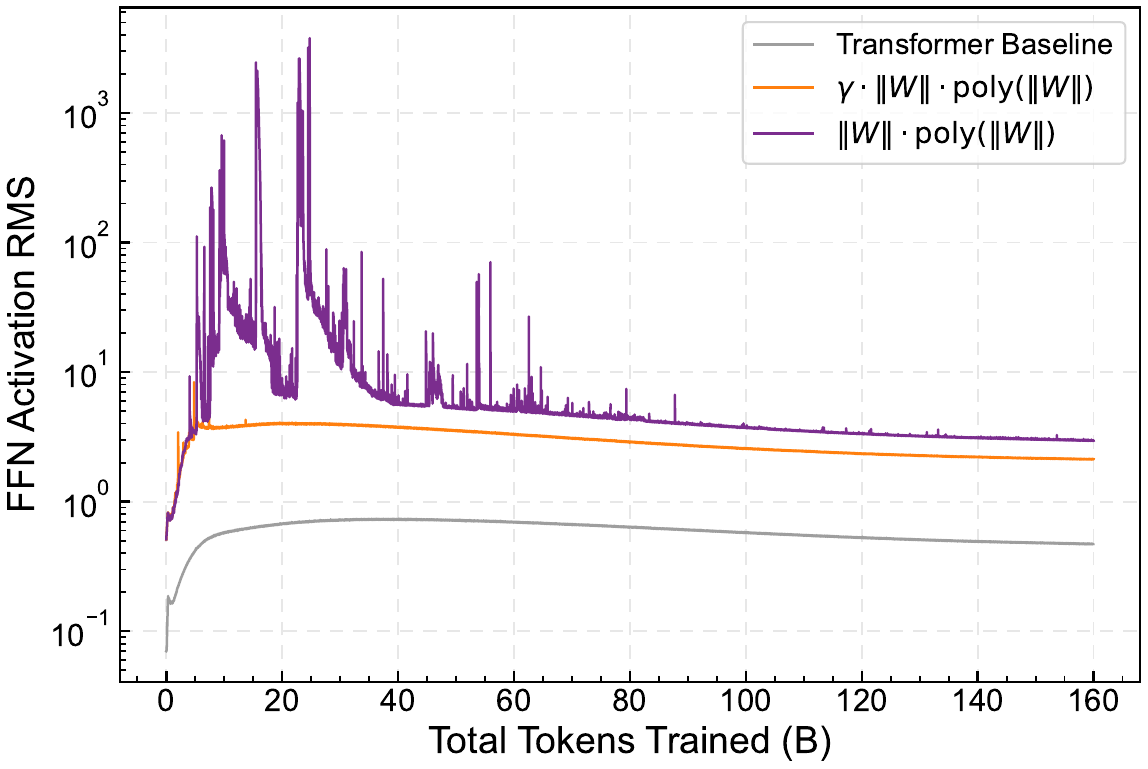}
    \caption{\textbf{Cross-layer average FFN activation RMS.}}
    \label{fig:ab-gamma-stable-ffnrms}
  \end{subfigure}
  \caption{\textbf{Learnable $\gamma$ stabilizes activation RMS at the 1B scale.} The run without $\gamma$ shows many spikes in attention RMS and FFN RMS, while the run with $\gamma$ remains stable.}
  \label{fig:ab-gamma-stable-sp}
\end{figure*}

\section{Theory: Why Controlling the Spectrum?}
\label{sec:theory-spectrum}

We provide a theoretical justification for the spectrum-control principle underlying PC layers by analyzing \textit{deep linear networks}.
Specifically, if all weight matrices maintain uniformly bounded condition numbers during training, then the Jacobian is well-conditioned, leading to geometric convergence of gradient descent.

\paragraph{Settings of theoretical analysis.}
We consider an $L$-layer linear network that maps $x\in\mathbb{R}^{d_x}$ to
\[
F(\theta;x) \;=\; W_L W_{L-1}\cdots W_1 x \in \mathbb{R}^{d_y},
\]
where $\theta=(W_1,\dots,W_L)$, and $W_l\in\mathbb{R}^{d_l\times d_{l-1}}$ for $l \in\{1,\dots,L\}$,
with $d_0=d_x$ and $d_L=d_y$.
We assume a pyramidal architecture as follows.
\begin{asmp}[Pyramidal architecture]
\label{asmp:pyramid}
There exists some $r \in \{1, \dots, L\}$, such that
\[
d_0 \le d_1 \le \cdots \le d_r
\qquad\text{and}\qquad
d_r \ge d_{r+1} \ge \cdots \ge d_L .
\]
\end{asmp}
Suppose that we have $n$ training samples. We denote $X=[x_1,\dots,x_n]\in\mathbb{R}^{d_x\times n}$ as the input, $F(\theta;X)=(F(\theta;x_1);\dots;F(\theta;x_n))\in\mathbb{R}^{n d_y}$ as the stacked model predictions, and $y=(y_1;\dots;y_n)\in\mathbb{R}^{n d_y}$ as the stacked labels.
We further impose the following standard assumption on the input data matrix.
\begin{asmp}[Non-degenerate input data]
\label{asmp:full-col-rank}
The input matrix $X\in\mathbb{R}^{d_x\times n}$ has full column rank, i.e., $\sigma_{\min}(X) > 0$.
\end{asmp}
Assumption~\ref{asmp:full-col-rank} implicitly places us in the over-parameterized regime, i.e., $n \le d_x$.
See Remark \ref{rem:linear-nn} in Appendix~\ref{app:theory} for further discussion of the linear-network setting; see Remark \ref{rem:pyramidal} in Appendix~\ref{app:theory} for the roles of Assumption~\ref{asmp:pyramid} and \ref{asmp:full-col-rank} in our theoretical analysis.

We consider the squared loss of the training data:
\[
\mathcal{L}(\theta) = \frac12 \sum_{i=1}^n \lVert F(\theta; x_i) - y_i \rVert^2 = \frac{1}{2} \lVert F(\theta;X)-y \rVert^2.
\]
We analyze gradient descent applied to the loss \(\mathcal{L}(\theta)\).
Given an initial parameter \(\theta(0)\), the iterates are updated as
\[
\theta(t+1) = \theta(t) - \eta \, \nabla \mathcal{L}(\theta(t)),
\]
where \(\eta > 0\) denotes the learning rate.

For given $\{(\tau_l,\mu_l)\}_{l=1}^L$ with $\tau_l \ge 1 \ge \mu_l>0$, we define the region
\[
\mathcal{R}
:=\Bigl\{\theta=(W_1,\dots,W_L)\; \Big|\;
\sigma_{\max}(W_l)\le \tau_l,\;\; \sigma_{\min}(W_l)\ge \mu_l,\;\; \forall l\in \{ 1,2,\dots,L\}\Bigr\}.
\]
Within $\mathcal{R}$, each layer weight matrix is well-conditioned: its singular values are bounded
away from both $0$ and $\infty$. The next theorem shows that, as long as the training trajectory stays inside
$\mathcal{R}$, gradient descent achieves a geometric decrease.

\begin{thm}[Geometric convergence within $\mathcal{R}$]
\label{thm:geom-cvrg}
Suppose the iterates satisfy $\theta(t)\in\mathcal{R}$ for all $t\in\{0,1,\dots,T\}$.
Define
\[
\beta 
:= \left( \prod_{l=1}^L \tau_l \right)^2 \left( \sqrt{2 \mathcal{L}(\theta(0))} + \| X \|_F  \right) L \sigma_{\max}(X),
\qquad 
\mu 
:= \left( \prod_{l=1}^L \mu_l \right)^2 \sigma_{\min}(X)^2 .
\]
Then for any learning rate $\eta \in \bigl(0, 1/\beta\bigr]$, it holds that
\[
\mathcal{L}(\theta(t+1)) \le \bigl(1-\eta\mu\bigr)\, \mathcal{L}(\theta(t)),
\qquad \forall\, t\in\{0,1,\dots,T\}.
\]
In particular, choosing $\eta=1/\beta$ yields the contraction factor $1-\mu/\beta$:
\[
\mathcal{L}(\theta(t+1)) \le \Bigl(1-\frac{\mu}{\beta}\Bigr)\, \mathcal{L}(\theta(t)),
\qquad \forall\, t\in\{0,1,\dots,T\}.
\]
\end{thm}

The geometric decrease in Theorem~\ref{thm:geom-cvrg} is closely related to \textbf{global convergence}. 
If the entire training trajectory satisfies $\theta(t)\in\mathcal{R}$ for all $t\ge 0$, then iterating the one-step contraction yields
\[
\mathcal{L}(\theta(t)) \le (1-\eta\mu)^t\,\mathcal{L}(\theta(0)) \xrightarrow[t\to\infty]{} 0.
\]

\paragraph{Proof sketch.}
The argument follows a direct implication chain:
(i) For $\theta\in\mathcal{R}$, the weight matrices stay well-conditioned (each layer satisfies
$\sigma_{\min}(W_l)\ge \mu_l$ and $\sigma_{\max}(W_l)\le \tau_l$), which implies that the \textit{Jacobian
matrix \(G(\theta)\)} is well-conditioned, with its spectrum bounded as $\mu I \preceq G(\theta)^\top G(\theta) \preceq\beta I$;
(ii) The above spectral bounds imply a Polyak–Łojasiewicz (PL) inequality and (local) smoothness of the loss function; (iii) under the PL inequality and smoothness condition, a standard convergence proof for gradient descent yields the geometric decrease of the squared loss.
See Appendix~\ref{app:theory} for the full proof and a detailed discussion.

\paragraph{Dependence on global conditioning.}
Theorem~\ref{thm:geom-cvrg} reveals a direct relationship between the weight conditioning within $\mathcal{R}$ and the convergence rate $1 - \mu/\beta$.
Let $\kappa(W_l):=\sigma_{\max}(W_l)/\sigma_{\min}(W_l)$ denote the layer-wise condition number, and define the global condition-number bound over $\mathcal{R}$ by
\[
\kappa_{\mathcal{R}}
:=\Bigl(\prod_{l=1}^L \frac{\tau_l}{\mu_l}\Bigr)^{1/L}.
\]
Then $\beta/\mu = C\,\kappa_{\mathcal{R}}^{2L}$, where $C$ depends only on the data $(X,y)$ and the initialization $\theta(0)$, but not on the layer-wise condition-number bounds.
Hence better weight conditioning (smaller $\kappa_{\mathcal{R}}$) directly translates to a faster geometric convergence rate.

\begin{cor}[Iteration complexity]
\label{cor:iter}
Under the assumptions of Theorem~\ref{thm:geom-cvrg}, take $\eta=1/\beta$. Then for all $t\in\{0,1,\dots,T\}$,
\[
\mathcal{L}(\theta(t)) \le \Bigl(1-\frac{\mu}{\beta}\Bigr)^t \mathcal{L}(\theta(0)).
\]
Consequently, reaching $\mathcal{L}(\theta(T)) \le \epsilon$ is guaranteed after
\[
T = O\!\left( \kappa_{\mathcal{R}}^{2L} \log\!\left(\frac{\mathcal{L}(\theta(0))}{\epsilon}\right) \right)
\]
iterations within the region $\mathcal{R}$.
\end{cor}

Corollary~\ref{cor:iter} implies that smaller layer-wise upper bounds $\tau_\ell/\mu_\ell$ lead to a smaller global conditioning bound $\kappa_{\mathcal{R}}$, and hence a smaller upper bound on the iteration complexity $T$ required to reach a given error level (see Appendix~\ref{app:cor:iter-proof} for the proof).
This result provides theoretical support for the design principle underlying PC layers: 
\textbf{better weight conditioning during training leads to faster convergence}.
The preceding sections have instantiated this principle via polynomial weight preconditioning.

\section{Related Work}
\label{sec:related}

\paragraph{Weight control during NN training.}

Several prior works improve optimization by controlling the weight parameterization or statistics during training.
One line of work focuses on channel-level weight normalizations to improve signal propagation and training stability.
For example, weight normalization (WN) \citep{salimans2016weight} and related schemes \citep{karras2024analyzing,loshchilov2024ngpt,franke2025learning,fu2025nemotron} reparameterize weights along rows, columns, or filters to control the norms. Relevant variants include but are not limited to centered WN \citep{huang2017centered} and weight standardization \citep{qiao2019micro,kolesnikov2020big}, which further control mean and variance along weight channels. 
Along a conceptually related line, \citep{xie2026mhcmanifoldconstrainedhyperconnections} proposes \textit{Manifold-Constrained Hyper-Connections} (mHC) ,
which projects the residual mixing matrices onto the Birkhoff polytope (i.e., the manifold of non-negative matrices with each row and each column summing to one) to stabilize signal propagation.

Another line of work controls matrix-level properties of weight matrices.
For example, some methods project weights into or onto Frobenius-norm
hyperballs/spheres \citep{wen2025hyperball,ren2026rethinking}, while others
constrain the spectral norm of weight matrices
\citep{yoshida2017spectral,miyato2018spectral,xie2026controlled,newhouse2025training,jiang2026enhancing}.
In contrast, our PC layer targets spectrum conditioning: it regulates the relative structure of the singular-value spectrum, rather than upper-bounding or shrinking the maximum singular value.

\paragraph{Weight spectrum conditioning during NN training.}

There are many methods that target spectrum conditioning during training. Some of them are not ``built-in layer'' solutions: some require explicit or repeated SVD/eigen-decomposition computations \citep{jia2017improving, huang2018orthogonal,jiang2019computation,li2019orthogonal}, while others incorporate constraints via optimizer co-designs \citep{cisse2017parseval, bansal2018can, bernstein2025manifolds}. In contrast, our PC layer is designed as a plug-and-play module that can be inserted into existing architectures without modifying the optimizer or the overall training pipeline.

A few recent works have also explored built-in conditioning mechanisms that operate through weight matrices. \citet{saratchandran2024weight} preconditions weights via diagonal row/column equilibration and \citet{saratchandran2026spectral, saratchandran2026conditioned} modifies the query, key, and value projection weights with additive correction terms or conditioned initialization to improve attention-Jacobian conditioning. These methods provide complementary evidence for the usefulness of weight conditioning, with empirical studies largely in CNNs, vision transformers and BERT-style masked-language-modeling settings. In contrast, PC targets decoder-only LLM pre-training and performs (low-degree) polynomial weight preconditioning. This polynomial formulation provides a flexible mechanism for encouraging the desired reshaping of the effective weight spectrum during training.

Some recent works
maintain the singular-value spectrum of weight matrices during training via
orthogonal-equivalence transformations
\citep{qiu2026reparameterized, qiu2026poet, shi2026pion}; in contrast, our PC
layer actively reshapes the spectrum through polynomial
preconditioning.

From a conceptual standpoint, our PC layer is related to the principle behind orthogonal weight initialization \citep{xiao2018dynamical, hu2020provable}. 
Orthogonal initialization is motivated by the dynamical isometry theory \citep{saxe2013exact,pennington2017resurrecting, pennington2018emergence}, which suggests that (approximately) orthogonal transformations help maintain a well-conditioned input–output Jacobian, thereby improving training dynamics. Subsequent works have extended the orthogonality principle to the entire training process \citep{jia2017improving,cisse2017parseval,bansal2018can,huang2018orthogonal,li2019orthogonal}.
In contrast, PC layer performs soft spectrum conditioning rather than collapsing all singular values to one, thereby preserving controlled spectral variation and model expressivity.

\paragraph{Convergence of neural network optimization.}
The output Jacobian (i.e., the Jacobian of the mapping from network parameters to the output vector) serves as a foundational mathematical object in the theoretical analysis of neural network training \citep{jacot2018neural,lee2019wide,du2019gradient,arora2019exact},  generalization \citep{jacot2018neural,chen2020generalized,simon2021neural} and architecture-dependent inductive biases \citep{bietti2019inductive,yang2021tensor}.
Many global convergence analyzes
\citep{jacot2018neural,lee2019wide,du2019gradient,arora2019exact} utilized the following simple fact:
gradient descent converges to a global minimum if the Jacobian remains non-degenerate along the iterates.
Nevertheless, to rigorously guarantee this condition is nontrivial.
A common route to ensure this condition is to assume an ultra-wide (overparameterized) regime, in which parameters move only mildly from initialization and the Jacobian stays close to its initial value throughout training \citep{lee2019wide,du2019gradient,arora2019exact,allen2019convergence, xiao2020disentangling}.

In contrast, our analysis does not rely on an ultra-width assumption; instead, our convergence analysis is based on the weight spectrum remaining bounded during training.
This permits nontrivial parameter movement, while still implies non-degenerate Jacobian.
From a strict theoretical perspective, the spectral bound is a simplifying assumption. However, it serves as a well-defined and architecturally actionable anchor that allows us to move from abstract convergence analysis to concrete model design.

\paragraph{Matrix preconditioning and matrix polynomials.}
Preconditioning is a fundamental technique in numerical linear algebra for accelerating iterative methods \citep{chen2005matrix}.
In general, it improves convergence by reducing the effective condition number or clustering the spectrum of the underlying operator, and has been widely used for solving large-scale linear systems \citep{saad2003iterative, benzi2002preconditioning}.

Among the preconditioning strategies, \emph{polynomial preconditioning} leverages matrix polynomials as an efficient yet effective mechanism for spectrum shaping.
By applying a suitably chosen low-degree polynomial to a matrix, one can enhance eigenvalue clustering, thereby improving the convergence of iterative solvers such as the conjugate gradient method
\citep{johnson1983polynomial}. Importantly, it can be implemented using only a few matrix multiplications, without explicit matrix inverses or decompositions.

In deep learning, a notable recent success of matrix-polynomial preconditioning is the Muon optimizer
(MomentUm Orthogonalized by Newton--Schulz) \citep{jordan2024muon}. Muon applies a few Newton--Schulz steps to the momentum matrix to approximate its polar (orthogonal) factor, i.e., pushing singular values toward $1$.
While both the Muon and our PC layer leverage matrix polynomials, they differ in three key ways:

\begin{enumerate}[(i)]
\item \textbf{Where the polynomial acts.}
Muon acts on \emph{momentum}, whereas PC acts on \emph{weight matrices}.
Muon is an optimization algorithm, while PC layer is a new architecture component.

\item \textbf{Target spectral transform.}
Muon’s polynomial is designed to (approximately) realize a \emph{matrix sign / polar-decomposition} map that drives singular values toward $1$ \citep{amsel2025polar}.
In contrast, our PC layer approximates a \emph{piecewise-linear} spectral map that stays close to $1$, while allowing a \emph{controlled deviation} on small singular values to preserve model expressivity. See Section~\ref{sec:find-poly} for the details of PC polynomials.

\item \textbf{Theoretical analysis.}
Muon, as a general-purpose optimizer, is analyzed under broad non-convex settings
\citep{bernstein2024old,bernstein2025deriving,shen2025convergence,wang2025muon,li2025note}.
Because the PC layer is explicitly designed for neural networks, our analysis leverages the specific structure of neural networks. This allows us to derive problem-specific guarantees, such as convergence to global minima under certain conditions.

\end{enumerate}

Notably, the PC layer can be \emph{orthogonally combined} with Muon, since the two methods act at different parts of the training process: Muon manipulates the momentum matrices, whereas our PC layer preconditions the weight matrices themselves. Empirically, adding PC on top of Muon yields further improvement (\S\ref{subsec:main_res}), with additional spectral and ablation analysis in Appendix~\ref{app:muon-results}.

\section{Conclusion and Limitations}
\label{sec:conclusion}
In this work, we introduce a \emph{Preconditioning (PC) layer}, a weight-based built-in module that maintains healthy weight conditioning throughout LLM pre-training.
We provide theoretical support in deep linear networks, showing that uniformly bounding each layer’s condition number yields geometric convergence of gradient descent, which motivates training-time control of the weight spectrum.
Guided by this principle, we develop an efficient \emph{polynomial preconditioning} algorithm that reshapes singular-value spectra using only lightweight matrix multiplications, avoids expensive decompositions, and can be merged into the original model after training with no inference overhead. We empirically validate that PC consistently improves optimization in Llama-1B pre-training.

Our empirical evaluation is constrained by available computational resources, so we mainly focus on Llama pre-training at 271M and 1B parameters. We do not yet report results on a broader range of architectures (e.g., non-Llama transformer variants) or substantially larger model scales.
While the observed gains are consistent across these two settings, more extensive studies are needed to characterize the generality of PC across different model families and parameter regimes.
In addition, PC introduces design choices (e.g., \texttt{pc\_level}) that may admit layer-specific optima in transformers. For instance, shallow and deep layers may benefit from different degrees of preconditioning. Systematically exploring such heterogeneous configurations and developing principled or adaptive rules for selecting PC settings are promising directions that we leave to future work.

\newpage

\bibliographystyle{abbrvnat}
\bibliography{reference.bib}

@inproceedings{allen2019convergence,
  title={A convergence theory for deep learning via over-parameterization},
  author={Allen-Zhu, Zeyuan and Li, Yuanzhi and Song, Zhao},
  booktitle={International Conference on Machine Learning},
  pages={242--252},
  year={2019},
  organization={PMLR}
}

@article{amsel2025polar,
  title={The polar express: Optimal matrix sign methods and their application to the {Muon} algorithm},
  author={Amsel, Noah and Persson, David and Musco, Christopher and Gower, Robert M},
  journal={arXiv preprint arXiv:2505.16932},
  year={2025}
}

@inproceedings{arora2018convergence,
  title={A convergence analysis of gradient descent for deep linear neural networks},
  author={Arora, Sanjeev and Cohen, Nadav and Golowich, Noah and Hu, Wei},
  booktitle={International Conference on Learning Representations},
  year={2019}
}

@article{arora2019exact,
  title={On exact computation with an infinitely wide neural net},
  author={Arora, Sanjeev and Du, Simon S and Hu, Wei and Li, Zhiyuan and Salakhutdinov, Russ R and Wang, Ruosong},
  journal={Advances in Neural Information Processing Systems},
  volume={32},
  year={2019}
}

@article{arora2019implicit,
  title={Implicit regularization in deep matrix factorization},
  author={Arora, Sanjeev and Cohen, Nadav and Hu, Wei and Luo, Yuping},
  journal={Advances in Neural Information Processing Systems},
  volume={32},
  year={2019}
}

@article{ba2016layer,
  title={Layer normalization},
  author={Ba, Jimmy Lei and Kiros, Jamie Ryan and Hinton, Geoffrey E},
  journal={arXiv preprint arXiv:1607.06450},
  year={2016}
}

@article{bansal2018can,
  title={Can we gain more from orthogonality regularizations in training deep networks?},
  author={Bansal, Nitin and Chen, Xiaohan and Wang, Zhangyang},
  journal={Advances in Neural Information Processing Systems},
  volume={31},
  year={2018}
}

@inproceedings{bartlett2018gradient,
  title={Gradient descent with identity initialization efficiently learns positive definite linear transformations by deep residual networks},
  author={Bartlett, Peter and Helmbold, Dave and Long, Philip},
  booktitle={International Conference on Machine Learning},
  pages={521--530},
  year={2018},
  organization={PMLR}
}

@article{benzi2002preconditioning,
  title={Preconditioning techniques for large linear systems: a survey},
  author={Benzi, Michele},
  journal={Journal of Computational Physics},
  volume={182},
  number={2},
  pages={418--477},
  year={2002},
  publisher={Elsevier}
}

@inproceedings{bernstein2024old,
  title={Old Optimizer, New Norm: An Anthology},
  author={Bernstein, Jeremy and Newhouse, Laker},
  booktitle={OPT 2024: Optimization for Machine Learning},
  year = {2024}
}

@misc{bernstein2025deriving,
  author = {Jeremy Bernstein},
  title = {Deriving {Muon}},
  url = {https://jeremybernste.in/writing/deriving-muon},
  year = {2025}
}

@article{bernstein2025manifolds,
  author = {Jeremy Bernstein},
  title = {Modular Manifolds},
  journal = {Thinking Machines Lab: Connectionism},
  year = {2025},
  note = {https://thinkingmachines.ai/blog/modular-manifolds/},
  doi = {10.64434/tml.20250926}
}

@article{bietti2019inductive,
  title={On the inductive bias of neural tangent kernels},
  author={Bietti, Alberto and Mairal, Julien},
  journal={Advances in Neural Information Processing Systems},
  volume={32},
  year={2019}
}

@article{brown2020language,
  title={Language models are few-shot learners},
  author={Brown, Tom and Mann, Benjamin and Ryder, Nick and Subbiah, Melanie and Kaplan, Jared D and Dhariwal, Prafulla and Neelakantan, Arvind and Shyam, Pranav and Sastry, Girish and Askell, Amanda and others},
  journal={Advances in Neural Information Processing Systems},
  volume={33},
  pages={1877--1901},
  year={2020}
}

@inproceedings{cisse2017parseval,
  title={Parseval networks: Improving robustness to adversarial examples},
  author={Cisse, Moustapha and Bojanowski, Piotr and Grave, Edouard and Dauphin, Yann and Usunier, Nicolas},
  booktitle={International Conference on Machine Learning},
  pages={854--863},
  year={2017},
  organization={PMLR}
}

@inproceedings{clark2019boolq,
  title={{BoolQ}: Exploring the surprising difficulty of natural yes/no questions},
  author={Clark, Christopher and Lee, Kenton and Chang, Ming-Wei and Kwiatkowski, Tom and Collins, Michael and Toutanova, Kristina},
  booktitle={Proceedings of the 2019 conference of the north American chapter of the association for computational linguistics: Human language technologies, volume 1 (long and short papers)},
  pages={2924--2936},
  year={2019}
}

@article{chen2020generalized,
  title={A generalized neural tangent kernel analysis for two-layer neural networks},
  author={Chen, Zixiang and Cao, Yuan and Gu, Quanquan and Zhang, Tong},
  journal={Advances in Neural Information Processing Systems},
  volume={33},
  pages={13363--13373},
  year={2020}
}

@book{chen2005matrix,
  title={Matrix preconditioning techniques and applications},
  author={Chen, Ke},
  number={19},
  year={2005},
  publisher={Cambridge University Press}
}

@article{chen2025muon,
  title={Muon Optimizes Under Spectral Norm Constraints},
  author={Chen, Lizhang and Li, Jonathan and Liu, Qiang},
  journal={arXiv preprint arXiv:2506.15054},
  year={2025}
}

@inproceedings{du2018gradient,
  title={Gradient descent provably optimizes over-parameterized neural networks},
  author={Du, Simon S and Zhai, Xiyu and Poczos, Barnabas and Singh, Aarti},
  booktitle={International Conference on Learning Representations},
  year={2019}
}

@inproceedings{du2019gradient,
  title={Gradient descent finds global minima of deep neural networks},
  author={Du, Simon and Lee, Jason and Li, Haochuan and Wang, Liwei and Zhai, Xiyu},
  booktitle={International Conference on Machine Learning},
  pages={1675--1685},
  year={2019},
  organization={PMLR}
}

@inproceedings{du2019width,
  title={Width provably matters in optimization for deep linear neural networks},
  author={Du, Simon and Hu, Wei},
  booktitle={International Conference on Machine Learning},
  pages={1655--1664},
  year={2019},
  organization={PMLR}
}

@misc{eval-harness,
  author       = {Gao, Leo and Tow, Jonathan and Abbasi, Baber and Biderman, Stella and Black, Sid and DiPofi, Anthony and Foster, Charles and Golding, Laurence and Hsu, Jeffrey and Le Noac'h, Alain and Li, Haonan and McDonell, Kyle and Muennighoff, Niklas and Ociepa, Chris and Phang, Jason and Reynolds, Laria and Schoelkopf, Hailey and Skowron, Aviya and Sutawika, Lintang and Tang, Eric and Thite, Anish and Wang, Ben and Wang, Kevin and Zou, Andy},
  title        = {The language model evaluation harness},
  month        = 07,
  year         = 2024,
  publisher    = {Zenodo},
  version      = {v0.4.3},
  doi          = {10.5281/zenodo.12608602},
  url          = {https://zenodo.org/records/12608602}
}

@misc{
fang2021precondition,
title={Precondition Layer and Its Use for {GAN}s},
author={Tiantian Fang and Alex Schwing and Ruoyu Sun},
year={2021},
url={https://openreview.net/forum?id=1yXhko8GZEE}
}

@article{franke2025learning,
  title={Learning in Compact Spaces with Approximately Normalized Transformer},
  author={Franke, J{\"o}rg KH and Spiegelhalter, Urs and Nezhurina, Marianna and Jitsev, Jenia and Hutter, Frank and Hefenbrock, Michael},
  journal={arXiv preprint arXiv:2505.22014},
  year={2025}
}

@article{fu2025nemotron,
  title={Nemotron-{F}lash: Towards Latency-Optimal Hybrid Small Language Models},
  author={Fu, Yonggan and Dong, Xin and Diao, Shizhe and Ye, Hanrong and Byeon, Wonmin and Karnati, Yashaswi and Liebenwein, Lucas and Zhang, Hannah and Binder, Nikolaus and Khadkevich, Maksim and others},
  journal={Advances in Neural Information Processing Systems},
  volume={38},
  year={2025}
}

@article{gadre2024language,
  title={Language models scale reliably with over-training and on downstream tasks},
  author={Gadre, Samir Yitzhak and Smyrnis, Georgios and Shankar, Vaishaal and Gururangan, Suchin and Wortsman, Mitchell and Shao, Rulin and Mercat, Jean and Fang, Alex and Li, Jeffrey and Keh, Sedrick and others},
  journal={arXiv preprint arXiv:2403.08540},
  year={2024}
}

@inproceedings{glorot2010understanding,
  title={Understanding the difficulty of training deep feedforward neural networks},
  author={Glorot, Xavier and Bengio, Yoshua},
  booktitle={International Conference on Artificial Intelligence and Statistics},
  pages={249--256},
  year={2010},
  organization={JMLR Workshop and Conference Proceedings}
}

@inproceedings{he2015delving,
  title={Delving deep into rectifiers: Surpassing human-level performance on imagenet classification},
  author={He, Kaiming and Zhang, Xiangyu and Ren, Shaoqing and Sun, Jian},
  booktitle={Proceedings of the IEEE International Conference on Computer Vision},
  pages={1026--1034},
  year={2015}
}

@inproceedings{henry2020query,
  title={Query-key normalization for transformers},
  author={Henry, Alex and Dachapally, Prudhvi Raj and Pawar, Shubham Shantaram and Chen, Yuxuan},
  booktitle={Findings of the Association for Computational Linguistics: EMNLP 2020},
  pages={4246--4253},
  year={2020}
}

@article{hoffmann2022training,
  title={Training compute-optimal large language models},
  author={Hoffmann, Jordan and Borgeaud, Sebastian and Mensch, Arthur and Buchatskaya, Elena and Cai, Trevor and Rutherford, Eliza and Casas, Diego de Las and Hendricks, Lisa Anne and Welbl, Johannes and Clark, Aidan and others},
  journal={arXiv preprint arXiv:2203.15556},
  year={2022}
}

@article{horner1819xxi,
  title={{XXI.} {A} new method of solving numerical equations of all orders, by continuous approximation},
  author={Horner, William George},
  journal={Philosophical Transactions of the Royal Society of London},
  number={109},
  pages={308--335},
  year={1819},
  publisher={The Royal Society London}
}

@inproceedings{hu2020provable,
  title={Provable benefit of orthogonal initialization in optimizing deep linear networks},
  author={Hu, Wei and Xiao, Lechao and Pennington, Jeffrey},
  booktitle={International Conference on Learning Representations},
  year={2020}
}

@inproceedings{huang2017centered,
  title={Centered weight normalization in accelerating training of deep neural networks},
  author={Huang, Lei and Liu, Xianglong and Liu, Yang and Lang, Bo and Tao, Dacheng},
  booktitle={Proceedings of the IEEE International Conference on Computer Vision},
  pages={2803--2811},
  year={2017}
}

@inproceedings{huang2018orthogonal,
  title={Orthogonal weight normalization: Solution to optimization over multiple dependent stiefel manifolds in deep neural networks},
  author={Huang, Lei and Liu, Xianglong and Lang, Bo and Yu, Adams and Wang, Yongliang and Li, Bo},
  booktitle={Proceedings of the AAAI Conference on Artificial Intelligence},
  volume={32},
  number={1},
  year={2018}
}

@article{huang2023normalization,
  title={Normalization techniques in training {DNN}s: Methodology, analysis and application},
  author={Huang, Lei and Qin, Jie and Zhou, Yi and Zhu, Fan and Liu, Li and Shao, Ling},
  journal={IEEE Transactions on Pattern Analysis and Machine Intelligence},
  volume={45},
  number={8},
  pages={10173--10196},
  year={2023},
  publisher={IEEE}
}

@inproceedings{ioffe2015batch,
  title={Batch normalization: Accelerating deep network training by reducing internal covariate shift},
  author={Ioffe, Sergey and Szegedy, Christian},
  booktitle={International Conference on Machine Learning},
  pages={448--456},
  year={2015},
  organization={PMLR}
}

@article{jacot2018neural,
  title={Neural tangent kernel: Convergence and generalization in neural networks},
  author={Jacot, Arthur and Gabriel, Franck and Hongler, Cl{\'e}ment},
  journal={Advances in Neural Information Processing Systems},
  volume={31},
  year={2018}
}

@inproceedings{ji2018gradient,
  title={Gradient descent aligns the layers of deep linear networks},
  author={Ji, Ziwei and Telgarsky, Matus},
  booktitle={International Conference on Learning Representations},
  year={2019}
}

@inproceedings{jia2017improving,
  title={Improving training of deep neural networks via singular value bounding},
  author={Jia, Kui and Tao, Dacheng and Gao, Shenghua and Xu, Xiangmin},
  booktitle={Proceedings of the IEEE Conference on Computer Vision and Pattern Recognition},
  pages={4344--4352},
  year={2017}
}

@inproceedings{jiang2019computation,
  title={On computation and generalization of generative adversarial networks under spectrum control},
  author={Jiang, Haoming and Chen, Zhehui and Chen, Minshuo and Liu, Feng and Wang, Dingding and Zhao, Tuo},
  booktitle={International Conference on Learning Representations},
  year={2019}
}

@article{jiang2026enhancing,
  title={Enhancing LLM Training via Spectral Clipping},
  author={Jiang, Xiaowen and Semenov, Andrei and Stich, Sebastian U},
  journal={arXiv preprint arXiv:2603.14315},
  year={2026}
}

@article{johnson1983polynomial,
  title={Polynomial preconditioners for conjugate gradient calculations},
  author={Johnson, Olin G and Micchelli, Charles A and Paul, George},
  journal={SIAM Journal on Numerical Analysis},
  volume={20},
  number={2},
  pages={362--376},
  year={1983},
  publisher={SIAM}
}

@misc{jordan2024muon,
  author       = {Keller Jordan and Yuchen Jin and Vlado Boza and Jiacheng You and
                  Franz Cesista and Laker Newhouse and Jeremy Bernstein},
  title        = {Muon: An optimizer for hidden layers in neural networks},
  year         = {2024},
  url          = {https://kellerjordan.github.io/posts/muon/}
}

@article{kaplan2020scaling,
  title={Scaling laws for neural language models},
  author={Kaplan, Jared and McCandlish, Sam and Henighan, Tom and Brown, Tom B and Chess, Benjamin and Child, Rewon and Gray, Scott and Radford, Alec and Wu, Jeffrey and Amodei, Dario},
  journal={arXiv preprint arXiv:2001.08361},
  year={2020}
}

@inproceedings{karras2024analyzing,
  title={Analyzing and improving the training dynamics of diffusion models},
  author={Karras, Tero and Aittala, Miika and Lehtinen, Jaakko and Hellsten, Janne and Aila, Timo and Laine, Samuli},
  booktitle={Proceedings of the IEEE/CVF Conference on Computer Vision and Pattern Recognition},
  pages={24174--24184},
  year={2024}
}

@article{kawaguchi2016deep,
  title={Deep learning without poor local minima},
  author={Kawaguchi, Kenji},
  journal={Advances in Neural Information Processing Systems},
  volume={29},
  year={2016}
}

@inproceedings{kingma2014adam,
  title={Adam: A method for stochastic optimization},
  author={Kingma, Diederik P},
  booktitle={International Conference on Learning Representations},
  year={2015}
}

@inproceedings{kolesnikov2020big,
  title={Big transfer (bit): General visual representation learning},
  author={Kolesnikov, Alexander and Beyer, Lucas and Zhai, Xiaohua and Puigcerver, Joan and Yung, Jessica and Gelly, Sylvain and Houlsby, Neil},
  booktitle={European conference on computer vision},
  pages={491--507},
  year={2020},
  organization={Springer}
}

@inproceedings{laurent2018deep,
  title={Deep linear networks with arbitrary loss: All local minima are global},
  author={Laurent, Thomas and Brecht, James},
  booktitle={International Conference on Machine Learning},
  pages={2902--2907},
  year={2018},
  organization={PMLR}
}

@article{lee2019wide,
  title={Wide neural networks of any depth evolve as linear models under gradient descent},
  author={Lee, Jaehoon and Xiao, Lechao and Schoenholz, Samuel and Bahri, Yasaman and Novak, Roman and Sohl-Dickstein, Jascha and Pennington, Jeffrey},
  journal={Advances in Neural Information Processing Systems},
  volume={32},
  year={2019}
}

@article{li2019orthogonal,
  title={Orthogonal deep neural networks},
  author={Li, Shuai and Jia, Kui and Wen, Yuxin and Liu, Tongliang and Tao, Dacheng},
  journal={IEEE transactions on pattern analysis and machine intelligence},
  volume={43},
  number={4},
  pages={1352--1368},
  year={2019},
  publisher={IEEE}
}

@inproceedings{
   liang2025torchtitan,
   title={{TorchTitan}: One-stop {PyTorch} native solution for production ready {LLM} pretraining},
   author={Wanchao Liang and Tianyu Liu and Less Wright and Will Constable and Andrew Gu and Chien-Chin Huang and Iris Zhang and Wei Feng and Howard Huang and Junjie Wang and Sanket Purandare and Gokul Nadathur and Stratos Idreos},
   booktitle={International Conference on Learning Representations},
   year={2025}
}

@article{lin2021faster,
  title={Faster Directional Convergence of Linear Neural Networks under Spherically Symmetric Data},
  author={Lin, Dachao and Sun, Ruoyu and Zhang, Zhihua},
  journal={Advances in Neural Information Processing Systems},
  volume={34},
  year={2021}
}

@article{liu2025muon,
  title={Muon is scalable for {LLM} training},
  author={Liu, Jingyuan and Su, Jianlin and Yao, Xingcheng and Jiang, Zhejun and Lai, Guokun and Du, Yulun and Qin, Yidao and Xu, Weixin and Lu, Enzhe and Yan, Junjie and others},
  journal={arXiv preprint arXiv:2502.16982},
  year={2025}
}

@inproceedings{loshchilov2024ngpt,
  title={n{GPT}: Normalized transformer with representation learning on the hypersphere},
  author={Loshchilov, Ilya and Hsieh, Cheng-Ping and Sun, Simeng and Ginsburg, Boris},
  booktitle={International Conference on Learning Representations},
  year={2025}
}

@inproceedings{loshchilov2017decoupled,
  title={Decoupled weight decay regularization},
  author={Loshchilov, Ilya and Hutter, Frank},
  booktitle={International Conference on Learning Representations},
  year={2019}

}

@inproceedings{miyato2018spectral,
  title={Spectral normalization for generative adversarial networks},
  author={Miyato, Takeru and Kataoka, Toshiki and Koyama, Masanori and Yoshida, Yuichi},
  booktitle={International Conference on Learning Representations},
  year={2018}
}

@article{newhouse2025training,
  title={Training Transformers with Enforced {Lipschitz} Constants},
  author={Newhouse, Laker and Hess, R Preston and Cesista, Franz and Zahorodnii, Andrii and Bernstein, Jeremy and Isola, Phillip},
  journal={arXiv preprint arXiv:2507.13338},
  year={2025}
}

@inproceedings{nguyen2017loss,
  title={The loss surface of deep and wide neural networks},
  author={Nguyen, Quynh and Hein, Matthias},
  booktitle={International Conference on Machine Learning},
  pages={2603--2612},
  year={2017},
  organization={PMLR}
}

@article{nguyen2020global,
  title={Global convergence of deep networks with one wide layer followed by pyramidal topology},
  author={Nguyen, Quynh N and Mondelli, Marco},
  journal={Advances in Neural Information Processing Systems},
  volume={33},
  pages={11961--11972},
  year={2020}
}

@article{penedo2024the,
  title={The {FineWeb} Datasets: Decanting the Web for the Finest Text Data at Scale},
  author={Penedo, Guilherme and Kydl{\'\i}{\v{c}}ek, Hynek and Lozhkov, Anton and Mitchell, Margaret and Raffel, Colin A and Von Werra, Leandro and Wolf, Thomas and others},
  journal={Advances in Neural Information Processing Systems},
  volume={37},
  pages={30811--30849},
  year={2024}
}

@article{pennington2017resurrecting,
  title={Resurrecting the sigmoid in deep learning through dynamical isometry: theory and practice},
  author={Pennington, Jeffrey and Schoenholz, Samuel and Ganguli, Surya},
  journal={Advances in Neural Information Processing Systems},
  volume={30},
  year={2017}
}

@article{qiao2019micro,
  title={Micro-batch training with batch-channel normalization and weight standardization},
  author={Qiao, Siyuan and Wang, Huiyu and Liu, Chenxi and Shen, Wei and Yuille, Alan},
  journal={arXiv preprint arXiv:1903.10520},
  year={2019}
}

@article{qiu2026reparameterized,
  title={Reparameterized {LLM} training via orthogonal equivalence transformation},
  author={Qiu, Zeju and Buchholz, Simon and Xiao, Tim and Dax, Maximilian and Sch{\"o}lkopf, Bernhard and Liu, Weiyang},
  journal={Advances in Neural Information Processing Systems},
  volume={38},
  pages={140775--140821},
  year={2026}
}

@article{qiu2026poet,
  title={{POET-X}: Memory-efficient {LLM} Training by Scaling Orthogonal Transformation},
  author={Qiu, Zeju and Liu, Lixin and Weller, Adrian and Shi, Han and Liu, Weiyang},
  journal={arXiv preprint arXiv:2603.05500},
  year={2026}
}

@inproceedings{pennington2018emergence,
  title={The emergence of spectral universality in deep networks},
  author={Pennington, Jeffrey and Schoenholz, Samuel and Ganguli, Surya},
  booktitle={International Conference on Artificial Intelligence and Statistics},
  pages={1924--1932},
  year={2018},
  organization={PMLR}
}

@article{ren2026rethinking,
  title={Rethinking Language Model Scaling under Transferable Hypersphere Optimization},
  author={Ren, Liliang and Liu, Yang and Shen, Yelong and Chen, Weizhu},
  journal={arXiv preprint arXiv:2603.28743},
  year={2026}
}

@book{saad2003iterative,
  title={Iterative methods for sparse linear systems},
  author={Saad, Yousef},
  year={2003},
  publisher={SIAM}
}

@article{salimans2016weight,
  title={Weight normalization: A simple reparameterization to accelerate training of deep neural networks},
  author={Salimans, Tim and Kingma, Durk P},
  journal={Advances in Neural Information Processing Systems},
  volume={29},
  year={2016}
}

@inproceedings{sardana2023beyond,
  title={Beyond {C}hinchilla-Optimal: Accounting for Inference in Language Model Scaling Laws},
  author={Sardana, Nikhil and Portes, Jacob and Doubov, Sasha and Frankle, Jonathan},
  booktitle={International Conference on Machine Learning},
  pages={43445--43460},
  year={2024},
  organization={PMLR}
}

@inproceedings{saratchandran2024weight,
  title={Weight conditioning for smooth optimization of neural networks},
  author={Saratchandran, Hemanth and Wang, Thomas X and Lucey, Simon},
  booktitle={European Conference on Computer Vision},
  pages={310--325},
  year={2024},
  organization={Springer}
}

@inproceedings{
saratchandran2026spectral,
title={Spectral Conditioning of Attention Improves Transformer Performance},
author={Hemanth Saratchandran and Simon Lucey},
booktitle={The Thirty-ninth Annual Conference on Neural Information Processing Systems},
year={2026},
url={https://openreview.net/forum?id=RBx1AfoL2J}
}

@inproceedings{
saratchandran2026conditioned,
title={Conditioned Initialization for Attention},
author={Hemanth Saratchandran and Simon Lucey},
booktitle={The Fourteenth International Conference on Learning Representations},
year={2026},
url={https://openreview.net/forum?id=cKNOCYPo2W}
}

@inproceedings{saxe2013exact,
  title={Exact solutions to the nonlinear dynamics of learning in deep linear neural networks},
  author={Saxe, Andrew M and McClelland, James L and Ganguli, Surya},
  booktitle={International Conference on Learning Representations},
  year={2014}
}

@article{shi2026pion,
  title={Pion: A Spectrum-Preserving Optimizer via Orthogonal Equivalence Transformation},
  author={Shi, Kexuan and Li, Hanxuan and Qiu, Zeju and Wen, Yandong and Buchholz, Simon and Liu, Weiyang},
  journal={arXiv preprint arXiv:2605.12492},
  year={2026}
}

@misc{
simon2021neural,
title={Neural tangent kernel eigenvalues accurately predict generalization},
author={James B Simon and Madeline Dickens and Michael R DeWeese},
year={2021},
url={https://openreview.net/forum?id=lycl1GD7fVP}
}

@article{touvron2023llama,
  title={Llama 2: Open foundation and fine-tuned chat models},
  author={Touvron, Hugo and Martin, Louis and Stone, Kevin and Albert, Peter and Almahairi, Amjad and Babaei, Yasmine and Bashlykov, Nikolay and Batra, Soumya and Bhargava, Prajjwal and Bhosale, Shruti and others},
  journal={arXiv preprint arXiv:2307.09288},
  year={2023}
}

@article{ulyanov2016instance,
  title={Instance normalization: The missing ingredient for fast stylization},
  author={Ulyanov, Dmitry and Vedaldi, Andrea and Lempitsky, Victor},
  journal={arXiv preprint arXiv:1607.08022},
  year={2016}
}

@article{wang2025muon,
  title={Muon Outperforms {Adam} in Tail-End Associative Memory Learning},
  author={Wang, Shuche and Zhang, Fengzhuo and Li, Jiaxiang and Du, Cunxiao and Du, Chao and Pang, Tianyu and Yang, Zhuoran and Hong, Mingyi and Tan, Vincent YF},
  journal={arXiv preprint arXiv:2509.26030},
  year={2025}
}

@article{wen2025fantastic,
  title={Fantastic pretraining optimizers and where to find them},
  author={Wen, Kaiyue and Hall, David and Ma, Tengyu and Liang, Percy},
  journal={arXiv preprint arXiv:2509.02046},
  year={2025}
}

@online{wen2025hyperball,
  title   = {Fantastic Pretraining Optimizers and Where to Find Them 2.1: Hyperball Optimization},
  author  = {Wen, Kaiyue and Dang, Xingyu and Lyu, Kaifeng and Ma, Tengyu and Liang, Percy},
  year    = {2025},
  month   = {12},
  day     = {15},
  url     = {https://tinyurl.com/muonh},
  urldate = {2025-12-15}
}

@inproceedings{wu2018group,
  title={Group normalization},
  author={Wu, Yuxin and He, Kaiming},
  booktitle={Proceedings of the European Conference on Computer Vision (ECCV)},
  pages={3--19},
  year={2018}
}

@inproceedings{xiao2018dynamical,
  title={Dynamical isometry and a mean field theory of {CNN}s: How to train 10,000-layer vanilla convolutional neural networks},
  author={Xiao, Lechao and Bahri, Yasaman and Sohl-Dickstein, Jascha and Schoenholz, Samuel and Pennington, Jeffrey},
  booktitle={International Conference on Machine Learning},
  pages={5393--5402},
  year={2018},
  organization={PMLR}
}

@inproceedings{xiao2020disentangling,
  title={Disentangling trainability and generalization in deep neural networks},
  author={Xiao, Lechao and Pennington, Jeffrey and Schoenholz, Samuel},
  booktitle={International Conference on Machine Learning},
  pages={10462--10472},
  year={2020},
  organization={PMLR}
}

@article{xie2026controlled,
  title={Controlled llm training on spectral sphere},
  author={Xie, Tian and Luo, Haoming and Tang, Haoyu and Hu, Yiwen and Liu, Jason Klein and Ren, Qingnan and Wang, Yang and Zhao, Wayne Xin and Yan, Rui and Su, Bing and others},
  journal={arXiv preprint arXiv:2601.08393},
  year={2026}
}

@misc{xie2026mhcmanifoldconstrainedhyperconnections,
      title={{mHC}: Manifold-Constrained Hyper-Connections}, 
      author={Zhenda Xie and Yixuan Wei and Huanqi Cao and Chenggang Zhao and Chengqi Deng and Jiashi Li and Damai Dai and Huazuo Gao and Jiang Chang and Kuai Yu and Liang Zhao and Shangyan Zhou and Zhean Xu and Zhengyan Zhang and Wangding Zeng and Shengding Hu and Yuqing Wang and Jingyang Yuan and Lean Wang and Wenfeng Liang},
      year={2026},
      eprint={2512.24880},
      archivePrefix={arXiv},
      primaryClass={cs.CL},
      url={https://arxiv.org/abs/2512.24880}, 
}

@inproceedings{yang2021tensor,
  title={Tensor programs {IIb}: Architectural universality of neural tangent kernel training dynamics},
  author={Yang, Greg and Littwin, Etai},
  booktitle={International Conference on Machine Learning},
  pages={11762--11772},
  year={2021},
  organization={PMLR}
}

@article{yang2021tuning,
  title={Tuning large neural networks via zero-shot hyperparameter transfer},
  author={Yang, Ge and Hu, Edward and Babuschkin, Igor and Sidor, Szymon and Liu, Xiaodong and Farhi, David and Ryder, Nick and Pachocki, Jakub and Chen, Weizhu and Gao, Jianfeng},
  journal={Advances in Neural Information Processing Systems},
  volume={34},
  pages={17084--17097},
  year={2021}
}

@article{yang2023spectral,
  title={A spectral condition for feature learning},
  author={Yang, Greg and Simon, James B and Bernstein, Jeremy},
  journal={arXiv preprint arXiv:2310.17813},
  year={2023}
}

@article{yoshida2017spectral,
  title={Spectral norm regularization for improving the generalizability of deep learning},
  author={Yoshida, Yuichi and Miyato, Takeru},
  journal={arXiv preprint arXiv:1705.10941},
  year={2017}
}

@article{zhang2019root,
  title={Root mean square layer normalization},
  author={Zhang, Biao and Sennrich, Rico},
  journal={Advances in neural information processing systems},
  volume={32},
  year={2019}
}

@inproceedings{zou2020global,
  title={On the global convergence of training deep linear {ResNets}},
  author={Zou, Difan and Long, Philip M and Gu, Quanquan},
  booktitle={International Conference on Learning Representations},
  year={2020}
}

@article{li2025note,
  title={A Note on the Convergence of {Muon}},
  author={Li, Jiaxiang and Hong, Mingyi},
  journal={arXiv preprint arXiv:2502.02900},
  year={2025}
}

@article{shen2025convergence,
  title={On the convergence analysis of {Muon}},
  author={Shen, Wei and Huang, Ruichuan and Huang, Minhui and Shen, Cong and Zhang, Jiawei},
  journal={arXiv preprint arXiv:2505.23737},
  year={2025}
}

@inproceedings{paperno2016lambada,
  title={The {LAMBADA} dataset: Word prediction requiring a broad discourse context},
  author={Paperno, Denis and Kruszewski, Germ{\'a}n and Lazaridou, Angeliki and Pham, Ngoc-Quan and Bernardi, Raffaella and Pezzelle, Sandro and Baroni, Marco and Boleda, Gemma and Fern{\'a}ndez, Raquel},
  booktitle={Proceedings of the 54th annual meeting of the association for computational linguistics (volume 1: Long papers)},
  pages={1525--1534},
  year={2016}
}

@inproceedings{zellers2019hellaswag,
  title={{HellaSwag}: Can a machine really finish your sentence?},
  author={Zellers, Rowan and Holtzman, Ari and Bisk, Yonatan and Farhadi, Ali and Choi, Yejin},
  booktitle={Proceedings of the 57th annual meeting of the association for computational linguistics},
  pages={4791--4800},
  year={2019}
}

@article{sakaguchi2021winogrande,
  title={{WinoGrande}: An adversarial winograd schema challenge at scale},
  author={Sakaguchi, Keisuke and Bras, Ronan Le and Bhagavatula, Chandra and Choi, Yejin},
  journal={Communications of the ACM},
  volume={64},
  number={9},
  pages={99--106},
  year={2021},
  publisher={ACM New York, NY, USA}
}

@inproceedings{bisk2020piqa,
  title={{PIQA}: Reasoning about physical commonsense in natural language},
  author={Bisk, Yonatan and Zellers, Rowan and Gao, Jianfeng and Choi, Yejin and others},
  booktitle={Proceedings of the AAAI conference on artificial intelligence},
  volume={34},
  number={05},
  pages={7432--7439},
  year={2020}
}

@article{clark2018think,
  title={Think you have solved question answering? {Try} {ARC}, the {AI2} reasoning challenge},
  author={Clark, Peter and Cowhey, Isaac and Etzioni, Oren and Khot, Tushar and Sabharwal, Ashish and Schoenick, Carissa and Tafjord, Oyvind},
  journal={arXiv preprint arXiv:1803.05457},
  year={2018}
}

@inproceedings{talmor2019commonsenseqa,
  title={Commonsenseqa: A question answering challenge targeting commonsense knowledge},
  author={Talmor, Alon and Herzig, Jonathan and Lourie, Nicholas and Berant, Jonathan},
  booktitle={Proceedings of the 2019 Conference of the North American Chapter of the Association for Computational Linguistics: Human Language Technologies, Volume 1 (Long and Short Papers)},
  pages={4149--4158},
  year={2019}
}

@inproceedings{mihaylov2018can,
  title={Can a suit of armor conduct electricity? A new dataset for open book question answering},
  author={Mihaylov, Todor and Clark, Peter and Khot, Tushar and Sabharwal, Ashish},
  booktitle={Proceedings of the 2018 conference on empirical methods in natural language processing},
  pages={2381--2391},
  year={2018}
}

\appendix
\newpage
\section{Notation Definition}
\begin{table}[h]
\centering
\begin{tabular}{l|l}
\hline
\multicolumn{2}{l}{\emph{Theory (deep linear network)}} \\
\hline
$L \in \mathbb{N}$ & number of layers \\
$n \in \mathbb{N}$ & number of training samples \\
$W_l \in \mathbb{R}^{d_l \times d_{l-1}}, l \in \{1, \dots, L\}$ & weight matrix of the $l^{\text{th}}$ layer \\
$\theta = (W_1, \dots, W_L)$ & collection of all parameters \\
$x \in \mathbb{R}^{d_x \times 1}$ & input of the neural network \\
$X \in \mathbb{R}^{d_x \times n}$ & collection of inputs \\
$y \in \mathbb{R}^{n d_y \times 1}$ & collection of labels \\
$t \in \{0,1,\dots,T\}$ & gradient descent iteration \\
$\mu_l \in \mathbb{R}$ & lower bound of weight matrix $W_l$ spectrum in region $\mathcal{R}$ \\
$\tau_l \in \mathbb{R}$ & upper bound of weight matrix $W_l$ spectrum in region $\mathcal{R}$ \\
$\mathcal{R}$ & region where every $W_l$ is well-conditioned ($\sigma_{\min}(W_l)\!\ge\!\mu_l,\ \sigma_{\max}(W_l)\!\le\!\tau_l$) \\
\hline
\multicolumn{2}{l}{\emph{Polynomial weight preconditioning}} \\
\hline
$W \in \mathbb{R}^{n \times m}$ & rectangular matrix \\
$\widetilde{W} \in \mathbb{R}^{n \times m}$ & scaled $W$ whose spectrum lies in $[0, 1]$ \\
$\sigma_1 \ge \cdots \ge \sigma_m \ge 0$ & singular values (ordered from largest to smallest) \\
$p,\ a = (a_0, \dots, a_k)$ & polynomial $p$ of degree $k$ and its coefficients \\
$g(\sigma) = p(\sigma^2)\,\sigma$ & induced scalar preconditioning map \\
$\texttt{pc\_level} = k$ & PC level; $g_k(\sigma)=p_k(\sigma^2)\sigma$ has overall degree $2k+1$ \\
$\mathrm{PL}_b(\sigma) = \min(\sigma/b,\,1)$ & piecewise-linear target with cutoff $b>0$ \\
$s(W) \approx \lVert W \rVert_2$ & streaming power-iteration estimate of the spectral norm \\
$\gamma \in \mathbb{R}$ & per-block learnable scalar (initialized to $1$) \\
\hline
\multicolumn{2}{l}{\emph{Spectrum metric}} \\
\hline
$\tilde{\kappa}(W) = \sigma_1 / \bar{\sigma}_{\text{bottom-10\%}}$ & modified condition number \\
\hline
\end{tabular}
\caption{Main symbols used throughout the paper.}
\end{table}

\newpage
\section{Background on Preconditioning}
\label{app:precondition-intro}

\emph{Preconditioning} originated in numerical linear algebra as a technique for accelerating iterative solvers for large sparse linear systems, where the cost of direct methods (e.g., Gaussian elimination) is prohibitive. The central observation is that the convergence speed of iterative methods is governed by the conditioning of the system matrix, and that a problem-specific transformation can substantially improve this conditioning. The canonical setting in which this idea is developed is solving a symmetric linear system $Qx = b$ via the conjugate gradient method; we briefly review this setting below before turning to polynomial preconditioners, which are the variant most relevant to our work.

Consider a linear system of equations
\begin{equation*}
    Qx = b,
\end{equation*}
where $Q \in \mathbb{R}^{n \times n}$ is real symmetric, and $b \in \mathbb{R}^{n \times 1}$. Conjugate gradient (CG) is one of the most popular methods to solve the system of equations. It has iteration complexity $O(\sqrt{\kappa(Q)} \log 1/\epsilon)$, where $\kappa(Q)$ is the condition number of $Q$. For ill-conditioned problems (i.e., large $\kappa(Q)$), the convergence can be slow. Thus, in practice, preconditioned CG is commonly used instead of the original CG.

Suppose there is a certain way to find a preconditioner $M$ that reduces the condition number, i.e., $\kappa(MQ) < \kappa(Q)$. Define $\tilde{Q} = MQ$ and $\tilde{b} = Mb$, then we can solve the alternative problem
\begin{equation*}
    \tilde{Q}x = \tilde{b},
\end{equation*}
for which CG (and other gradient methods) converge faster. One simple example is Jacobi preconditioning (closely related to whitening in machine learning) where $M$ is a diagonal matrix with $M_{ii} = 1/\sqrt{Q_{ii}}$.

\subsection{Polynomial Preconditioner}
\label{subapp:poly-preconditioner-johnson}

We review the polynomial preconditioners proposed by \citet{johnson1983polynomial}. However, we do not directly utilize the polynomials proposed by \citet{johnson1983polynomial} since our setting differs. But we do borrow two lessons, which we will explain at the end of this subsection.

Consider a linear system of equations
\begin{equation*}
    Qx = b,
\end{equation*}
where $Q \in \mathbb{R}^{n \times n}$ is real symmetric, and $b \in \mathbb{R}^{n \times 1}$. To find a polynomial preconditioning $g(Q) = p(Q)Q$ which is well-conditioned, we only need to find a polynomial $g$ such that $g(\lambda) = p(\lambda)\lambda$ maps $[\lambda_1, \lambda_m]$ to $[1 - \epsilon, 1]$. This can be formulated as an approximation theory problem: find a polynomial $g(x)$ that approximates a function $f$ where $f(0) = 0, f([\lambda_1, \lambda_m]) = 1$. Since the exact values of $\lambda_1, \lambda_m$ vary for each $Q$, we loosen the range from $[\lambda_1, \lambda_m]$ to $[\gamma_L, \gamma_U]$ such that $[\gamma_L, \gamma_U] \supseteq [\lambda_1, \lambda_m]$.

Define $P_k$ to be the set of all polynomials with degree no more than $k$, i.e.,
\vspace{-0.5em}
\begin{equation*}
    P_k = \{p(\lambda) \mid p(\lambda) = \sum_{j=0}^k a_j \lambda^j, a_j \in \mathbb{R} \text{ for any } j \in \{0,1,\dots,k\} \}.
\end{equation*}

\citet{johnson1983polynomial} consider two polynomial preconditioners: minimax and least-squares polynomials. Minimax polynomials are the solution to the following problem:
\begin{equation}
    \min_{p \in P_k} \max_{\lambda \in [\gamma_L, \gamma_U]} |1 - \lambda p(\lambda)|.
\end{equation}
Denote $g(\lambda) = \lambda p(\lambda)$, then there is a closed-form solution to the above problem: $g^*(\lambda) = 1 - \frac{T_{k+1}(\mu(\lambda))}{T_{k+1}(\mu(0))}$, where $\mu(\lambda) = \frac{\gamma_U + \gamma_L - 2\lambda}{\gamma_U - \gamma_L}$, and $T_k$ is the (first-kind) Chebyshev polynomial satisfying $T_k(\cos(z))=\cos(kz)$.

\citet{johnson1983polynomial} also consider least-squares polynomials, which are the solutions to the problem
\begin{equation}
    \min_g \int_{\gamma_L}^{\gamma_U} |1 - g(\lambda)| w(\lambda) d\lambda.
\end{equation}
There is a closed-form solution to the above problem since it is a quadratic problem in the coefficients of $g$. For the Jacobian weight function $w(\lambda) = (\gamma_U - \lambda)^\alpha (\lambda - \gamma_L)^\beta$ where $\alpha \ge \beta \ge -1/2$, the optimal solution $g^*$ stays positive in $[\gamma_L, \gamma_U]$.

As we will implement the preconditioning on rectangular weight matrices while $Q$ is a square matrix, we cannot directly borrow the polynomials used by \citet{johnson1983polynomial}. Nevertheless, we borrow two lessons for our design. First, the polynomial preconditioner can be designed by solving an optimization problem (either minimax or least squares). Second, they found that least-squares polynomials perform better than minimax polynomials for iterative algorithms. For this reason, we adopt the least-squares polynomials instead of the minimax polynomials.

\subsection{Relation between Preconditioning and the Spectrum}
\label{subapp:precond-spectrum}

For researchers less familiar with preconditioning, there may appear to be a discrepancy between theoretical results and practical implementation: while theory often relates the condition number to convergence speed, in practice, convergence frequently depends on the behavior of the entire spectrum. This distinction is well-recognized within numerical linear algebra and optimization. We have incorporated several considerations to bridge this conceptual gap, as elaborated below.

\medskip
\noindent \textbf{(1) Spectral metrics and practical impact.}
(a) We utilize a modified condition number, defined as the ratio of the largest singular value to the average of the bottom $10\%$ of singular values (see Sec.~\ref{sec:pc_improve_spec}), to assess the health of the singular value distribution. This metric accounts for the broader spectral distribution rather than being sensitive only to the extreme outliers. While one could further investigate the optimality of this metric following \citet{chen2005matrix} (Sec. 1.5), we found our current approach sufficient for achieving robust empirical results.
(b) Our preconditioner is designed to improve the entire spectrum, not merely the extreme condition number. For example, if an initial spectrum is $[1, 0.1, 0.1, 10^{-8}]$, our method aims to shift it toward a more favorable distribution, such as $[1, 0.2, 0.2, 2 \times 10^{-8}]$, thereby enhancing overall convergence characteristics.

\medskip
\noindent \textbf{(2) Historical context in numerical analysis.} The distinction between the condition number and the full spectrum has long been documented. \citet{johnson1983polynomial} observed that "the optimal polynomial preconditioner $M$ (which achieves the best condition number in a certain set) may map small eigenvalues of $A$ into large ones of $M^{-1}A$," which can degrade convergence. They argued that minimizing the condition number should not be the sole objective; instead, the goal should be the improvement of the whole spectrum. While the ideal metric remains a subject of discussion, we adopt their least-squares preconditioning approach as an effective alternative.

Etymologically, "preconditioning" often evokes the "condition number," yet historically, the term encompasses the broader intent of spectral improvement. As noted in the preface of \citet{chen2005matrix} (page xvi), the two most relevant terms in the field are "condition number and clustering." Section 1.5 of that text further emphasizes the importance of eigenvalues clustering. While a more descriptive name might be "spectral improver," we adhere to the standard terminology of "preconditioner."

\medskip
\noindent \textbf{(3) Theoretical vs. empirical objectives.} There remains an inherent disconnect in the optimization literature: researchers often prove results based on condition numbers (e.g., $O(\kappa \log 1/\epsilon)$ iteration complexity) while designing algorithms that practically manipulate the full spectrum. This pattern is evident in foundational works; for instance, \citet{johnson1983polynomial} aims to improve Conjugate Gradient bounds, and Nesterov’s acceleration is framed around improving the Gradient Descent bound. The community generally acknowledges that while the spectrum ultimately dictates performance, the condition number serves as a tractable, albeit weaker, proxy for theoretical analysis. Bridging this theoretical-practical divide remains an important area for future research.

\medskip
In summary, the standard paradigm, proving theorems on the condition number to justify methods that improve the spectrum, is a well-established practice in numerical analysis \citep{johnson1983polynomial}. We follow this convention and provide this discussion to clarify potential confusion regarding these related concepts.

\newpage
\section{Proof of Theorem \ref{thm:geom-cvrg}}
\label{app:theory}

In this section, we will provide a detailed proof of Theorem \ref{thm:geom-cvrg}. Note that, for the squared loss, the Gram matrix of the Jacobian coincides with the neural tangent kernel (NTK); therefore, conditioning of the Jacobian is equivalent to conditioning of the NTK. Throughout the proofs in this appendix, we will use the NTK-based formulation.
We decompose the argument into three steps.
\begin{itemize}
    \item \textbf{Step 1 (Appendix \ref{subapp:weight-cond--ntk-cond}) establishes weight conditioning → NTK conditioning}: uniform lower and upper bounds on \emph{the singular values of all weight matrices} imply, respectively, a lower and an upper bound on \emph{the eigenvalues of the empirical NTK} along the optimization trajectory.

    \item \textbf{Step 2 (Appendix \ref{subapp:ntk-cond--pl+smooth}) establishes NTK conditioning → optimization properties of the loss function}: the NTK lower-eigenvalue bound implies a Polyak–Łojasiewicz (PL) inequality for the squared loss, and the NTK upper-eigenvalue bound implies (local) smoothness of the loss.

    \item \textbf{Step 3 (Appendix \ref{subapp:std-cvrg-proof}) follows the standard global-convergence proof template for gradient descent}: combining the descent lemma (from local smoothness) with the PL inequality yields a per-iteration contraction of the loss, and hence the usual geometric convergence conclusion.

\end{itemize}

Throughout, we keep the constants explicit so that the final convergence rate can be read off directly from the weight spectral bounds.

For the reader’s convenience, we restate the setting and Theorem \ref{thm:geom-cvrg} here before giving the proof. We consider an $L$-layer linear network that maps $x\in\mathbb{R}^{d_x}$ to
\[
F(\theta;x) \;=\; W_L W_{L-1}\cdots W_1 x \in \mathbb{R}^{d_y},
\]
where $\theta=(W_1,\dots,W_L)$, and $W_l\in\mathbb{R}^{d_l\times d_{l-1}}$ for $l\in\{1,\dots,L\}$,
with $d_0=d_x$ and $d_L=d_y$.

\begin{rem}[Deep linear networks]
\label{rem:linear-nn}
The deep linear network is a widely used setting in theoretical analysis. While the input--output map is linear, the objective is generally nonconvex due to the matrix-factorization parameterization, making the setting non-trivial. Deep linear networks have been extensively studied for optimization landscape characterizations \citep{kawaguchi2016deep,laurent2018deep}, convergence and training dynamics \citep{arora2018convergence,du2019width,hu2020provable,lin2021faster}, implicit bias \citep{arora2019implicit,ji2018gradient}, etc. Even global convergence in this setting is non-trivial and has received extensive study \citep{arora2018convergence, bartlett2018gradient, ji2018gradient, zou2020global}; the closely related theory of dynamical isometry was also pioneered in linear models \citep{saxe2013exact}. Our analysis is in the same spirit as a line of works that aim at narrowing the gap between theory and practice through linear networks \citep{pennington2017resurrecting, pennington2018emergence, xiao2018dynamical, lee2019wide, xiao2020disentangling, hu2020provable,lin2021faster}: we use this setting to extract the actionable insight that a well-behaved weight spectrum facilitates training, which in turn motivates our algorithm design.
\end{rem}

We assume a pyramidal architecture as follows.
\newtheorem*{asmprestatedpyr}{Assumption \ref{asmp:pyramid} (Restated)}
\begin{asmprestatedpyr}[Pyramidal architecture]
There exists some \( r \in \{1, \dots, L\} \), such that
\[
d_0 \le d_1 \le \cdots \le d_r
\qquad\text{and}\qquad
d_r \ge d_{r+1} \ge \cdots \ge d_L.
\]
\end{asmprestatedpyr}

Suppose that we have $n$ training samples. We denote $X=[x_1,\dots,x_n]\in\mathbb{R}^{d_x\times n}$ as the input, $F(\theta;X)=(F(\theta;x_1);\dots;F(\theta;x_n))\in\mathbb{R}^{n d_y}$ as the stacked model predictions, and $y=(y_1;\dots;y_n)\in\mathbb{R}^{n d_y}$ as the stacked labels.
We further impose the following standard assumption on the input data matrix.
\newtheorem*{asmprestatedfcr}{Assumption \ref{asmp:full-col-rank} (Restated)}
\begin{asmprestatedfcr}[Non-degenerate input data]
The input matrix $X\in\mathbb{R}^{d_x\times n}$ has full column rank, i.e., $\sigma_{\min}(X) > 0$.
\end{asmprestatedfcr}
Assumption~\ref{asmp:full-col-rank} implicitly places us in the over-parameterized regime, which implies $n \le d_x$.

\begin{rem}[Pyramidal architecture]
\label{rem:pyramidal}
The pyramidal architecture assumption (Assumption~\ref{asmp:pyramid}) is not particularly restrictive, and it has been adopted in prior theory as a convenient way to exclude rank-bottleneck pathologies while still covering realistic width profiles.
First, the widely used ``equal-width'' hidden-layer assumption is a special case of a pyramidal topology, and it is standard in over-parameterization/NTK-style (neural tangent kernel) analyses for fully-connected networks \citep{hu2020provable, allen2019convergence, du2018gradient}.
Second, pyramidal architecture is a common structural assumption in prior optimization analyses to obtain benign loss landscape and global convergence guarantees \citep{nguyen2017loss,nguyen2020global,laurent2018deep}.

In our theoretical analysis, together with the full-column-rank input assumption (Assumption~\ref{asmp:full-col-rank}), the pyramidal architecture provides a convenient sufficient condition for obtaining a non-degenerate Jacobian,
which is exactly what enables the convergence results presented later.
\end{rem}

We consider the quadratic loss of the training data:
\[
\mathcal{L}(\theta) = \frac12 \sum_{i=1}^n \lVert F(\theta; x_i) - y_i \rVert^2 = \frac{1}{2} \lVert F(\theta;X)-y \rVert^2.
\]
We analyze gradient descent applied to the loss \(\mathcal{L}(\theta)\).
Given an initial parameter \(\theta(0)\), the iterates are updated as
\[
\theta(t+1) = \theta(t) - \eta \, \nabla \mathcal{L}(\theta(t)),
\]
where \(\eta > 0\) denotes the learning rate.

For given $\{(\tau_l,\mu_l)\}_{l=1}^L$ with $\tau_l \ge 1 \ge \mu_l>0$, we define the regions
\[
\begin{aligned}
\mathcal{W}_{\text{low}}(\mu_1, \dots, \mu_L) &:= \{ \theta = (W_1, \dots, W_L) \mid \sigma_{\min}(W_l) \geq \mu_l,\ \forall l \in \{1, \dots, L\} \}, \\
\mathcal{W}_{\text{up}}(\tau_1, \dots, \tau_L) &:= \{ \theta = (W_1, \dots, W_L) \mid \sigma_{\max}(W_l) \leq \tau_l,\ \forall l \in \{1, \dots, L\} \}, \\
\mathcal{R} &:= \mathcal{W}_{\text{up}}(\tau_1, \dots, \tau_L) \cap \mathcal{W}_{\text{low}}(\mu_1, \dots, \mu_L).
\end{aligned}
\]
Within $\mathcal{R}$, each layer weight matrix is well-conditioned: its singular values are bounded
away from both $0$ and $\infty$. The next theorem shows that, as long as the training trajectory stays inside
$\mathcal{R}$, gradient descent achieves a geometric decrease.

\newtheorem*{theoremrestated}{Theorem \ref{thm:geom-cvrg} (Restated)}

\begin{theoremrestated}
[Geometric convergence within $\mathcal{R}$]
Suppose the iterates satisfy $\theta(t)\in\mathcal{R}$ for all $t\in\{0,1,\dots,T\}$.
Define
\[
\beta
:= \left( \prod_{l=1}^L \tau_l \right)^2 \left( \sqrt{2 \mathcal{L}(\theta(0))} + \| X \|_F  \right) L \sigma_{\max}(X),
\qquad
\mu
:= \left( \prod_{l=1}^L \mu_l \right)^2 \sigma_{\min}(X)^2 .
\]
Then for any learning rate $\eta \in \bigl(0, 1/\beta\bigr]$, it holds that
\[
\mathcal{L}(\theta(t+1)) \le \bigl(1-\eta\mu\bigr)\, \mathcal{L}(\theta(t)),
\qquad \forall\, t\in\{0,1,\dots,T\}.
\]
In particular, choosing $\eta=1/\beta$ yields the contraction factor $1-\mu/\beta$:
\[
\mathcal{L}(\theta(t+1)) \le \Bigl(1-\frac{\mu}{\beta}\Bigr)\, \mathcal{L}(\theta(t)),
\qquad \forall\, t\in\{0,1,\dots,T\}.
\]
\end{theoremrestated}

\subsection{From Weight Matrices Conditioning to NTK Conditioning}
\label{subapp:weight-cond--ntk-cond}

In this subsection, we first review some basics of the neural tangent kernel (NTK), then derive the upper bound (Lemma \ref{lem:ntk-upper-bound}) and the lower bound (Lemma \ref{lem:ntk-lower-bound}) of NTK eigenvalues.

The neural tangent kernel (NTK) was introduced by \cite{jacot2018neural} (see also \cite{du2019gradient}). For a sample $x_i$, the parameter Jacobian is $\frac{\partial F(\theta; x_i)}{\partial \theta}\in\mathbb{R}^{P\times d_y}$, where $P$ is the number of parameters.
Stacking the Jacobians for $n$ samples gives the (global) Jacobian matrix
\[
G(\theta)
:= \frac{\partial F(\theta; X)}{\partial \theta} = \left[\frac{\partial F(\theta; x_1)}{\partial \theta},\dots, \frac{\partial F(\theta; x_n)}{\partial \theta} \right]
\in\mathbb{R}^{P\times n d_y}.
\]
The neural tangent kernel (NTK) is then defined as
\[
K(\theta):=G(\theta)^\top G(\theta)\in\mathbb{R}^{n d_y \times n d_y}.
\]

For the $L$-layer linear network with weights $\{W_l\}_{l=1}^L$ and widths $d_0,\dots,d_L$
(where $W_l\in\mathbb{R}^{d_l\times d_{l-1}}$ and $d_L=d_y$), we define the shorthand
\[
W_{i:j}:=W_i W_{i-1}\cdots W_j\qquad (i\ge j).
\]
Then the Jacobian $G(\theta)$ admits the block form
\[
G(\theta)= \begin{bmatrix}
    G_1(\theta) \\
    \vdots \\
    G_l(\theta) \\
    \vdots \\
    G_L(\theta) \\
\end{bmatrix} =
\begin{bmatrix}
X\otimes (W_{L:2})^\top \\
\vdots \\
W_{l-1,1} X\otimes (W_{L:l+1})^\top \\
\vdots \\
W_{L-1:1}X\otimes I_{d_L}
\end{bmatrix},
\]
where $\otimes$ denotes the Kronecker product and $I_{d_L}$ is the $d_L\times d_L$ identity matrix.
Here $G(\theta)$ is formed by stacking the layerwise Jacobian blocks
$G_l(\theta) \in \mathbb{R}^{d_{l-1}d_l \times n d_L}$ for
$l = 1,\dots,L$, so that in total $G(\theta)$ has
$P = \sum_{l=1}^L d_{l-1}d_l$ rows and $n d_L$ columns. Moreover, the NTK has the form
\[
K(\theta) =G(\theta)^\top G(\theta) = \sum_{l=1}^L G_l(\theta)^\top G_l(\theta).
\]

\begin{lem}[Upper bound of NTK eigenvalues]
\label{lem:ntk-upper-bound}
Consider a deep linear network of any shape.
For any $\theta \in \mathcal{W}_{\text{up}}(\tau_1, \dots, \tau_L)$, it holds that
\[
\lambda_{\max}(K(\theta)) \leq L \sigma_{\max}(X)^2 (\tau_1 \dots \tau_L)^2,
\]
or equivalently,
\[
\Vert G(\theta) \Vert_2 \leq \sqrt{L} \sigma_{\max}(X) \tau_1 \dots \tau_L.
\]
\end{lem}

\begin{proof}
By the submultiplicative property of matrix norms, we have
\[
\begin{aligned}
\Vert W_{l-1:1} X \Vert_2 &\leq \Vert W_{l-1} \Vert_2 \dots \Vert W_1 \Vert_2 \Vert X \Vert_2 \leq \sigma_{\max}(X) \tau_1 \dots \tau_{l-1}, \\
\Vert W_{L:l+1} \Vert_2 &\leq \Vert W_{L} \Vert_2 \dots \Vert W_{l+1} \Vert_2 \leq \tau_L \dots \tau_{l+1}.
\end{aligned}
\]
Since $\tau_l \ge 1$ for any $l \in \{1,2,\dots, L\}$, we have
\[
\begin{aligned}
\sigma_{\max} (G_l(\theta)) &= \sigma_{\max} (W_{l-1:1} X\otimes (W_{L:l+1})^\top) = \sigma_{\max} (W_{l-1:1} X) \cdot \sigma_{\max} ((W_{L:l+1})^\top) \\
&\leq (\sigma_{\max}(X) \tau_1 \dots \tau_{l-1}) \cdot (\tau_{l+1} \dots \tau_{L}) \le \sigma_{\max}(X) \tau_1 \dots \tau_{L}.
\end{aligned}
\]
Since this holds for every \( l \), we have
\[
\lambda_{\max}(K(\theta)) \le \sum_{l=1}^L \lambda_{\max} (G_l(\theta)^\top G_l(\theta)) \le \sum_{l=1}^L \sigma_{\max} (G_l(\theta))^2 \leq L \sigma_{\max}(X)^2 (\tau_1 \dots \tau_L)^2.
\]
Thus,
\[
\lVert G(\theta) \rVert_2 = \sqrt{\lambda_{\max} (G(\theta)^\top G(\theta))} = \sqrt{\lambda_{\max}(K(\theta)) } \le \sqrt{L} \sigma_{\max}(X) \tau_1 \dots \tau_L.
\]
\end{proof}

Lemma~\ref{lem:ntk-upper-bound} shows that an upper bound on the spectral norms of the weight matrices (i.e., $\theta\in\mathcal{W}_{\text{up}}(\tau_1,\dots,\tau_L)$) directly implies an upper bound on the largest NTK eigenvalue $\lambda_{\max}(K(\theta))$. Notably, this argument does not depend on any pyramid-like constraint on the network widths and therefore applies to deep linear networks of arbitrary shapes. In contrast, to derive a meaningful lower bound on $\lambda_{\min}(K(\theta))$, the pyramidal-architecture assumption (Assumption~\ref{asmp:pyramid}) becomes essential.

\begin{lem}
\label{lem:sig-min-multiply}
Suppose $A_m, A_{m-1}, \dots, A_1$ are matrices of size $n_m \times n_{m-1}, n_{m-1} \times n_{m-2}, \dots, n_1 \times n_0$, where $n_m \geq n_{m-1} \geq \dots \geq n_0$. Suppose $\sigma_{\min}(A_i) \geq \mu_i, i = 1, \dots, m$. Then the product $M = A_m A_{m-1} \dots A_1$ satisfies $\sigma_{\min}(M) \geq \mu_1 \mu_2 \dots \mu_m$.
\end{lem}

\begin{proof}
We first prove the following result: for a $k_1 \times k_2$ matrix $A$ and a $k_2 \times k_3$ matrix $B$, where $k_1 \geq k_2 \geq k_3$,
\[
\sigma_{\min}(AB) \geq \sigma_{\min}(A) \sigma_{\min}(B).
\]
This is proved by the following chain of inequalities:
\begin{align*}
    \lambda_{\min}(B^T A^T A B) = \min_{\|u\| = 1} u^T B^T A^T A B u \geq & \lambda_{\min}(A^T A) \|B u\|^2 = \lambda_{\min}(A^T A) u^T B^T B u \\
    \geq & \lambda_{\min}(A^T A) \lambda_{\min}(B^T B).
\end{align*}
Since matrices $A$, $B$, and $AB$ are all tall matrices, we have
\begin{align*}
    \sigma_{\min}(AB) = \sqrt{\lambda_{\min}((AB)^\top AB)} \ge \sqrt{\lambda_{\min}(A^\top A)} \sqrt{\lambda_{\min}(B^\top B)} = \sigma_{\min}(A) \sigma_{\min}(B).
\end{align*}
Applying the result multiple times, we immediately obtain the desired result.
\end{proof}

\begin{lem}
[Lower bound of NTK eigenvalues]
\label{lem:ntk-lower-bound}
Consider an $L$-layer pyramidal linear network satisfying Assumptions~\ref{asmp:pyramid} and~\ref{asmp:full-col-rank}.
For any $\theta \in \mathcal{W}_{\rm low} (\mu_1, \dots, \mu_L)$, it holds that
\[
\lambda_{\min}(K(\theta)) \geq \sigma_{\min}(X)^2 (\mu_1 \dots \mu_L)^2.
\]
\end{lem}

\begin{proof}
Recall the pyramidal-architecture assumption (Assumption \ref{asmp:pyramid}), we have
\[
d_0 \le d_1 \le \dots \le d_r \quad \text{and} \quad d_r \ge d_{r+1} \ge \dots \ge d_L.
\]
We focus on the $r$-th Jacobian block $G_r(\theta) = (W_{r-1:1}X) \otimes (W_{L:r+1})^\top$, which is a $(d_{l-1}d_l \times d_{y}n)$ matrix.
Since the number of training samples satisfies $n \le d_x = d_0$ (a consequence of Assumption~\ref{asmp:full-col-rank}),
\[
W_{r-1:1}X = W_r W_{r-1} \dots W_1 X
\]
is a multiplication of many tall matrices. By Lemma \ref{lem:sig-min-multiply}, we have
\[
\sigma_{\min}(W_{r-1:1} X) \geq \sigma_{\min}(X) \mu_1 \dots \mu_{r-1}.
\]
Similarly, $(W_{L:r+1})^\top = W_{r+1}^\top \dots W_L^\top$ is also a multiplication of many tall matrices, so
\[
\sigma_{\min}((W_{L:r+1})^\top) \geq \mu_{r+1} \dots \mu_L.
\]
and $\sigma_{\min}(W^T_{L:r+1}) \geq \mu_{r+1} \dots \mu_L$.
Since that $G_r(\theta)$ is also a tall matrix, we have
\begin{align*}
    \sigma_{\min}(G_r(\theta)) &= \sigma_{nd_y}(G_r(\theta)) = \sigma_{n} (W_{r-1:1} X) \cdot \sigma_{d_y} ((W_{L:r+1})^\top) = \sigma_{\min} (W_{r-1:1} X) \cdot \sigma_{\min} ((W_{L:r+1})^\top) \\
    &= (\sigma_{\min}(X) \mu_1 \dots \mu_{r-1}) (\mu_{r+1} \dots \mu_L) \\
    &\ge \sigma_{\min}(X) \mu_1 \dots \mu_{r-1} \mu_r \mu_{r+1} \dots \mu_L,
\end{align*}
where the inequality follows from the assumption that $\mu_r \le 1$.

Therefore,
\[
K(\theta) = \sum_{l=1}^L G_l^\top(\theta) G_l(\theta) \succeq G^\top_{r}(\theta) G_{r}(\theta) \succeq (\sigma_{\min}(X) \mu_1 \dots \mu_L)^2 I_n.
\]
\end{proof}

\subsection{From NTK Conditioning to PL Inequality and Smoothness of the Loss}
\label{subapp:ntk-cond--pl+smooth}

In this subsection, we convert the NTK eigenvalue bounds from Appendix \ref{subapp:weight-cond--ntk-cond} into two key analytic properties of the training loss function. More precisely, Lemma \ref{lem:pl-loss} uses the lower bound on the minimum NTK eigenvalue to derive a Polyak–Łojasiewicz (PL) inequality for the squared loss, while Lemma \ref{lem:smoothness} uses the upper bound on the maximum NTK eigenvalue to establish the local smoothness of the loss function. We note that both results hold provided that the parameters remain within the well-conditioned region.

\begin{lem}
[PL inequality]
\label{lem:pl-loss}
Consider an $L$-layer pyramidal linear network satisfying Assumptions~\ref{asmp:pyramid} and~\ref{asmp:full-col-rank}. For any $\theta \in \mathcal{W}_{\rm low} (\mu_1, \dots, \mu_L)$, it holds that
\[
\lVert \nabla \mathcal{L}(\theta;X) \rVert^2 \ge 2 \mu \mathcal{L}(\theta;X),
\]
where $\mu := \sigma_{\min}(X)^2 (\mu_1 \mu_2 \dots \mu_L)^2$.
\end{lem}

\begin{proof}
For the quadratic loss function, we have
\[
\nabla \mathcal{L}(\theta;X)=G(\theta)e(\theta),
\qquad \text{where }
e(\theta)= F(\theta;X)-y.
\]
By Lemma \ref{lem:ntk-lower-bound}, we have
\[
\lVert \nabla \mathcal{L}(\theta;X) \rVert^2 = e(\theta)^\top G(\theta)^\top G(\theta)e(\theta) \ge \mu \lVert e(\theta) \rVert^2 = 2 \mu \mathcal{L}(\theta;X).
\]

\end{proof}

\begin{rem}[Global minimum is zero]
To interpret Lemma~\ref{lem:pl-loss} as a standard PL inequality, We note that the model is \emph{realizable} in our setting, i.e., the network can interpolate the data, and hence the global minimum of $\mathcal{L}(\cdot;X)$ is zero.

Under the pyramidal-architecture assumption (Assumption~\ref{asmp:pyramid}), the widths satisfy the no-bottleneck condition
\[
\min_{l=1,\dots,L-1} d_l \;\ge\; \min(d_0,d_L).
\]
It is shown by \citet{laurent2018deep} that, under this condition, the deep linear parametrization \(A=W_L\cdots W_1\) can represent \emph{any} linear map \(A\in\mathbb{R}^{d_y\times d_x}\). Since the input matrix \(X\) has full column rank (Assumption~\ref{asmp:full-col-rank}), there exists a linear map \(A^\star\in\mathbb{R}^{d_y\times d_x}\) that interpolates the data, i.e., \(A^\star X = Y\). Combining these two facts, there exist weights \((W_1^\star,\dots,W_L^\star)\) such that
\[
W_L^\star\cdots W_1^\star X = Y,
\]
and hence \(\mathcal{L}(\theta^\star;X)=0\). Therefore, the global minimum value of \(\mathcal{L}(\cdot;X)\) is zero, and Lemma~\ref{lem:pl-loss} indeed provides the usual PL condition
\[
\lVert \nabla \mathcal{L}(\theta;X) \rVert^2 \ge 2 \mu \left( \mathcal{L}(\theta;X) - \mathcal{L}(\theta^*;X) \right).
\]
\end{rem}

\begin{lem}
[Lipschitz continuity of $F$]
\label{lem:F-lip}
Consider an $L$-layer linear network of any shape.
For any $\theta, \hat{\theta} \in \mathcal{W}_{\rm up}(\tau_1,\dots, \tau_L)$, it holds that
\[
\Vert F(\theta; X) - F(\hat{\theta}; X) \Vert \le \sqrt{L} \Vert X \Vert_{\rm F} \tau_L \dots \tau_1 \cdot \Vert \theta - \hat{\theta} \Vert .
\]
\end{lem}

\begin{proof}
For any $\theta, \hat{\theta} \in \mathcal{W}_{\rm up}(\tau_1,\dots, \tau_L)$, we have
\begin{equation}
\label{01022045}
\begin{aligned}
\|F(\theta;X)-F(\hat{\theta};X)\|
&=\|W_L\cdots W_1X-\hat{W}_L\cdots \hat{W}_1X\|_{\rm F} \\
&\le \|W_L\cdots W_1-\hat{W}_L\cdots \hat{W}_1\|_2\,\|X\|_{\rm F} \\
&= \Big\|\sum_{l=1}^L \hat{W}_L\cdots \hat{W}_{l+1}(W_l-\hat{W}_l)W_{l-1}\cdots W_1\Big\|_2\,\|X\|_F.
\end{aligned}
\end{equation}
By the assumption $\tau_l \ge 1$ and Cauchy's inequality, we obtain
\begin{align*}
\text{RHS of (\ref{01022045})} &\le \|X\|_F \sum_{l=1}^L \lVert \hat{W}_L \rVert_2 \dots \lVert\hat{W}_{l+1} \rVert_2 \lVert W_{l} - \hat{W}_{l} \rVert_2 \lVert \hat{W}_{l-1} \rVert_2 \dots \lVert\hat{W}_{1} \rVert_2  \\
&\le \|X\|_F \sum_{i=1}^L \tau_L\ \dots \tau_{l+1} \|W_i-\hat{W}_i\|_2 \tau_{l-1} \cdots \tau_1 \\
&\le \tau_L\cdots \tau_1\,\|X\|_F
\sum_{i=1}^L \|W_i-\hat{W}_i\|_2 \\
&\le \sqrt{L} \tau_L\cdots \tau_1\,\|X\|_F
\sqrt{\sum_{i=1}^L \|W_i-\hat{W}_i\|_2^2}.
\end{align*}
We can relax
\[
\sqrt{\sum_{i=1}^L \|W_i-\hat{W}_i\|_2^2}
\le
\sqrt{\sum_{i=1}^L \|W_i-\hat{W}_i\|_F^2}
= \|\theta-\hat{\theta}\|,
\]
thus obtaining the desired result.
\end{proof}

\begin{lem}
[Lipschitz continuity of $G$]
\label{lem:G-lip}
Consider a deep linear network of any shape. Assume \( \sigma_{\max}(X) \leq \tau_0 \). For any $\theta, \hat{\theta} \in \mathcal{W}_{\rm up}(\tau_1,\dots, \tau_L)$, it holds that
\[
\Vert G(\theta) - G(\hat{\theta}) \Vert_2 \leq L \sigma_{\max}(X) \tau_L \dots \tau_1  \Vert \theta - \hat{\theta} \Vert.
\]
\end{lem}

\begin{proof}
Recall that
\[
G(\theta)^\top=[G_1(\theta)^\top;G_2(\theta)^\top;\dots;G_L(\theta)^\top],
\qquad
G_l(\theta)=(W_{l-1:1}X)\otimes (W_{L:l+1})^{\top}.
\]
So we have
\begin{equation}
\label{01021954}
\begin{aligned}
\|G(\theta)-G(\hat{\theta})\|_2
&=\sqrt{\lambda_{\max}\!\big((G(\theta)-G(\hat{\theta}))^{\top}(G(\theta)-G(\hat{\theta}))\big)} \\
&=\sqrt{\lambda_{\max}\!\Big(\sum_{l=1}^L (G_l(\theta)-G_l(\hat{\theta}))^{\top}(G_l(\theta)-G_l(\hat{\theta}))\Big)} \\
&\le \sqrt{L}\,\max_{l}\|G_l(\theta)-G_l(\hat{\theta})\|_2.
\end{aligned}
\end{equation}

We denote $G_l(\hat{\theta})=(\hat{W}_{l-1:1}X)\otimes (\hat{W}_{L:l+1})^{\top}$. Then for any $l \in \{1,2,\dots, L\}$, we have
\begin{equation}
\label{01021955}
\begin{aligned}
\|G_l(\theta)-G_l(\hat{\theta})\|_2
&=\|(W_{l-1:1}X)\otimes (W_{L:l+1})^{\top}-(\hat{W}_{l-1:1}X)\otimes (\hat{W}_{L:l+1})^{\top}\|_2 \\
&=\|(W_{l-1:1}X-\hat{W}_{l-1:1}X)\otimes (W_{L:l+1})^{\top}
+(\hat{W}_{l-1:1}X)\otimes (W_{L:l+1}-\hat{W}_{L:l+1})^{\top}\|_2 \\
&\le \|(W_{l-1:1}X-\hat{W}_{l-1:1}X)\otimes (W_{L:l+1})^{\top}\|_2
+ \|(\hat{W}_{l-1:1}X)\otimes (W_{L:l+1}-\hat{W}_{L:l+1})^{\top}\|_2 \\
&= \|W_{l-1:1}X-\hat{W}_{l-1:1}X\|_2\,\|W_{L:l+1}\|_2
+ \|\hat{W}_{l-1:1}X\|_2\,\|W_{L:l+1}-\hat{W}_{L:l+1}\|_2.
\end{aligned}
\end{equation}
Following the same proof strategy as in Lemma \ref{lem:F-lip}, we obtain
\begin{align*}
\|W_{l-1:1}X-\hat{W}_{l-1:1}X\|_2 &\le \tau_{l-1}\cdots \tau_1 \sigma_{\max}(X) \sum_{i=1}^{l-1}\lVert W_i-\hat{W}_i\rVert_2, \\
\|W_{L:l+1}-\hat{W}_{L:l+1}\|_2 &\le \tau_{l+1}\cdots \tau_L  \sum_{i=l+1}^{L}\lVert W_i-\hat{W}_i\rVert_2.
\end{align*}
By the submultiplicative property of matrix norms, we have
\begin{align*}
\|\hat{W}_{l-1:1}X\|_2 \le \tau_{l-1}\cdots \tau_1 \sigma_{\max}(X), \qquad
\|W_{L:l+1}\|_2 \le \tau_{l+1}\cdots \tau_L.
\end{align*}
Since we assume that $\tau_l \geq 1$, we get
\begin{align*}
    &\ \text{RHS of (\ref{01021955})} \\
    &\le \left( \tau_{l-1}\cdots \tau_1 \sigma_{\max}(X) \sum_{i=1}^{l-1}\lVert W_i-\hat{W}_i\rVert_2 \right) (\tau_{l+1}\cdots \tau_L) + \left( \tau_{l-1}\cdots \tau_1 \sigma_{\max}(X) \right) \left( \tau_{l+1}\cdots \tau_L  \sum_{i=l+1}^{L}\lVert W_i-\hat{W}_i\rVert_2 \right) \\
    &\le \sigma_{\max}(X) \tau_1 \cdots \tau_L \sum_{l=1}^L \lVert W_i-\hat{W}_i\rVert_2 \\
    &\le \sqrt{L} \sigma_{\max}(X) \tau_1 \cdots \tau_L \lVert \theta - \hat{\theta} \rVert.
\end{align*}
Here, the last inequality follows the proof strategy of Lemma \ref{lem:F-lip}. Combining it with (\ref{01021954}) and (\ref{01021955}), we get
\begin{equation*}
    \lVert G(\theta)-G(\hat{\theta})\rVert_2 \le \sqrt{L} \max_l \lVert G_l(\theta)-G_l(\hat{\theta})\rVert_2 \le L \sigma_{\max}(X) \tau_1 \cdots \tau_L \lVert \theta - \hat{\theta} \rVert.
\end{equation*}
\end{proof}

\begin{lem}[Local smoothness of loss function]
\label{lem:smoothness}
Consider an $L$-layer linear network \( F(\theta; x) = W_L \dots W_1 x \). Then
\[
\Vert \nabla \mathcal{L}(\theta; X) - \nabla \mathcal{L}(\hat{\theta}; X) \Vert \leq \beta(\theta) \Vert \theta - \hat{\theta} \Vert, \quad \forall \theta, \hat{\theta} \in \mathcal{W}_{\text{up}}(\tau_1, \dots, \tau_L),
\]
where
\[
\beta(\theta) = L \sigma_{\max}(X) \left(\sqrt{2 \mathcal{L}(\theta)} + \lVert X \rVert_{\rm F} \right)  (\tau_1 \cdots \tau_L)^2.
\]
\end{lem}

\begin{proof}
For the quadratic loss function, we have
\[
\nabla \mathcal{L}(\theta;X)=G(\theta)e(\theta),
\qquad \text{where }
e(\theta)= F(\theta;X)-y.
\]
For any $\theta, \hat{\theta}$, the norm of their gradient difference satisfies
\begin{align*}
\|\nabla \mathcal{L}(\theta;X)-\nabla \mathcal{L}(\hat{\theta};X)\|
&=\|G(\theta)e(\theta)-G(\hat{\theta})e(\hat{\theta})\| \\
&\le \|G(\theta)-G(\hat{\theta})\|_2\,\|e(\hat{\theta})\|
+\|G(\theta)\|_2\,\|e(\theta)-e(\hat{\theta})\| \\
&= \|G(\theta)-G(\hat{\theta})\|_2\,\|e(\theta)\|
+\|G(\theta)\|_2\,\|F(\theta;X)-F(\hat{\theta};X)\|.
\end{align*}

Applying Lemma \ref{lem:ntk-upper-bound}, Lemma \ref{lem:F-lip}, and Lemma \ref{lem:G-lip}, since $\tau_l \ge 1$ for any $l \in \{1,2,\dots, L\}$, we get
\begin{align*}
&\quad \|\nabla \mathcal{L}(\theta;X)-\nabla \mathcal{L}(\hat{\theta};X)\| \\
&\le L \sigma_{\max}(X) \tau_1 \cdots \tau_L \lVert \theta - \hat{\theta} \rVert \cdot \lVert e(\theta)\rVert
+ (\sqrt{L} \sigma_{\max}(X) \tau_1 \cdots \tau_L) \cdot (\sqrt{L} \lVert X \rVert_{\rm F} \tau_1 \cdots \tau_L \cdot \lVert \theta - \hat{\theta} \rVert) \\
&\le L \sigma_{\max}(X) \left( \lVert e(\theta)  \rVert \tau_1 \cdots \tau_L + \lVert  X \rVert_{\rm F}  (\tau_1 \cdots \tau_L)^2 \right)  \|\theta-\hat{\theta}\| \\
&\le L \sigma_{\max}(X) \left(\sqrt{2 \mathcal{L}(\theta)} + \lVert X \rVert_{\rm F} \right)  (\tau_1 \cdots \tau_L)^2 \|\theta-\hat{\theta}\|.
\end{align*}
\end{proof}

\subsection{Geometric Convergence within the Well-Conditioned Region $\mathcal{R}$}
\label{subapp:std-cvrg-proof}

In this subsection, we complete the proof of Theorem \ref{thm:geom-cvrg} by applying a standard gradient-descent analysis under the PL inequality and smoothness established in Appendix \ref{subapp:ntk-cond--pl+smooth}. We choose a stepsize satisfying the usual stability condition determined by the smoothness constant (i.e., $\eta \le 1/\beta$), and then combine (i) the descent lemma with (ii) the PL inequality to obtain a one-step contraction of the loss.

\begin{proof}
[Proof of Theorem \ref{thm:geom-cvrg}]
We prove this theorem by induction. Suppose that the result holds for \( 0, 1, \dots, k-1 \). Then we have \( \mathcal{L}(\theta(k)) \leq \mathcal{L} (\theta(0)) \). Recall the definition of $\beta(\theta)$ and $\beta$,
\begin{align*}
    \beta(\theta) &= L \sigma_{\max}(X) \left(\sqrt{2 \mathcal{L}(\theta)} + \lVert X \rVert_{\rm F} \right)  (\tau_1 \cdots \tau_L)^2, \\
    \beta &= L \sigma_{\max}(X) \left(\sqrt{2 \mathcal{L}(\theta(0))} + \lVert X \rVert_{\rm F} \right)  (\tau_1 \cdots \tau_L)^2.
\end{align*}
This implies
\begin{equation}
\label{01022307}
    \beta(\theta(k)) \leq \beta(\theta(0)) = \beta.
\end{equation}

For the notational simplicity, let us denote \( g(\theta) = \nabla \mathcal{L}(\theta;X) \), and \( g_k = \nabla \mathcal{L}(\theta(k); X) \). Denote \( \delta_k = -\eta g_k = \theta(k+1) - \theta(k) \). We follow the proof of the descent lemma. We obtain
\begin{equation}
\label{01022310}
\begin{aligned}
    \mathcal{L}(\theta(k+1)) &= \mathcal{L}(\theta(k)) + g_k^\top \delta_k + \int_0^1 \left[ g(\theta(k) + t \delta_k) - g(\theta(k)) \right]^\top \delta_k dt \\
    &\le \mathcal{L}(\theta(k)) - \eta \lVert g_k \rVert^2 + \int_0^1 \lVert g(\theta(k) + t \delta_k) - g(\theta(k)) \rVert \cdot \lVert \delta_k \lVert dt.
\end{aligned}
\end{equation}

We denote
\begin{align*}
    \theta(k) := (W_1^{(k)}, \dots, W_L^{(k)}), \quad \theta(k+1) := (W_1^{(k+1)}, \dots, W_L^{(k+1)}),
\end{align*}
then
\begin{align*}
    \theta(k) + t\delta_k &= t (\theta(k) + \delta_k) + (1-t) \theta(k) \\
    &= t \theta(k+1) + (1-t) \theta(k) \\
    &= (tW_1^{(k)} + (1-t)W_1^{(k)}, \dots, tW_L^{(k)} + (1-t)W_L^{(k)}).
\end{align*}

Since $\theta(k), \theta(k+1) \in \mathcal{W}_{\rm up}(\tau_1, \dots, \tau_L)$, we have $\lVert W_l^{(k)} \rVert_2, \lVert W_l^{(k+1)} \rVert_2 \le \tau_l$. Since every norm is a convex function, for any $l \in \{1,2,\dots,L\}$, we have
\[
\lVert t W_l^{(k+1)} + (1-t) W_l^{(k)} \rVert_2 \leq t \lVert W_l^{(k+1)} \rVert_2 + (1-t) \lVert  W_l^{(k)} \rVert_2 \le t \tau_l + (1- t) \tau_l = \tau_l.
\]
This implies $\theta(k) + t \delta_k \in \mathcal{W}_{\rm up}(\tau_1, \dots, \tau_L)$.

According to Lemma \ref{lem:smoothness} and (\ref{01022307}), we have
\[
\Vert g(\theta(k) + t \delta_k) - g(\theta(k)) \Vert \leq \beta(\theta(k)) \Vert t \delta_k \Vert \leq \beta t \Vert\delta_k \Vert.
\]
Plugging this into (\ref{01022310}), we get
\[
\mathcal{L}(\theta(k+1)) \leq \mathcal{L}(\theta(k)) - \eta \Vert g_k \Vert^2 + \beta \Vert \delta_k \Vert^2 \int_0^1 t dt = \mathcal{L}(\theta(k)) - \left( \eta - \frac{\beta \eta^2}{2} \right) \Vert g_k \Vert^2.
\]
For any $0 \le \eta \le 1/\beta$, it holds that $\eta - \frac{\beta \eta^2}{2} \geq \frac{\eta}{2}$, so
\[
\mathcal{L}(\theta(k+1)) \leq \mathcal{L}(\theta(k)) - \left( \eta - \frac{\beta \eta^2}{2} \right) \Vert g_k \Vert^2 \le \mathcal{L}(\theta(k)) - \frac{\eta}{2} \Vert g_k \Vert^2.
\]
This completes the proof of the descent lemma.

According to the PL inequality (Lemma \ref{lem:pl-loss}), we get
\[
\lVert g_k\rVert^2 \ge 2\mu \mathcal{L}(\theta(k)).
\]
Combining the two inequalities above, we obtain
\[
\mathcal{L}(\theta(k+1)) \le \mathcal{L}(\theta(k)) - \frac{\eta}{2} \Vert g_k \Vert^2 \le \mathcal{L}(\theta(k)) - \eta \mu \mathcal{L}(\theta(k)) = (1- \eta \mu)\mathcal{L}(\theta(k)),
\]
for $\eta \in (0, 1/\beta]$.
This completes the proof.

\end{proof}

\newpage

\section{Proof of Corollary \ref{cor:iter}}
\label{app:cor:iter-proof}

Before proving Corollary~\ref{cor:iter}, we expand the ratio $\beta/\mu$ to make the data-dependent constant $C$ explicit. Recall from Theorem~\ref{thm:geom-cvrg} that
\[
\beta = \left(\prod_{l=1}^L \tau_l\right)^{2} \left(\sqrt{2 \mathcal{L}(\theta(0))} + \|X\|_F\right) L \sigma_{\max}(X),
\qquad
\mu = \left(\prod_{l=1}^L \mu_l\right)^{2} \sigma_{\min}(X)^2.
\]
Dividing $\beta$ by $\mu$ and grouping data-/initialization-dependent factors separately from the layer-wise condition-number bounds, we obtain
\[
\frac{\beta}{\mu}
= \underbrace{\frac{\bigl(\sqrt{2 \mathcal{L}(\theta(0))} + \|X\|_F\bigr) L \sigma_{\max}(X)}{\sigma_{\min}(X)^2}}_{=:\,C}
\cdot \left(\prod_{l=1}^L \frac{\tau_l}{\mu_l}\right)^{2}
= C\,\kappa_{\mathcal{R}}^{2L}.
\]
Here
\[
C := \frac{\bigl(\sqrt{2 \mathcal{L}(\theta(0))} + \|X\|_F\bigr) L \sigma_{\max}(X)}{\sigma_{\min}(X)^2}
\]
depends only on the data $(X, y)$, the depth $L$, and the initialization $\theta(0)$, and is independent of the layer-wise condition-number bounds $\{\tau_l, \mu_l\}_{l=1}^L$.

\begin{proof}[Proof of Corollary \ref{cor:iter}]
Take $\eta = 1/\beta$. By Theorem~\ref{thm:geom-cvrg}, for every $t\in\{0,1,\dots,T-1\}$,
\[
\mathcal{L}(\theta(t+1)) \le \Bigl(1-\frac{\mu}{\beta}\Bigr)\mathcal{L}(\theta(t)).
\]
Iterating this inequality yields the geometric bound. Formally, by induction on $t$:
for $t=0$ the claim is trivial. If
$\mathcal{L}(\theta(t)) \le (1-\mu/\beta)^t \mathcal{L}(\theta(0))$, then
\[
\mathcal{L}(\theta(t+1))
\le \Bigl(1-\frac{\mu}{\beta}\Bigr)\mathcal{L}(\theta(t))
\le \Bigl(1-\frac{\mu}{\beta}\Bigr)^{t+1}\mathcal{L}(\theta(0)).
\]
Hence, for all $t\in\{0,1,\dots,T\}$,
\[
\mathcal{L}(\theta(t))
\le
\Bigl(1-\frac{\mu}{\beta}\Bigr)^t \mathcal{L}(\theta(0)).
\]

For the iteration complexity, use the standard inequality $1-x \le e^{-x}$ for $x\in[0,1]$.
Since $\mu/\beta\in(0,1]$, we have
\[
\Bigl(1-\frac{\mu}{\beta}\Bigr)^T
\le
\exp\!\Bigl(-\frac{\mu}{\beta}T\Bigr).
\]
Therefore,
\[
\mathcal{L}(\theta(T))
\le
\exp\!\Bigl(-\frac{\mu}{\beta}T\Bigr)\,\mathcal{L}(\theta(0)).
\]
If
\[
T \ge \frac{\beta}{\mu}\log\!\Bigl(\frac{\mathcal{L}(\theta(0))}{\epsilon}\Bigr)
= C\,\kappa_{\mathcal{R}}^{2L}\log\!\Bigl(\frac{\mathcal{L}(\theta(0))}{\epsilon}\Bigr),
\]
then $\exp(-\frac{\mu}{\beta}T)\,\mathcal{L}(\theta(0)) \le \epsilon$, which implies
$\mathcal{L}(\theta(T))\le \epsilon$. This yields the claimed $T = O(\kappa_{\mathcal{R}}^{2L}\log(\mathcal{L}(\theta(0))/\epsilon))$ bound.
\end{proof}

\newpage

\section{Proof of Proposition \ref{clm:svd}}
\label{app:proof-clm-svd}

\begin{proof}
Let $W=U\Sigma V^\top$ be an SVD with $\Sigma=\mathrm{diag}(\sigma_1,\dots,\sigma_m)$ and $\sigma_i\ge 0$.
Then $WW^\top=U\Sigma^2U^\top$, and hence $p(WW^\top)=Up(\Sigma^2)U^\top$ with
$p(\Sigma^2)=\mathrm{diag}(p(\sigma_1^2),\dots,p(\sigma_m^2))$. Therefore
\[
g(W)=p(WW^\top)W
=U\,p(\Sigma^2)\Sigma\,V^\top
=U\,D\,V^\top,
\qquad
D\triangleq \mathrm{diag}\!\big(p(\sigma_i^2)\sigma_i\big)_{i=1}^m.
\]
Note that the diagonal entries of $D$ may be negative, while singular values are by definition nonnegative.
Let $|D|\triangleq \mathrm{diag}(|p(\sigma_i^2)\sigma_i|)_{i=1}^m$. Writing
$D = S\,|D|$ with $S=\mathrm{diag}(\mathrm{sign}(p(\sigma_i^2)\sigma_i))$, we have
\[
g(W)=U\,S\,|D|\,V^\top.
\]
Since $US$ is still orthogonal, this exhibits an SVD of $g(W)$ with singular values given by the diagonal
entries of $|D|$, i.e., $\{|p(\sigma_i^2)\sigma_i|\}_{i=1}^m=\{|g(\sigma_i)|\}_{i=1}^m$.
\end{proof}

\newpage
\section{Details of PC Layer Implementation}
\label{app:pc_alg_iplmt}

\subsection{Polynomial Fitting Algorithm}
\label{subapp: poly-fit-alg}

The detailed procedure of identifying polynomial $g$ is summarized in Algorithm~\ref{alg:poly-fitting}.
Additionally, note that this procedure is “offline”, i.e., it is not needed during deep net training.

\begin{algorithm}[H]
\caption{POLY-FITTING}
\label{alg:poly-fitting}
\begin{algorithmic}[1]
\STATE \textbf{Input:} target map $f:[0,1]\to\mathbb{R}$ (default $f(\sigma)=\mathrm{PL}_b(\sigma)$), fitting interval $[\gamma_L,\gamma_U]$ (default $[0,1.1]$).
\STATE \textbf{Hyper-parameters:} degree $k$ and $\alpha$ appearing in $w(\sigma) = \sigma^\alpha$.
\STATE \textbf{Sampling:} draw $\sigma_1,\ldots,\sigma_{n-1}\stackrel{\text{i.i.d.}}{\sim}\mathrm{Unif}[\gamma_L,\gamma_U]$ and set $\sigma_n=1$.
\STATE \textbf{Solve:} compute $a^\star\in\mathbb{R}^{k+1}$ by minimizing the discrete weighted least-squares objective \eqref{eq:wls_discrete}.
\STATE \textbf{Output:} $g(\sigma)=\sigma\sum_{t=0}^{k} a_t^\star \sigma^{2t}$.
\end{algorithmic}
\end{algorithm}

In Algorithm \ref{alg:poly-fitting}, we include the deterministic sample \(\sigma_n=1\) to anchor the
least-squares fit at the nominal unit scale after spectral normalization.
The enlarged fitting interval \([0,1.1]\) is used only as a robustness
margin for spectral-norm estimation error; \(\sigma=1\) remains the
reference value for the intended normalized top singular scale. Since
our target maps satisfy \(f(1)=1\), including this point encourages
\(g(1)\approx f(1)\) and prevents the finite-sample fit from
unintentionally changing the top normalized scale.

\subsection{Streaming Power Iteration for Spectral Normalization}
\label{app:streaming-pi}

Algorithm~\ref{alg:pc} normalizes each selected weight matrix by a scalar $s(W)$ that approximates the spectral norm $\|W\|_2$.
This subsection describes how $s(W)$ is computed in practice.

During training, each PC block maintains two auxiliary buffers $u$ and $v$ that
estimate the top left and right singular vectors of the current weight matrix.
At training step $t$, instead of running power iteration from scratch, we
initialize the iteration from the buffers saved at the previous step,
$(u_{t-1},v_{t-1})$. After a small number of power-iteration steps, the updated
vectors are written back to the buffers and the Rayleigh quotient is used as the
spectral-norm estimate. We refer to this procedure as \emph{streaming} power
iteration because the singular-vector estimates are carried over across
consecutive training steps as the weights evolve. This design avoids exact SVD computation and exploits the fact that weights change smoothly across consecutive training steps,
so that estimates from the previous step provide a high-quality initialization for the current step.

\begin{algorithm}[htbp]
\caption{Streaming power iteration for spectral-norm normalization}
\label{alg:streaming-pi}
\begin{algorithmic}[1]
\STATE \textbf{Input:} weight matrix $W_t\in\mathbb{R}^{n\times m}$ in \texttt{PC\_blocks} at training step $t$; persistent buffers $u_{t-1}\in\mathbb{R}^n$, $v_{t-1}\in\mathbb{R}^m$ from the previous step; number of power-iteration steps $q$ (default $q=10$); stability constant $\epsilon>0$.
\IF{$u_{t-1}, v_{t-1}$ are uninitialized}
    \STATE Initialize $u_{t-1}$ and $v_{t-1}$ as random unit vectors.
\ENDIF
\STATE $u \gets u_{t-1}$, \quad $v \gets v_{t-1}$ \hfill \textit{\# warm-start from the previous training step}
\FOR{$i=1,\ldots,q$}
    \STATE $v \gets \frac{W_t^{\top} u}{\Vert W_t^{\top} u \Vert_2}, \quad u \gets \frac{W_t\, v}{\Vert W_t\, v \Vert_2}$.
\ENDFOR
\STATE Update buffers: $u_t \gets u, v_t \gets v$.
\STATE Estimate spectral-norm: $s_t(W_t) \gets u^{\top} W_t\, v + \epsilon$.
\STATE \textbf{Return:} $s_t(W_t),\, u_t,\, v_t$.
\end{algorithmic}
\end{algorithm}

In all experiments, we use $q=10$ power-iteration steps per training step. The
streaming power-iteration routine is used only during training to estimate
$s(W_t)$. After training, we materialize the final effective weight ${\rm PC}(W)$ and
store it as the corresponding weight in the original transformer architecture;
the power-iteration buffers are then discarded and no spectral-norm estimation
is needed during inference.

\subsection{Tricks for Matrix Polynomial Calculation}
\label{app:pc_tricks}

We discuss a few computational tricks to implement a PC layer. These tricks can greatly reduce the computation cost.
In the following, suppose we decide to use a polynomial \(g(\sigma)=\sigma\,p(\sigma^{2})\) to implement the PC layer.

\paragraph{Trick 1: choose an efficient equivalent form.}
We use
\[
g(A)=A\,p(A^{\top}A)\quad\text{if \(A\) is \emph{tall}},
\qquad\text{otherwise use}\qquad
g(A)=p(AA^{\top})\,A.
\]
Note that \(g(A)=p(AA^{\top})A=A\,p(A^{\top}A)\). We call \(p(AA^{\top})A\) the \emph{\(AA^{\top}\)-form} and \(A\,p(A^{\top}A)\) the \emph{\(A^{\top}A\)-form}. Although the two forms give the same value, their implementation time can differ a lot. For \(A\in\mathbb{R}^{n\times m}\) with \(n\ge m\), forming \(A^{\top}A\) costs \(O(m^{2}n)\), while \(AA^{\top}\) costs \(O(n^{2}m)\); hence pick the \(A^{\top}A\)-form when \(A\) is tall (\(n\ge m\)), and the \(AA^{\top}\)-form when \(A\) is wide (\(n<m\)).

\paragraph{Trick 2: cache and reuse \(B\).}
Store
\[
B=\begin{cases}
A^{\top}A, & \text{if \(A\) is tall},\\[2pt]
AA^{\top}, & \text{if \(A\) is wide},
\end{cases}
\]
and reuse it. We will compute \(A\,p(B)\) (tall case) or \(p(B)\,A\) (wide case).

\paragraph{Trick 3: evaluate \(p\) via Horner's method \citep{horner1819xxi}.}
Write a degree-\(k\) polynomial as
\[
p(\sigma)=\sum_{i=0}^{k}a_i\sigma^{i}
= a_0+\sigma\bigl(a_1+\sigma\bigl(a_2+\cdots+\sigma(a_{k-1}+a_k\sigma)\bigr)\bigr).
\]
A naive implementation
\[
p(B)=\sum_{i=0}^{k}a_i B^{i}
\]
requires \(\sum_{i=1}^{k}(i-1)=\frac{k(k-1)}{2}\) matrix multiplications and \(k\) additions. Using Horner's method, we only need \(k\) matrix multiplications.
For example, for a degree-4 polynomial \(p\) (corresponding to a degree-9 \(g(\sigma)=\sigma p(\sigma^{2})\)), Horner's method reduces the work from \(6\) matrix multiplications to \(4\). For a degree-15 \(g\) (i.e., \(k=7\)), it reduces \(21\) multiplications to \(7\).

\paragraph{Summary of the three tricks.}
\begin{itemize}
\item Use the \(A^{\top}A\)-form \(A\,p(A^{\top}A)\) if \(A\) is tall, and the \(AA^{\top}\)-form \(p(AA^{\top})A\) if \(A\) is wide.
\item Store \(B=AA^{\top}\) or \(B=A^{\top}A\), and compute \(p(B)A\) or \(A\,p(B)\), respectively.
\item Use Horner's method to evaluate \(p(B)\).
\end{itemize}

\paragraph{Example.}
Suppose we choose \(g(x)=a_0 x+a_1 x^{3}+a_2 x^{5}+a_3 x^{7}\) and want to apply it to a wide matrix \(A\).
There are two steps:
\begin{enumerate}
\item Compute \(B=AA^{\top}\).
\item Compute
\[
g(A)=\bigl(a_0 + B\bigl(a_1 + B(a_2 + a_3 B)\bigr)\bigr)\,A.
\]
\end{enumerate}
If we apply it to a tall matrix \(A\), the two steps are:
\begin{enumerate}
\item Compute \(B=A^{\top}A\).
\item Compute
\[
g(A)=A\,\bigl(a_0 + B\bigl(a_1 + B(a_2 + a_3 B)\bigr)\bigr).
\]
\end{enumerate}

\newpage

\section{Computational and Memory Cost Analysis}
\label{app:pc-cost-compute-memory}

\subsection{Computational Cost Analysis}
\label{subapp:pc_cost}

We estimate the additional FLOPs introduced by the PC layer.
Following standard practice in numerical computing, we count one matrix multiplication of an $(a\times b)$ matrix with a $(b\times c)$ matrix as $\approx 2abc$ FLOPs.
We only account for matmul (matrix multiplication) FLOPs and omit lower-order elementwise/scalar operations.

\vspace{-0.5em}

\paragraph{Per-matrix overhead (forward).}
Consider a normalized weight matrix $\widetilde W\in\mathbb{R}^{n\times m}$. Without loss of generality we assume the \emph{wide-form}
\[
g(\widetilde W) = p( \widetilde W \widetilde W^\top)\, \widetilde W,
\]
where $p$ is a degree-$k$ polynomial evaluated by Horner's method. Note that the tall-form $\widetilde W\,p( \widetilde W^\top \widetilde W)$ is fully analogous; one always chooses the smaller Gram matrix to reduce cost.

Let $B=\widetilde W \widetilde W^\top\in\mathbb{R}^{n\times n}$ be the chosen Gram matrix in this form. In practice, we always choose the smaller Gram; hence the Gram dimension is $s=\min(n,m)$ and the other dimension is $\ell=\max(n,m)$.
The dominant additional matrix multiplications in the forward pass are:
\begin{enumerate}[(i)]
    \item One Gram construction ($\approx 2\ell s^2$ FLOPs);
    \item Horner calculation of a degree-$k$ polynomial $p(B)$, which requires $(k-1)$ multiplications of $s\times s$ matrices ($\approx 2(k-1)s^3$ FLOPs);
    \item One final multiplication to apply $p(B)$ to $\widetilde W$ ($\approx 2\ell s^2$ FLOPs).
\end{enumerate}
Therefore, ignoring lower-order terms, the forward overhead is approximately
\[
4\ell s^2 + 2(k-1)s^3.
\]

\vspace{-1em}

\paragraph{Spectral-norm estimation.}
To approximate the weight spectral norm, we use streaming power iteration with persistent buffers. Each power-iteration step uses two matrix-vector multiplications, $W^\top u$ and $Wv$, and therefore costs approximately $4\ell s$ FLOPs. Including the final
Rayleigh quotient, the spectral-norm estimation cost is approximately
\[
(4q+2)\ell s,
\]
where $q$ is the number of power iterations in each training step.
Combining the polynomial-preconditioning cost and the spectral-norm estimation
cost, the total additional forward FLOPs for one PC-applied matrix in a training
step are approximately
\[
\Delta \mathrm{FLOPs}_{\mathrm{PC,fwd}}
\approx
4\ell s^2 + 2(k-1)s^3 + (4q+2)\ell s .
\]
Since $s^3\le \ell s^2$, this admits the upper bound
\[
\Delta \mathrm{FLOPs}_{\mathrm{PC,fwd}}
\le
2(k+1)\ell s^2 + (4q+2)\ell s.
\]
For the matrix sizes used in our experiments, $s$ is large, so the streaming
power-iteration term is lower-order compared with the matrix-polynomial term.

\vspace{-1em}

\paragraph{Training-step overhead.}
A common rule of thumb is that, in typical deep networks, the backward pass costs about twice the forward pass in FLOPs \citep{brown2020language}, so a full training cost (forward + backward) is about $3\times$ the forward cost.
Thus, for the blocks where the PC layer is applied, the FLOPs overhead in one training step satisfies
\[
\Delta \mathrm{FLOPs}_{\mathrm{PC,train}}
\;\approx\;
3\,\Delta \mathrm{FLOPs}_{\mathrm{PC,fwd}}
\;\le\;
6(k+1)\,\ell s^2 + 6(2q+1)\, \ell s.
\]

For a standard block without PC, the baseline training per-step FLOPs (one forward and backward pass) is approximately $6nmB$, where $B$ is batch size (tokens) \citep{kaplan2020scaling}.
Using $nm=\ell s$, the relative FLOP overhead of the PC layer is bounded by
\begin{equation} \label{eq:overhead}
    \mathrm{Overhead}
    \;\triangleq\;
    \frac{\Delta \mathrm{FLOPs}_{\mathrm{PC,train}}}{6nmB}
    \;\le\;
    \frac{6(k+1)\ell s^2 + 6(2q+1)\, \ell s}{6\ell s\,B}
    \;=\;
    \frac{(k+1)s + 2q+1}{B}.
\end{equation}
Note that this bound is block-wise; the \emph{network-level} overhead is even smaller since PC is implemented only for a subset of blocks.
This overhead is small whenever the effective tokens-per-step $B$ is large compared to the model width $s$ and the PC polynomial degree $k$.

\paragraph{Instantiating with our Llama-1B setting.}
For transformer with PC layer under AdamW, we use pc\_level$=4$, i.e., a degree-$k$ polynomial with $k=4$, and $q=10$ streaming power-iteration steps for spectral-norm estimation (see Appendix~\ref{app:streaming-pi}). Our training processes
$B=2.62\times 10^6$ tokens per step.

For any PC-applied weight $W\in\mathbb{R}^{n\times m}$, let $s=\min(n,m)$ be the Gram dimension.
From Inequality~(\ref{eq:overhead}), the relative FLOPs overhead is bounded by
\[
\mathrm{Overhead}\;\le\;\frac{(k+1)s+2q+1}{B}
=\frac{5s+21}{B}.
\]
In our 1B model, PC is applied to $W_{\rm O}\in\mathbb{R}^{d\times d}$ and the FFN matrices
$W_{\rm gate}\in\mathbb{R}^{m_{\mathrm{ffn}}\times d}$, $W_{\rm up}\in\mathbb{R}^{m_{\mathrm{ffn}}\times d}$,
$W_{\rm down}\in\mathbb{R}^{d\times m_{\mathrm{ffn}}}$, where $d=2048$ and $m_{\mathrm{ffn}}=5632$ (see Appendix~\ref{app:exp_detail} for model configurations).
All these matrices have the same smaller dimension $s=d=2048$, hence
\[
\mathrm{Overhead}\;\le\;\frac{5\cdot 2048+21}{2.62\times 10^6}
\approx 3.92\times 10^{-3}
\approx 0.39\%.
\]
Therefore, the training-FLOP overhead brought by PC layer is acceptable in our setting.

\begin{rem}
When PC is combined with Muon, we empirically find that a lower polynomial degree \texttt{pc\_level}$=2$, i.e., $k=2$, is optimal (see \S\ref{subsec:main_res} for the results and Appendix~\ref{subapp:supp_disc} for an intuitive explanation).
Keeping $q=10$, our overhead bound becomes
\[
\mathrm{Overhead}(W)\;\le\;\frac{(k+1)s+2q+1}{B}=\frac{3s+21}{B}.
\]
Instantiating with $s=d=2048$ and $B=2.62\times 10^6$ tokens/step, this yields
\[
\mathrm{Overhead}(W)\;\le\;\frac{3\cdot 2048+21}{2.62\times 10^6}
\approx 2.35\times 10^{-3}
\approx 0.24\%,
\]
leading to even lower computational cost overhead.
\end{rem}

\vspace{-1em}

\subsection{Memory Cost Analysis}
\label{subapp:pc_memory}

To evaluate the memory efficiency of our proposed approach, we measured the peak memory footprint during training. All Llama-1B runs use 8$\times$NVIDIA H100 GPUs; the reported footprint is the per-GPU peak active memory on a single H100, taken as the maximum across the 8 ranks (in practice the largest value is on rank~0). Under the Llama-1B AdamW configuration, the transformer baseline incurs a per-GPU peak footprint of 65.90~GiB, whereas the PC layer consumes 72.20~GiB, an increase of approximately \textbf{9.56\%}. Under Muon, the corresponding footprints are 64.08~GiB for the baseline and 69.67~GiB with the PC layer, an increase of approximately \textbf{8.73\%}.

The extra peak memory observed with the PC layer likely comes from additional differentiable intermediate tensors that PyTorch autograd caches during the PC forward computation for use in the backward pass, such as the Gram matrix, the streaming power-iteration buffers, and the intermediate matrices produced when evaluating the polynomial $p(\cdot)$ (see Algorithm~\ref{alg:pc}). Our current implementation does not explicitly control which of these intermediates are retained, so the observed overhead may overestimate the minimum amount of state actually required for backpropagation through the PC layer. In principle, this footprint could be reduced by (i) activation recomputation (checkpointing) at the PC-block boundary, which trades additional FLOPs for lower memory, or (ii) a custom fused PC kernel with hand-managed backward, which re-materializes selected intermediates instead of relying on autograd to cache them. A more aggressive kernel-level memory optimization is left as future work.

\vspace{-0.5em}

\subsection{Summary}
\label{subapp:pc-cost-summary}

In summary, Table~\ref{tab:pc_overhead} provides an overview of the computational and memory overheads introduced by the PC layer.

\vspace{-0.5em}

\begin{table}[htbp]
\centering
\begin{tabular}{llcc}
\toprule
Type & Optimizer & Overhead & Note \\
\midrule
FLOPs & AdamW & $\le 0.39\%$ & $k=4$, $q=10$ \\
      & Muon  & $\le 0.24\%$ & $k=2$, $q=10$ \\
\midrule
Peak memory (per GPU) & AdamW & $\approx 9.56\%$ & 65.90 GiB $\to$ 72.20 GiB \\
                      & Muon  & $\approx 8.73\%$ & 64.08 GiB $\to$ 69.67 GiB \\
\bottomrule
\end{tabular}
\caption{Overhead summary of PC layer in Llama-1B experiments}
\label{tab:pc_overhead}
\end{table}

\newpage
\section{Additional Experimental Details}
\label{app:exp_detail}

\paragraph{Model configurations.}
The model configurations of the Llama 2 architecture used in our experiments are summarized in Table~\ref{tab:model_config}.

\begin{table}[htbp]
\centering
\setlength{\tabcolsep}{10pt}
\begin{tabular}{lrrrr}
\toprule
\textbf{Model} & \textbf{d\_{\text{model}}} & \textbf{n\_{\text{layers}}} & \textbf{n\_{\text{heads}}} & \textbf{hidden\_dim} \\
\midrule
\textbf{Llama-271M} & 1024 & 16 & 16 & 2816 \\
\textbf{Llama-1B}   & 2048 & 18 & 16 & 5632 \\
\bottomrule
\end{tabular}
\caption{Model configurations used in our experiments.}
\label{tab:model_config}
\end{table}

\paragraph{Batchsize realization.}
Recall from Section~\ref{subsec:exp_setting} that we use sequence length $8192$ and a fixed global batch of $2.62$M tokens per optimization step, which corresponds to $320$ sequences ($8192\times320\approx 2.62$M).
On 8 GPUs, this global batch is realized by setting the per-GPU micro-batch size and gradient accumulation to $(8, 5)$ for the 271M model and $(4, 10)$ for the 1B model, so that the total number of sequences per step satisfies $\text{micro}\times\text{accum}\times 8=320$.

\paragraph{Optimizers.}
We train all models with either AdamW \citep{kingma2014adam, loshchilov2017decoupled} or the Muon optimizer \citep{jordan2024muon}. For AdamW, we set $(\beta_1,\beta_2)=(0.9,0.95)$ and weight decay to $0.1$.
For Muon, we follow the standard practice of applying Muon only to matrix-shaped (2D) weight tensors, while using AdamW for lower-dimensional parameters; we also adopt the root mean square (RMS) matching trick \citep{liu2025muon}. Muon hyperparameters are \texttt{momentum}$=0.95$ and \texttt{ns\_steps}$=5$. Additionally, to ensure training stability, both optimizers apply gradient-norm clipping at 1.0.
The learning-rate schedule and the peak-LR grid search are described in Section~\ref{subsec:exp_setting}.

\newpage
\section{Accuracy of the Streaming Power-Iteration Spectral-Norm Estimator}
\label{app:power_iter_quality}

We empirically validate the streaming power-iteration estimator $s(W)$ used by the PC layer to approximate $\|W\|_2$ (Algorithm~\ref{alg:streaming-pi}).
On the LLaMA2-271M AdamW run we compare $s(W)$ against the ground-truth $\|W\|_2$ from a full SVD of the same weight, and report the relative error
\begin{equation*}
\mathrm{relerr}(W) \;=\; \frac{|s(W) - \|W\|_2|}{\|W\|_2}.
\end{equation*}
For each of the four preconditioned matrix types ($W_{\rm O}$, $W_{\rm gate}$, $W_{\rm up}$, $W_{\rm down}$) we aggregate $\mathrm{relerr}$ across the $16$ transformer blocks at every step, plotting per-step median and per-step maximum over the full $22{,}000$-step trajectory (Figure~\ref{fig:relerr-powiter}).

\begin{figure}[htbp]
    \centering
    \begin{subfigure}[t]{0.48\textwidth}
        \centering
        \includegraphics[width=\linewidth]{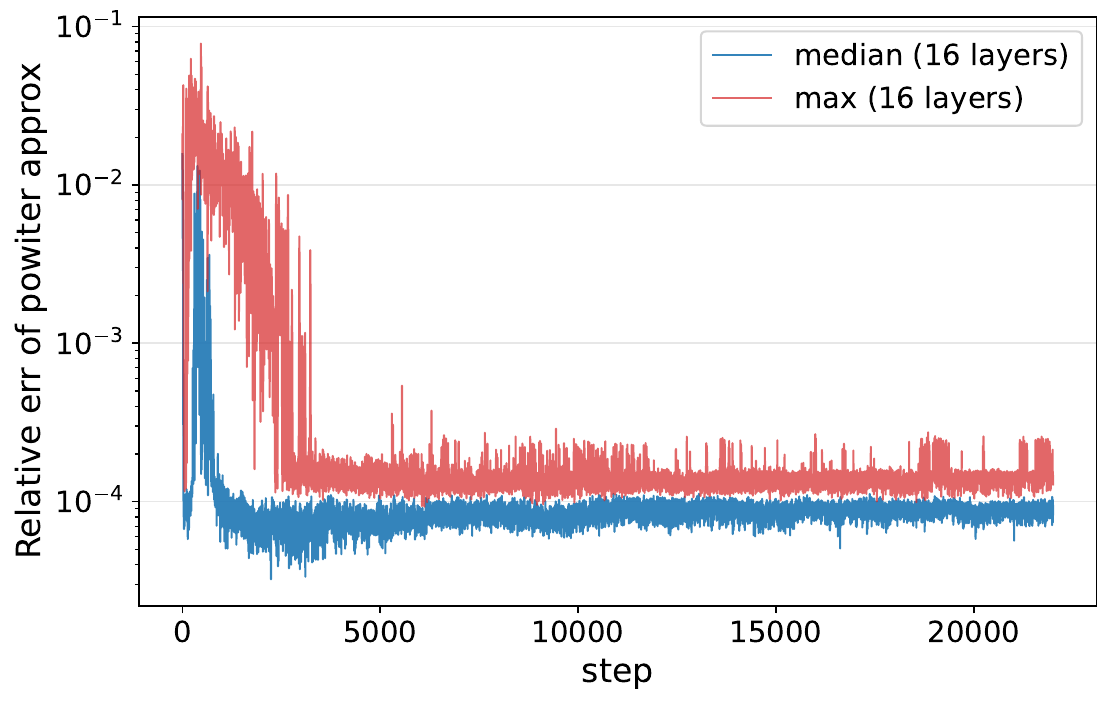}
        \caption{$W_{\rm O}$ (attention output projection)}
        \label{fig:relerr-wo}
    \end{subfigure}\hfill
    \begin{subfigure}[t]{0.48\textwidth}
        \centering
        \includegraphics[width=\linewidth]{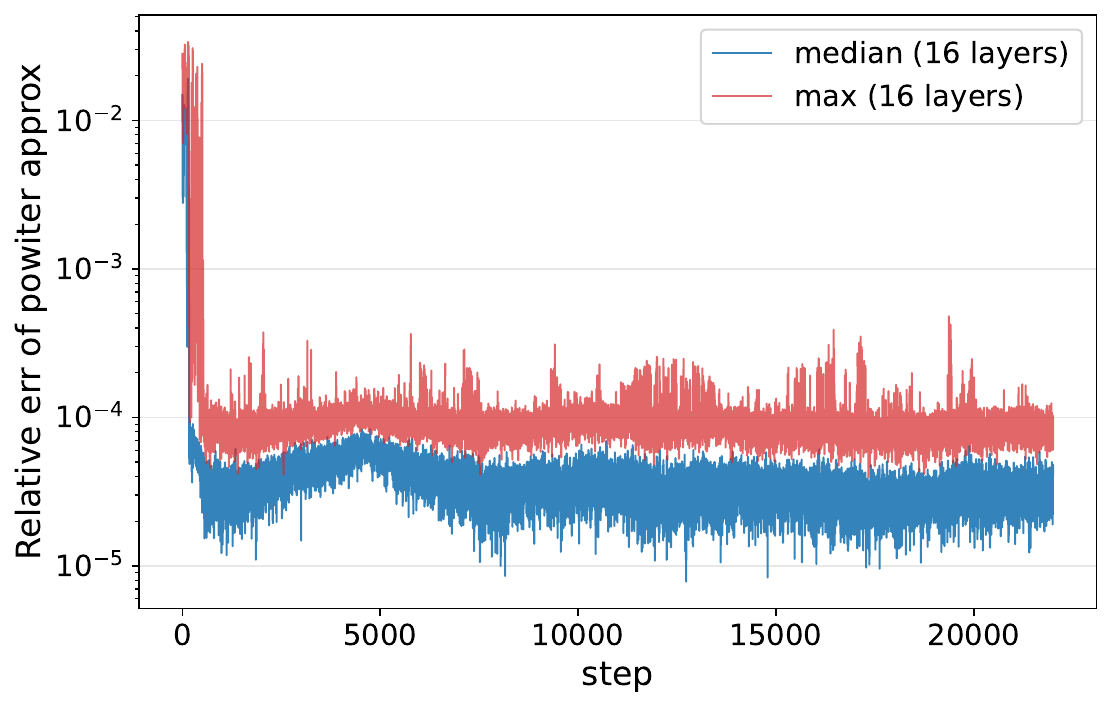}
        \caption{$W_{\rm gate}$ (FFN gate projection)}
        \label{fig:relerr-w1}
    \end{subfigure}

    \vspace{0.6em}

    \begin{subfigure}[t]{0.48\textwidth}
        \centering
        \includegraphics[width=\linewidth]{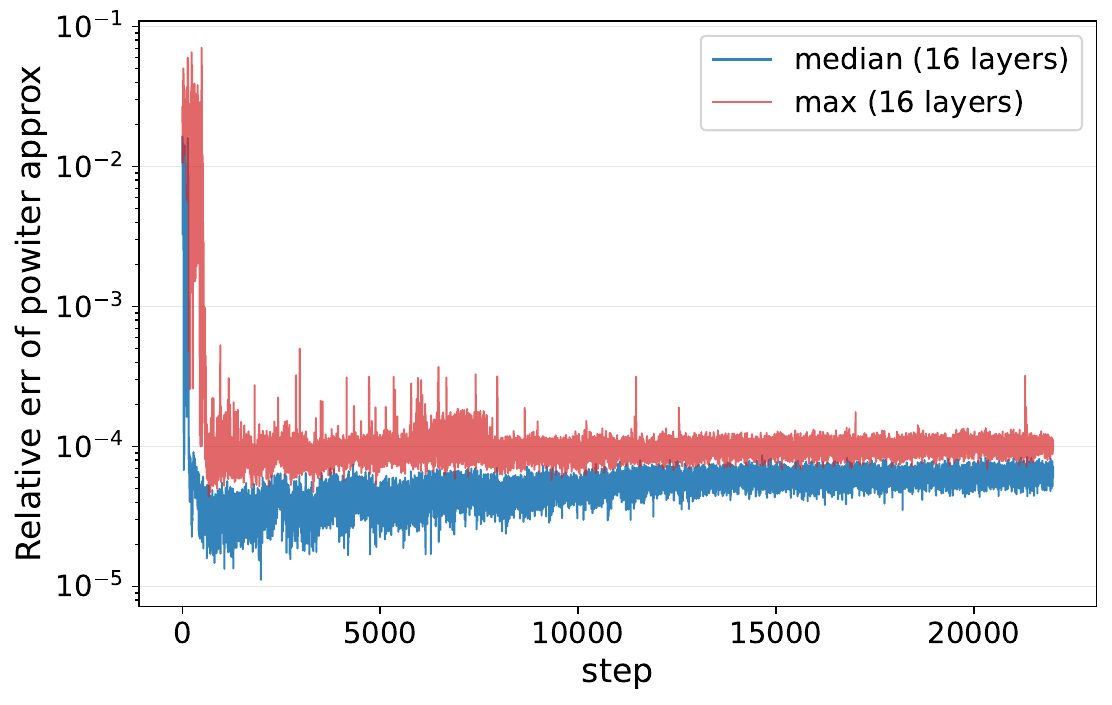}
        \caption{$W_{\rm up}$ (FFN up projection)}
        \label{fig:relerr-w3}
    \end{subfigure}\hfill
    \begin{subfigure}[t]{0.48\textwidth}
        \centering
        \includegraphics[width=\linewidth]{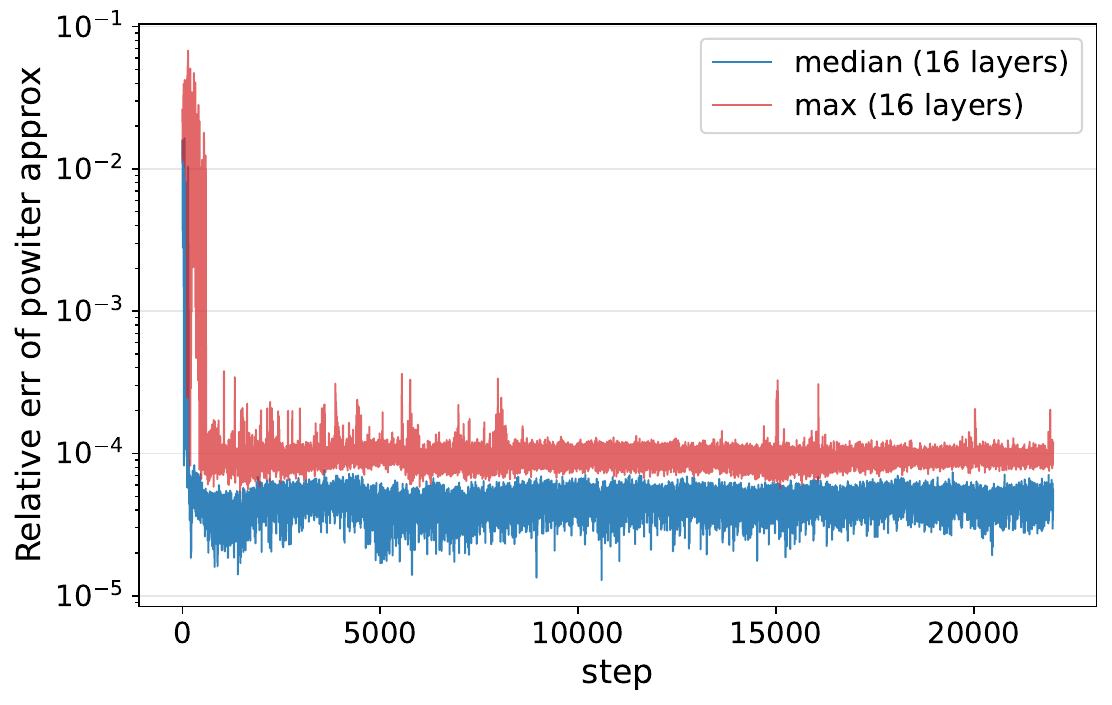}
        \caption{$W_{\rm down}$ (FFN down projection)}
        \label{fig:relerr-w2}
    \end{subfigure}
    \caption{\textbf{Relative error of the streaming power-iteration estimator $s(W)$ on LLaMA2-271M (AdamW, $22{,}000$ steps).}
    Per-step median and per-step maximum of $\mathrm{relerr}(W)$ aggregated over the $16$ transformer blocks.}
    \label{fig:relerr-powiter}
\end{figure}

\paragraph{Initial transient and steady state.}
During the first few hundred steps the weights drift rapidly while the warm-started $(u,v)$ buffers have not yet aligned with the dominant singular subspace, so the estimator exhibits a short transient, during which the per-step maximum across all four matrix types stays below $0.08$.
After roughly $3{,}000$ steps the estimator enters a tight steady state: the per-step maximum stays below $4\!\times\!10^{-3}$ and the per-step median drops below $1.5\!\times\!10^{-4}$ for the remainder of training, i.e., $s(W)$ tracks $\|W\|_2$ to three to four significant digits across all $16$ layers.

\paragraph{Behavior across matrix types.}
The output projection $W_{\rm O}$ is the slowest to converge: its per-step maximum keeps producing intermittent spikes up to $\sim\!10^{-2}$ until step $\approx\!3{,}000$, whereas the corresponding maxima for $W_{\rm gate}, W_{\rm up}, W_{\rm down}$ settle to their steady-state floor by step $\approx\!1{,}000$.
This is consistent with $W_{\rm O}$ having a tighter relative gap between its two leading singular values, which slows the geometric rate of power iteration.
Even so, $W_{\rm O}$ enters the same steady-state band as the other three matrices, confirming that ten power-iteration sweeps with warm-starting suffice even for the hardest case.

\paragraph{Implication for the $[0,1.1]$ fitting interval.}
Both regimes above are comfortably absorbed by the $10\%$ safety margin built into the fitting interval $[\gamma_L,\gamma_U]=[0,1.1]$ (Section~\ref{sec:find-poly}): in the worst case observed, the normalized spectrum $W/s(W)$ can overshoot $1$ by at most $\sim\!0.08 < 0.1$, so the polynomial $g$ is still evaluated inside its design domain throughout training.

\newpage
\section{Overly Aggressive Spectrum Flattening}
\label{app:overflatten}

Our PC polynomial is designed to improve conditioning without collapsing the entire singular-value
spectrum. To test the limiting case, we replace the default PC polynomial with an
\emph{over-flattening} polynomial whose scalar map nearly sends every nonzero normalized singular value
to one, as shown in Figure~\ref{fig:overflatten-map}. Concretely, we use the Polar Express composite
of degree-$5$ Newton--Schulz polynomials \citep{amsel2025polar}, which approximates the matrix-sign map.
On the normalized domain $\sigma\in[0,1]$, the over-flattening map is the composition of $T=8$ odd
degree-$5$ polynomials, $p(\sigma) = \big(p_8 \circ p_7 \circ \cdots \circ p_1\big)(\sigma)$, where
\[
\begin{aligned}
p_1(\sigma) &= 7.2086\,\sigma - 15.5131\,\sigma^3 + 9.0178\,\sigma^5, &
p_2(\sigma) &= 3.9623\,\sigma - 2.5813\,\sigma^3 + 0.4542\,\sigma^5, \\
p_3(\sigma) &= 3.9466\,\sigma - 2.5765\,\sigma^3 + 0.4544\,\sigma^5, &
p_4(\sigma) &= 3.8991\,\sigma - 2.5671\,\sigma^3 + 0.4566\,\sigma^5, \\
p_5(\sigma) &= 3.7186\,\sigma - 2.5308\,\sigma^3 + 0.4653\,\sigma^5, &
p_6(\sigma) &= 3.1390\,\sigma - 2.3073\,\sigma^3 + 0.4733\,\sigma^5, \\
p_7(\sigma) &= 2.1715\,\sigma - 1.5246\,\sigma^3 + 0.3885\,\sigma^5, &
p_8(\sigma) &= 1.8648\,\sigma - 1.2224\,\sigma^3 + 0.3577\,\sigma^5 .
\end{aligned}
\]
All other PC settings are kept identical to the per-optimizer default ($\texttt{PC\_blocks}=\{\texttt{ffn}, W_{\rm O}\}$,
with $\texttt{pc\_level}=4$ under AdamW and $\texttt{pc\_level}=2$ under Muon).

\begin{figure}[htbp]
  \centering
  \includegraphics[width=0.8\textwidth]{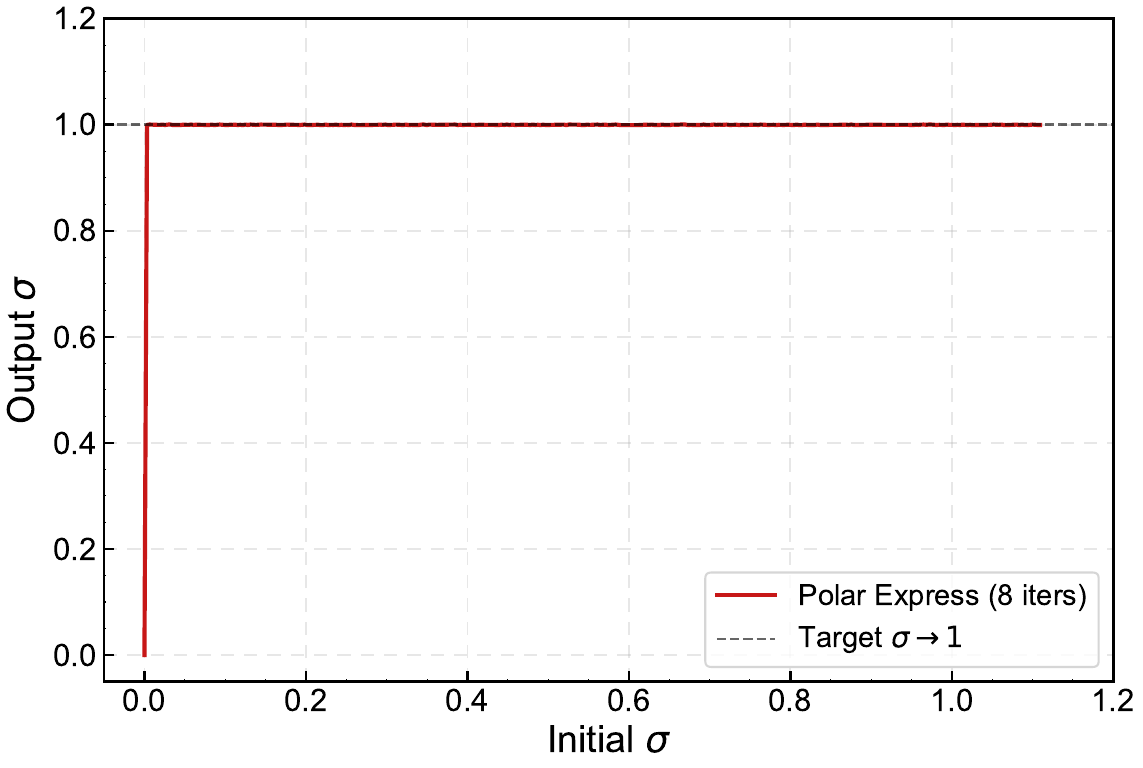}
  \caption{\textbf{Over-flattening scalar map.}
  The Polar Express composite of degree-$5$ Newton--Schulz polynomials \citep{amsel2025polar}
  nearly maps every nonzero normalized singular value to one.}
  \label{fig:overflatten-map}
\end{figure}

Figure~\ref{fig:overflatten-loss} shows that this aggressive flattening is harmful under both
optimizers: in each case the over-flattened polynomial gives higher validation loss than the
transformer baseline, whereas the softer default PC polynomial remains better than the baseline.
This suggests that PC should not aim to perfectly orthogonalize the selected weight matrices at every
training step. Preserving some spectral anisotropy appears important for expressiveness and
optimization. We therefore use moderate piecewise-linear targets, which improve conditioning while
retaining spectral flexibility.

\begin{figure}[htbp]
  \centering
  \begin{minipage}{0.48\textwidth}
    \centering
    \includegraphics[width=\linewidth]{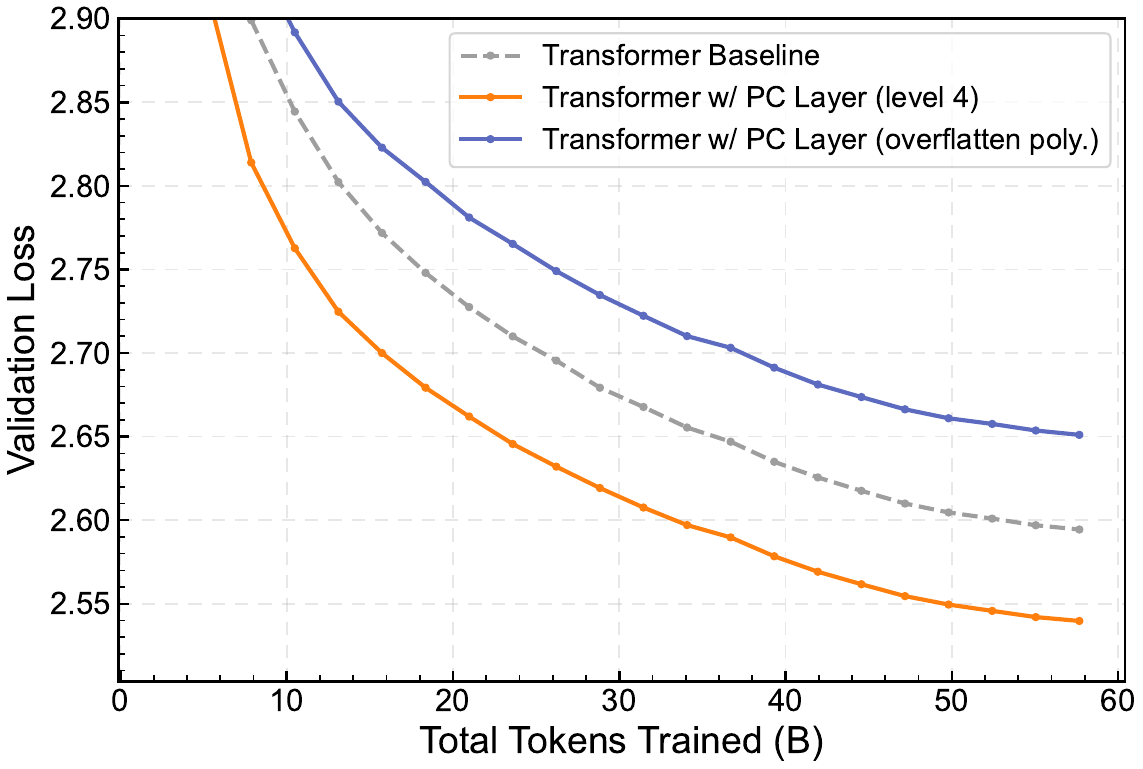}
    \vspace{-0.5em}

    {\small (a) AdamW.}
  \end{minipage}
  \hfill
  \begin{minipage}{0.48\textwidth}
    \centering
    \includegraphics[width=\linewidth]{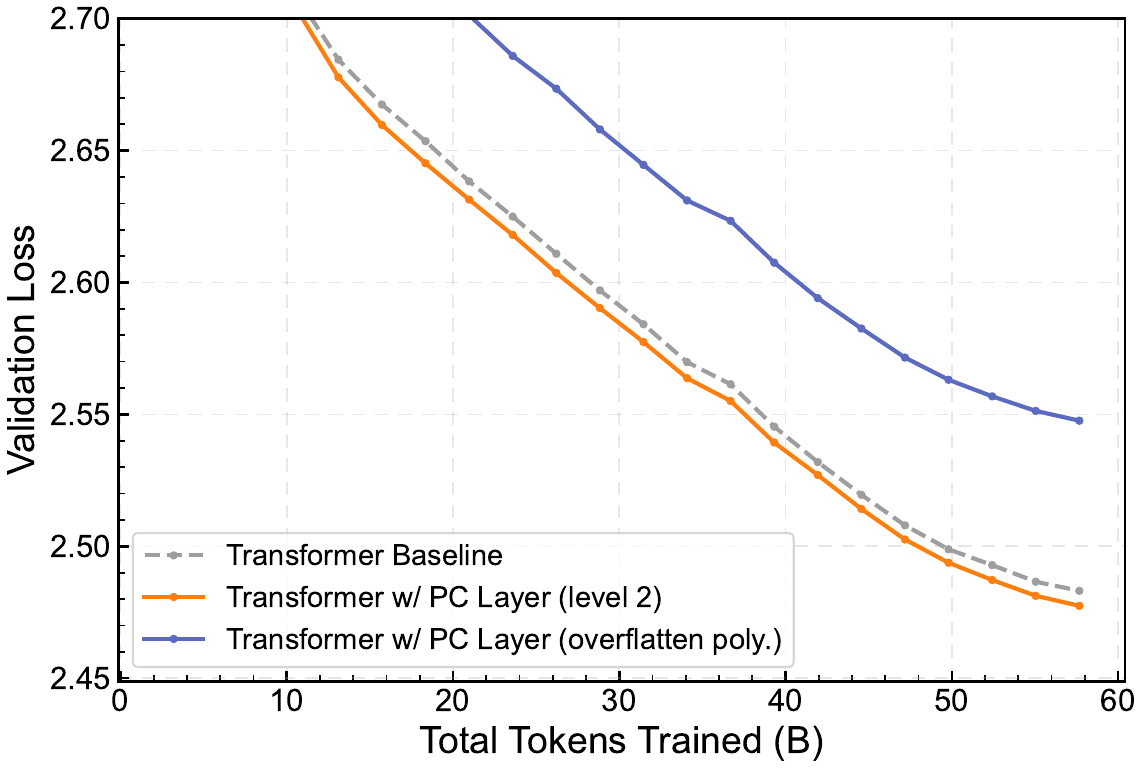}
    \vspace{-0.5em}

    {\small (b) Muon.}
  \end{minipage}

  \caption{\textbf{Overly aggressive spectrum flattening hurts under both optimizers.}
  Validation loss on Llama-271M under (a) AdamW and (b) Muon. In both cases the over-flattened
  polynomial performs worse than the transformer baseline, while the default PC layer
  ($\texttt{pc\_level}=4$ for AdamW, $\texttt{pc\_level}=2$ for Muon) stays below the baseline,
  supporting soft spectrum shaping over near-perfect flattening.}
  \label{fig:overflatten-loss}
\end{figure}

\clearpage

\clearpage
\section{Additional Results for PC Layer with Muon}
\label{app:muon-results}

Notably, since the PC layer and the Muon optimizer operate on different aspects of training (update matrices vs. weight matrices), they are largely orthogonal and thus compatible.
The main-text validation-loss curves and downstream comparison (Figure~\ref{fig:pc-scaling-muon}, Table~\ref{tab:downstream-1b}) already establish that PC improves training under Muon, and the \texttt{pc\_level}/\texttt{PC\_blocks} ablations under Muon are reported alongside the AdamW ablations in Section~\ref{subsec:ab_study}. This section provides the remaining supporting details: we first examine the effects on spectral conditioning (\S\ref{app:muon-spectrum}), and then offer additional discussions on the relationship between PC layer and Muon, including the computational template shared by both methods and their respective impact on weight-spectrum control (\S\ref{subapp:supp_disc}).

\subsection{Spectral Conditioning}
\label{app:muon-spectrum}
We then assess PC's effect on weight spectral conditioning under Muon by tracking the global modified condition number over training and visualizing final-checkpoint singular-value spectra across representative depths and blocks.

\paragraph{Global modified condition number (GMCN).}
Figure~\ref{fig:1b-muon-mcn} shows the evolution of the Global Modified Condition Number (GMCN) $\tilde{\kappa}$ throughout training under Muon.
Enabling PC shifts the curve downward on the blocks to which PC is applied (\texttt{ffn} and $W_{\rm O}$), while the non-preconditioned $W_{\rm Q}$, $W_{\rm K}$, and $W_{\rm V}$ blocks do not deteriorate and are slightly improved. The global aggregate therefore remains lower under PC, indicating improved spectral conditioning of weights. \\

\begin{figure}[htbp]
    \centering
    \begin{subfigure}[t]{0.32\textwidth}
        \centering
        \includegraphics[width=\linewidth]{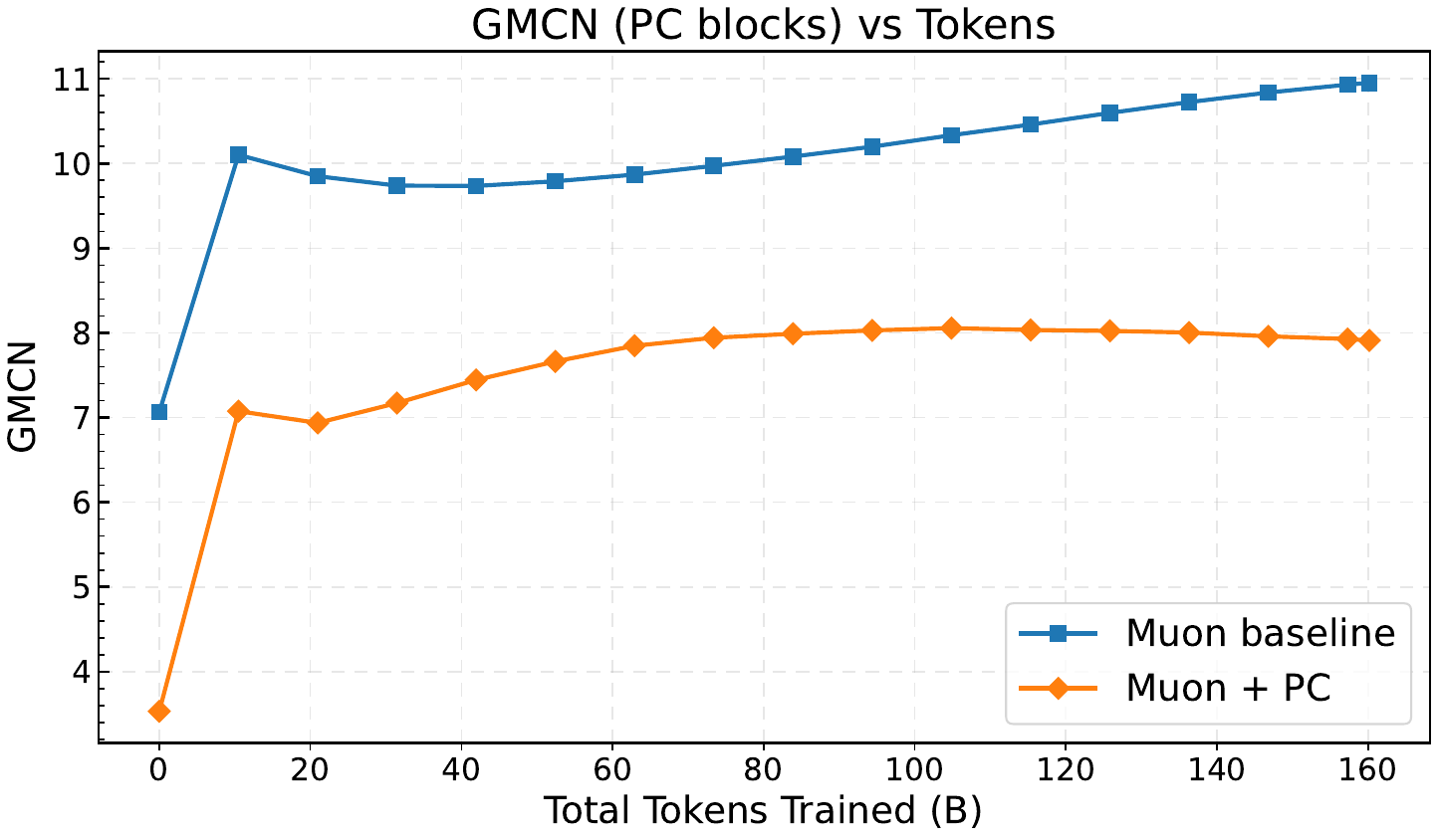}
        \caption{FFN + $W_{\rm O}$}
        \label{fig:1b-muon-mcn-o-ffn}
    \end{subfigure}\hfill
    \begin{subfigure}[t]{0.32\textwidth}
        \centering
        \includegraphics[width=\linewidth]{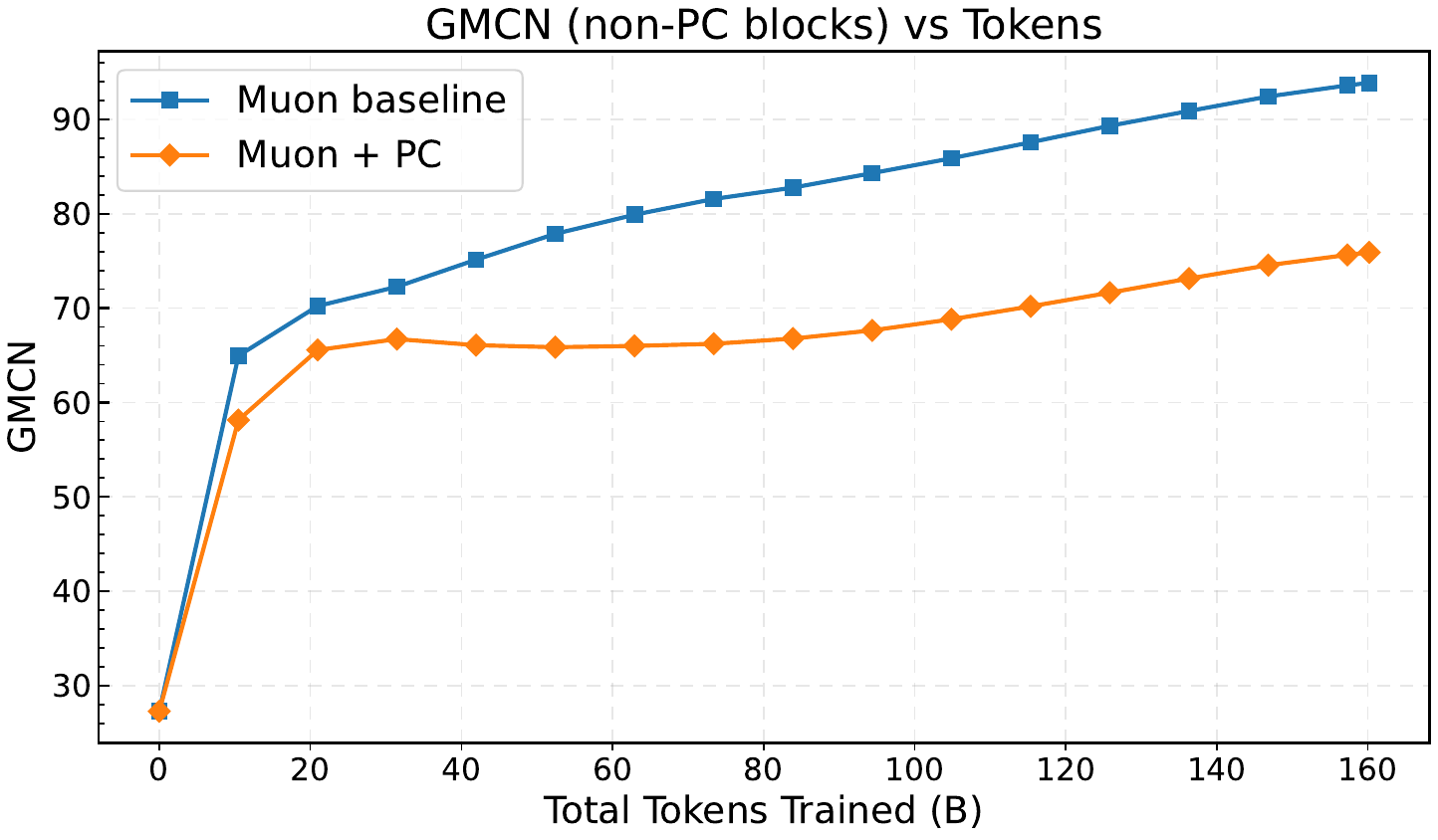}
        \caption{$W_{\rm Q}, W_{\rm K}, W_{\rm V}$}
        \label{fig:1b-muon-mcn-qkv}
    \end{subfigure}\hfill
    \begin{subfigure}[t]{0.32\textwidth}
        \centering
        \includegraphics[width=\linewidth]{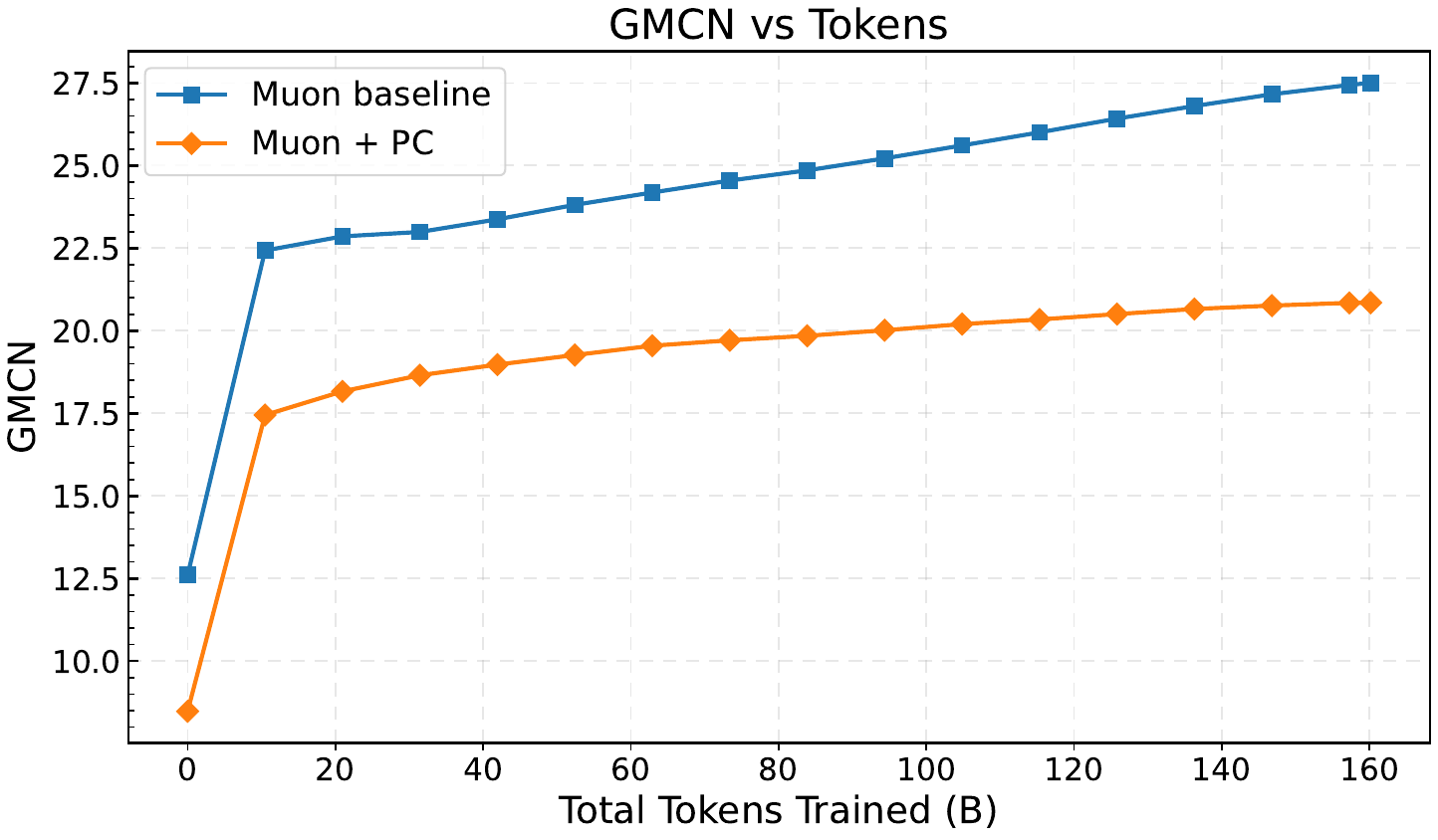}
        \caption{Global}
        \label{fig:1b-muon-mcn-global}
    \end{subfigure}
    \caption{\textbf{Block-wise and aggregate modified condition numbers under Muon.}
    We track $\tilde{\kappa}$ separately for the PC-targeted blocks (\texttt{ffn} and $W_{\rm O}$), the attention-input blocks left outside PC ($W_{\rm Q}, W_{\rm K}, W_{\rm V}$), and the resulting global aggregate. The targeted blocks show the clearest conditioning gain, and the untargeted QKV blocks also exhibit a noticeable improvement relative to the baseline, further contributing to the lower aggregate curve.}
    \label{fig:1b-muon-mcn}
\end{figure}

\paragraph{Singular-value spectrum comparison.}
To complement the scalar GMCN curves, we visualize the final-checkpoint singular-value distributions following the identical protocol as in Section~\ref{sec:pc_improve_spec} (same three representative depths, baseline spectra on $W$ versus PC spectra on $\mathrm{PC}(W)$, with per-matrix max-normalization to $[0,1]$); see Figure~\ref{fig:pc-sv-hist-muon}.
Figure~\ref{fig:pc-sv-hist-muon} corroborates the GMCN trend. Across depths and blocks, PC pulls the lower part of the spectrum upward toward $\sigma_{\max}$, and this is most visible in the \texttt{ffn} blocks. As a result, the relative spectral spread narrows and the conditioning improves.

\begin{figure}[t!]
    \centering
    \begin{subfigure}[t]{0.24\textwidth}
        \centering
        \includegraphics[width=\linewidth]{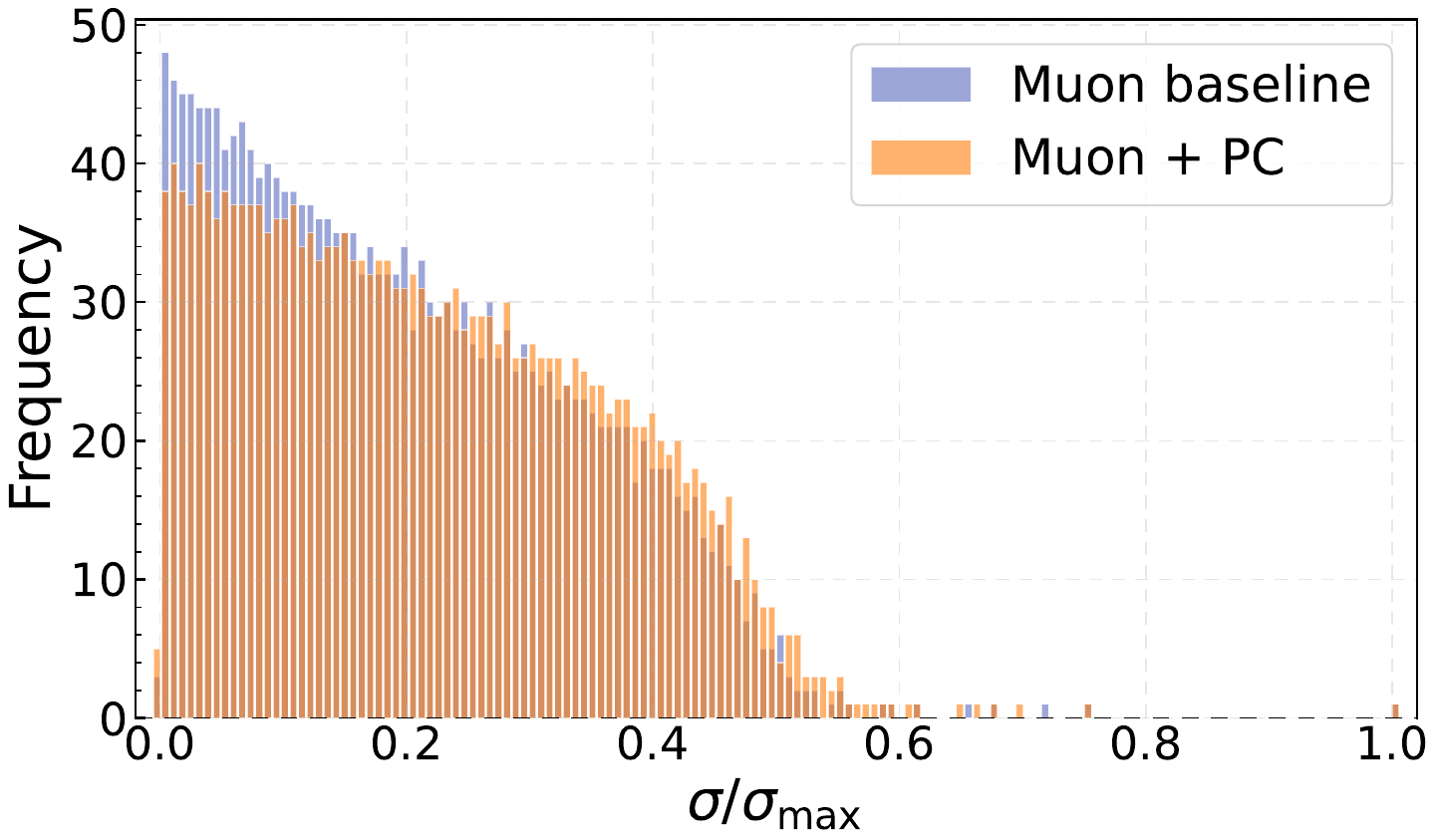}
        \caption{Layer 2: $W_{\rm O}$}
        \label{fig:muon-spec-l2-o}
    \end{subfigure}\hfill
    \begin{subfigure}[t]{0.24\textwidth}
        \centering
        \includegraphics[width=\linewidth]{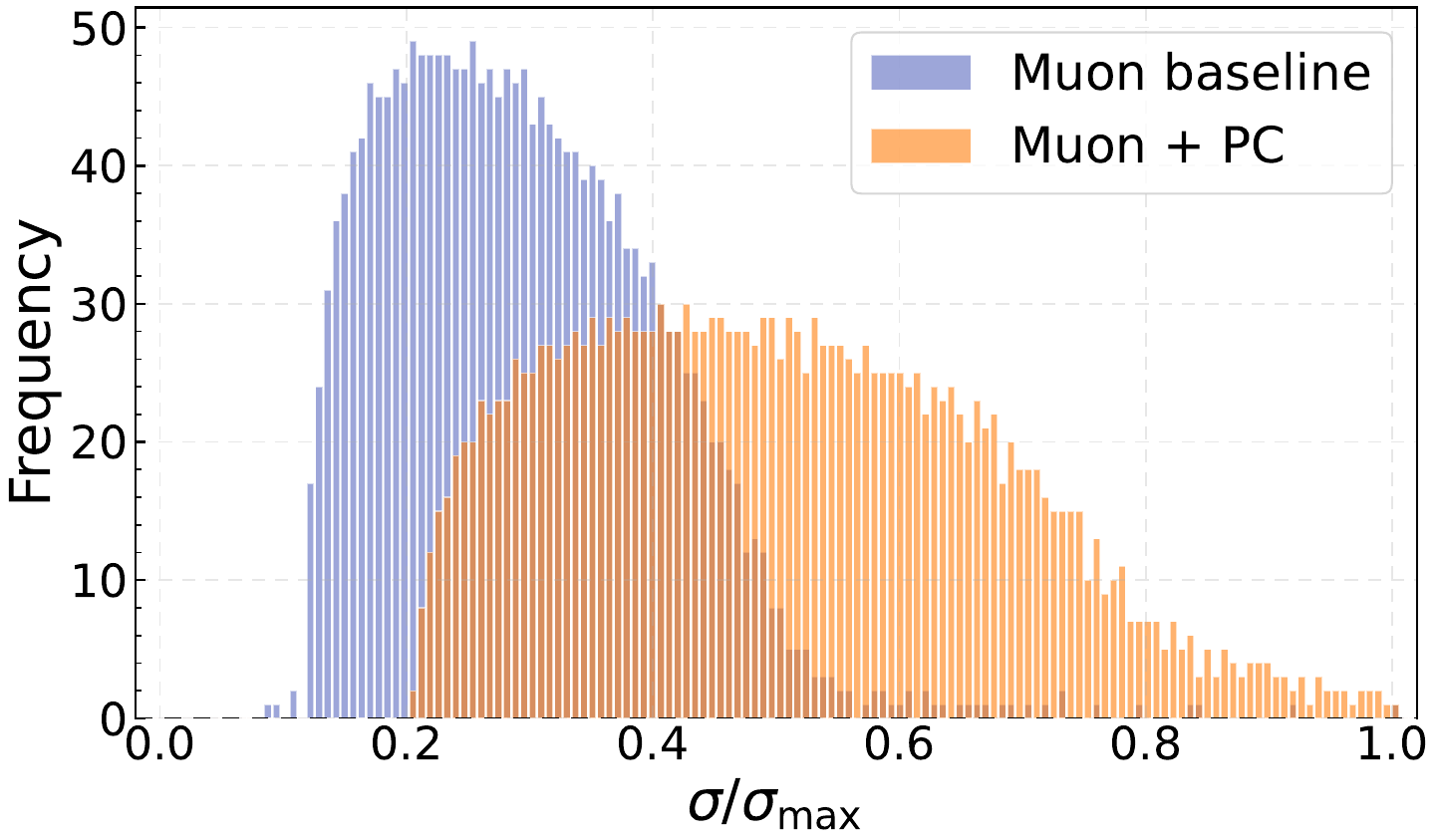}
        \caption{Layer 2: $W_{\rm gate}$}
        \label{fig:muon-spec-l2-w1}
    \end{subfigure}\hfill
    \begin{subfigure}[t]{0.24\textwidth}
        \centering
        \includegraphics[width=\linewidth]{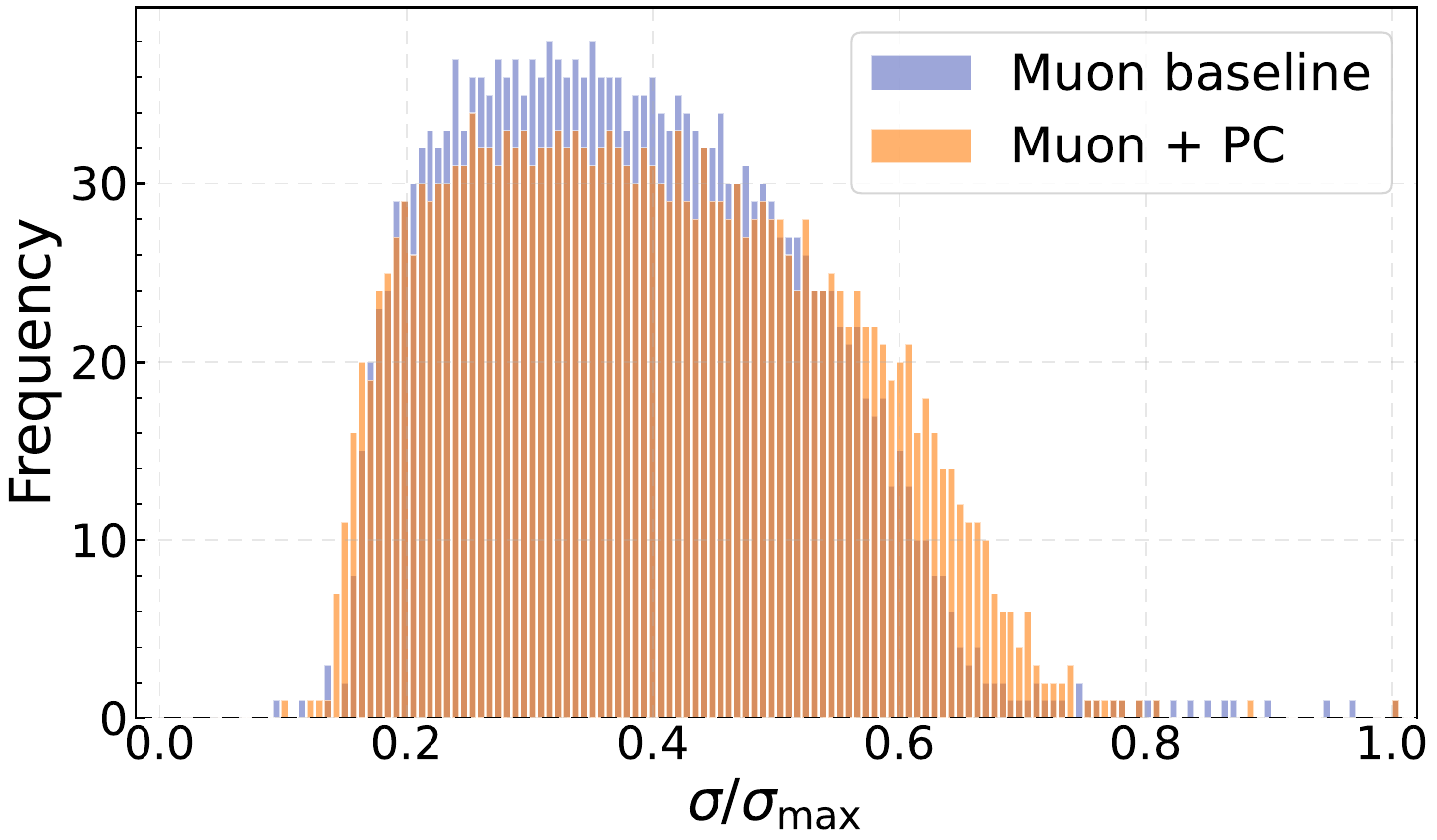}
        \caption{Layer 2: $W_{\rm up}$}
        \label{fig:muon-spec-l2-w3}
    \end{subfigure}\hfill
    \begin{subfigure}[t]{0.24\textwidth}
        \centering
        \includegraphics[width=\linewidth]{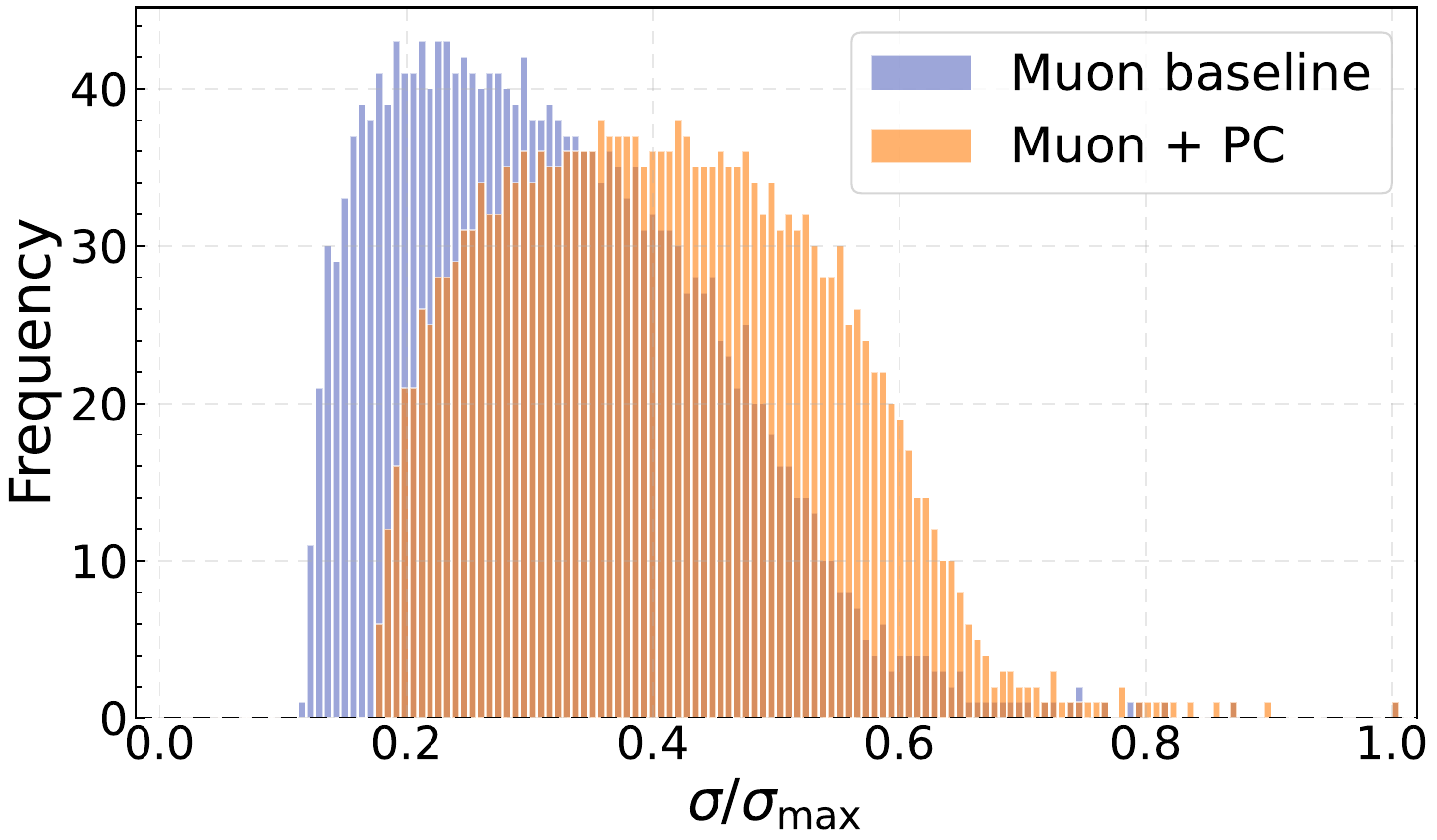}
        \caption{Layer 2: $W_{\rm down}$}
        \label{fig:muon-spec-l2-w2}
    \end{subfigure}

    \vspace{0.6em}

    \begin{subfigure}[t]{0.24\textwidth}
        \centering
        \includegraphics[width=\linewidth]{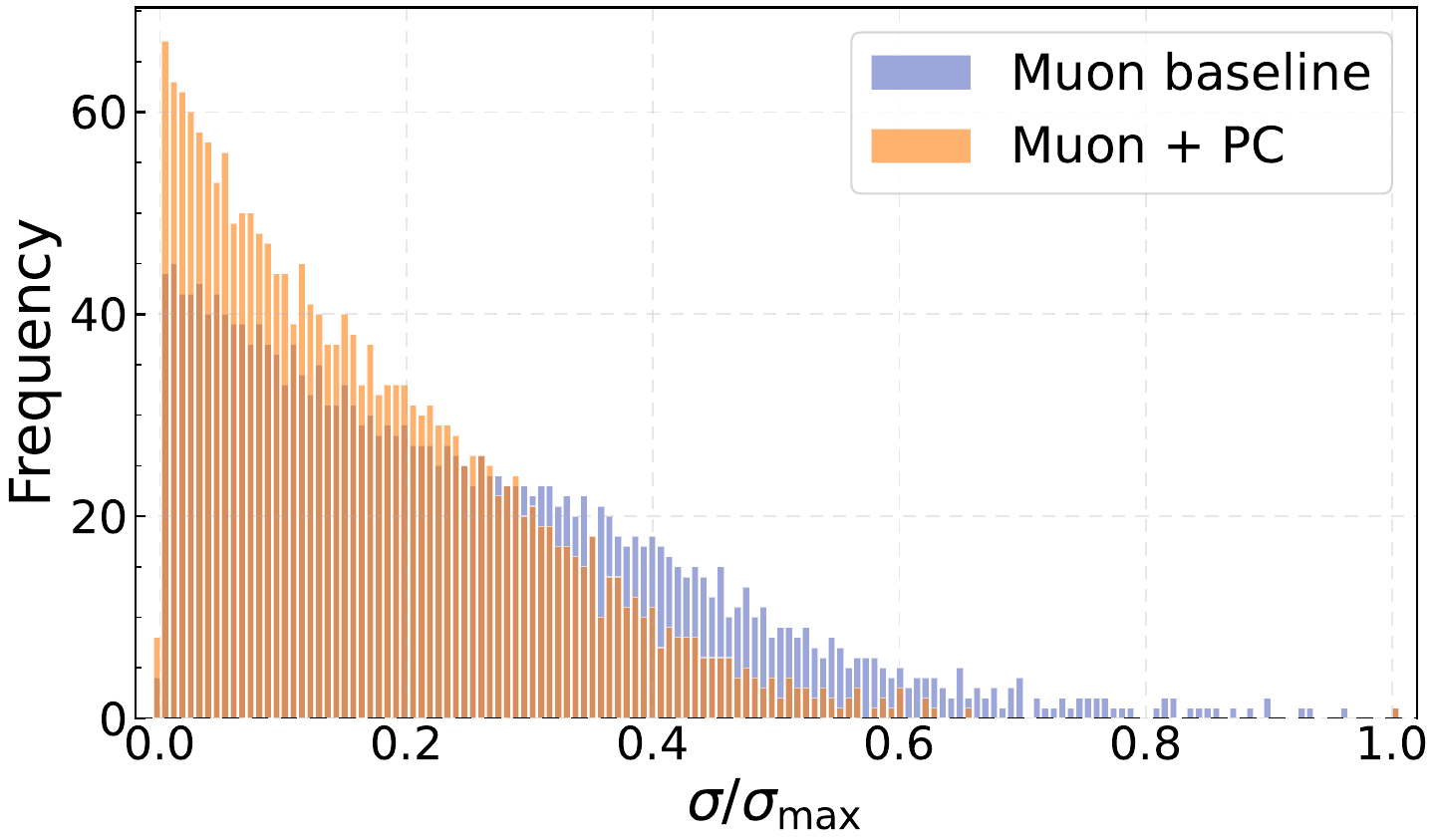}
        \caption{Layer 10: $W_{\rm O}$}
        \label{fig:muon-spec-l10-o}
    \end{subfigure}\hfill
    \begin{subfigure}[t]{0.24\textwidth}
        \centering
        \includegraphics[width=\linewidth]{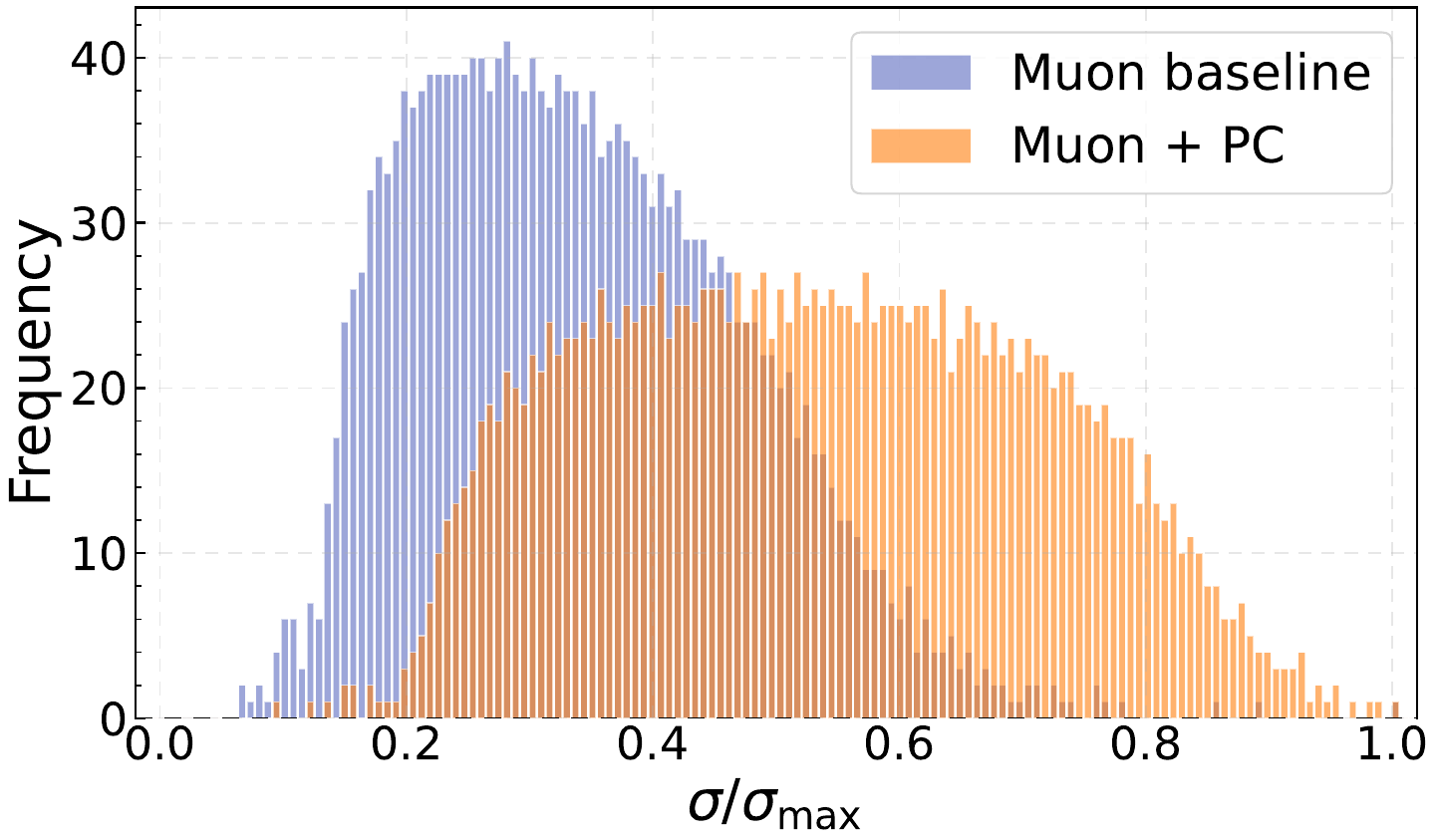}
        \caption{Layer 10: $W_{\rm gate}$}
        \label{fig:muon-spec-l10-w1}
    \end{subfigure}\hfill
    \begin{subfigure}[t]{0.24\textwidth}
        \centering
        \includegraphics[width=\linewidth]{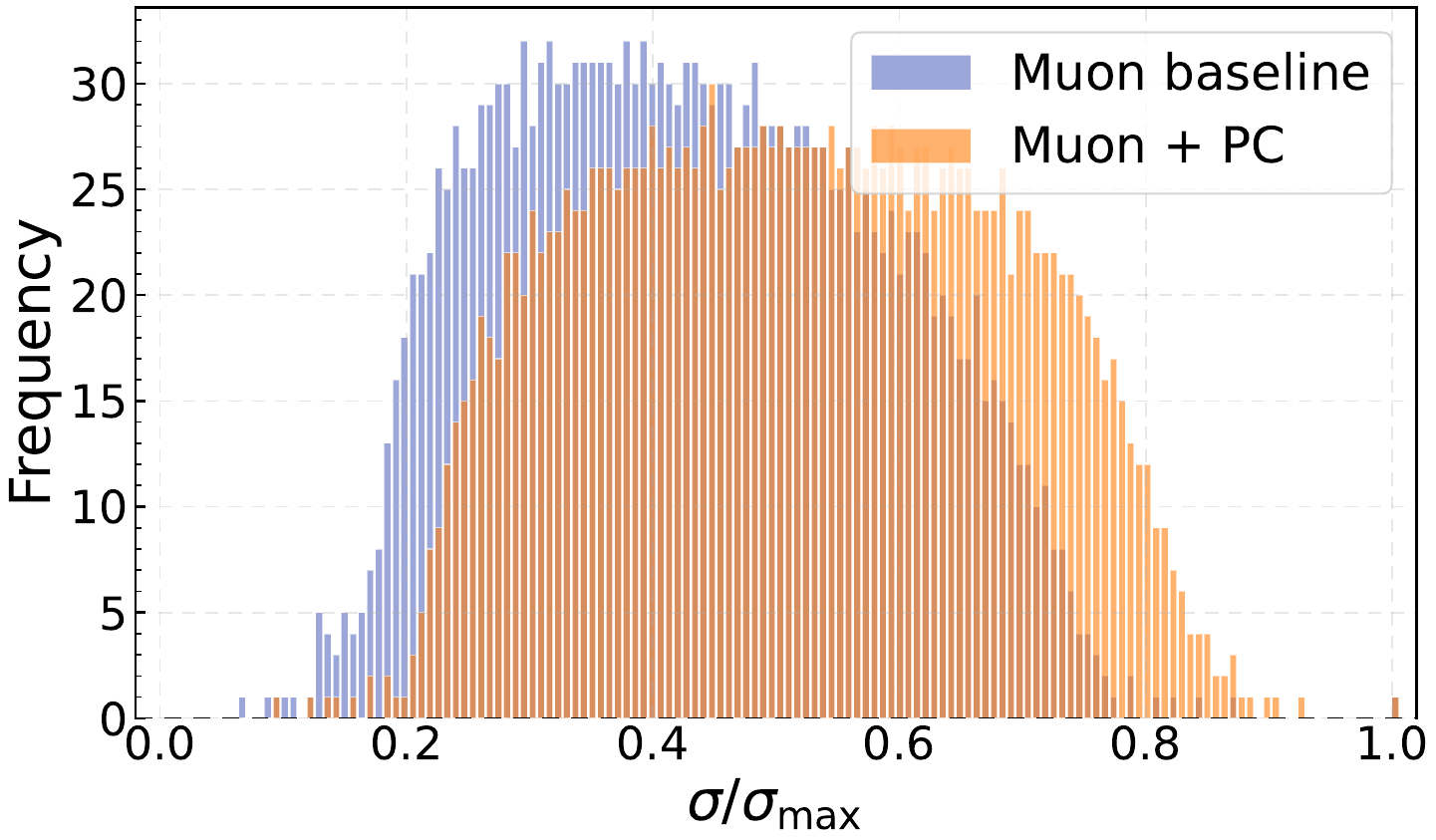}
        \caption{Layer 10: $W_{\rm up}$}
        \label{fig:muon-spec-l10-w3}
    \end{subfigure}\hfill
    \begin{subfigure}[t]{0.24\textwidth}
        \centering
        \includegraphics[width=\linewidth]{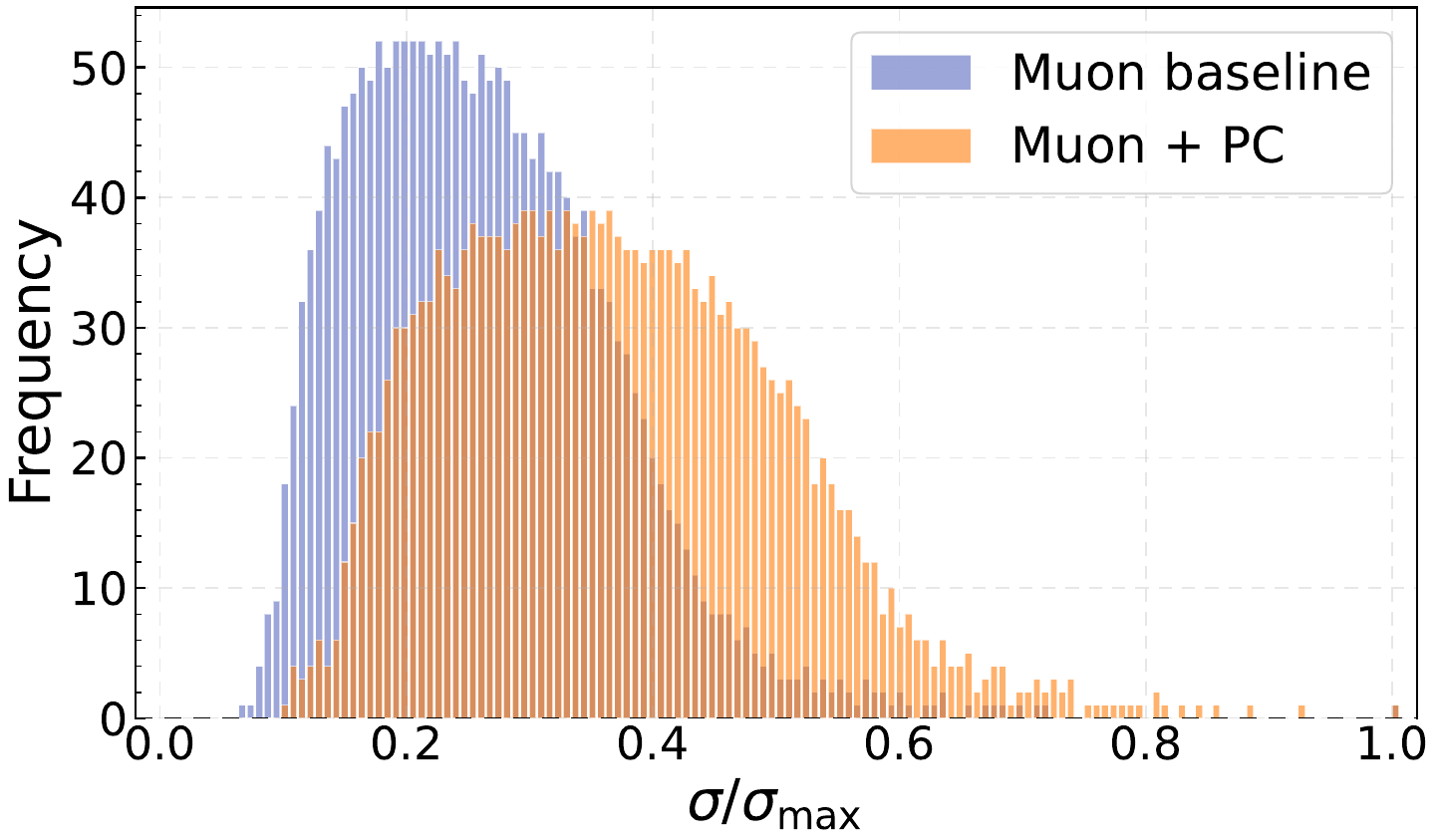}
        \caption{Layer 10: $W_{\rm down}$}
        \label{fig:muon-spec-l10-w2}
    \end{subfigure}

    \vspace{0.6em}

    \begin{subfigure}[t]{0.24\textwidth}
        \centering
        \includegraphics[width=\linewidth]{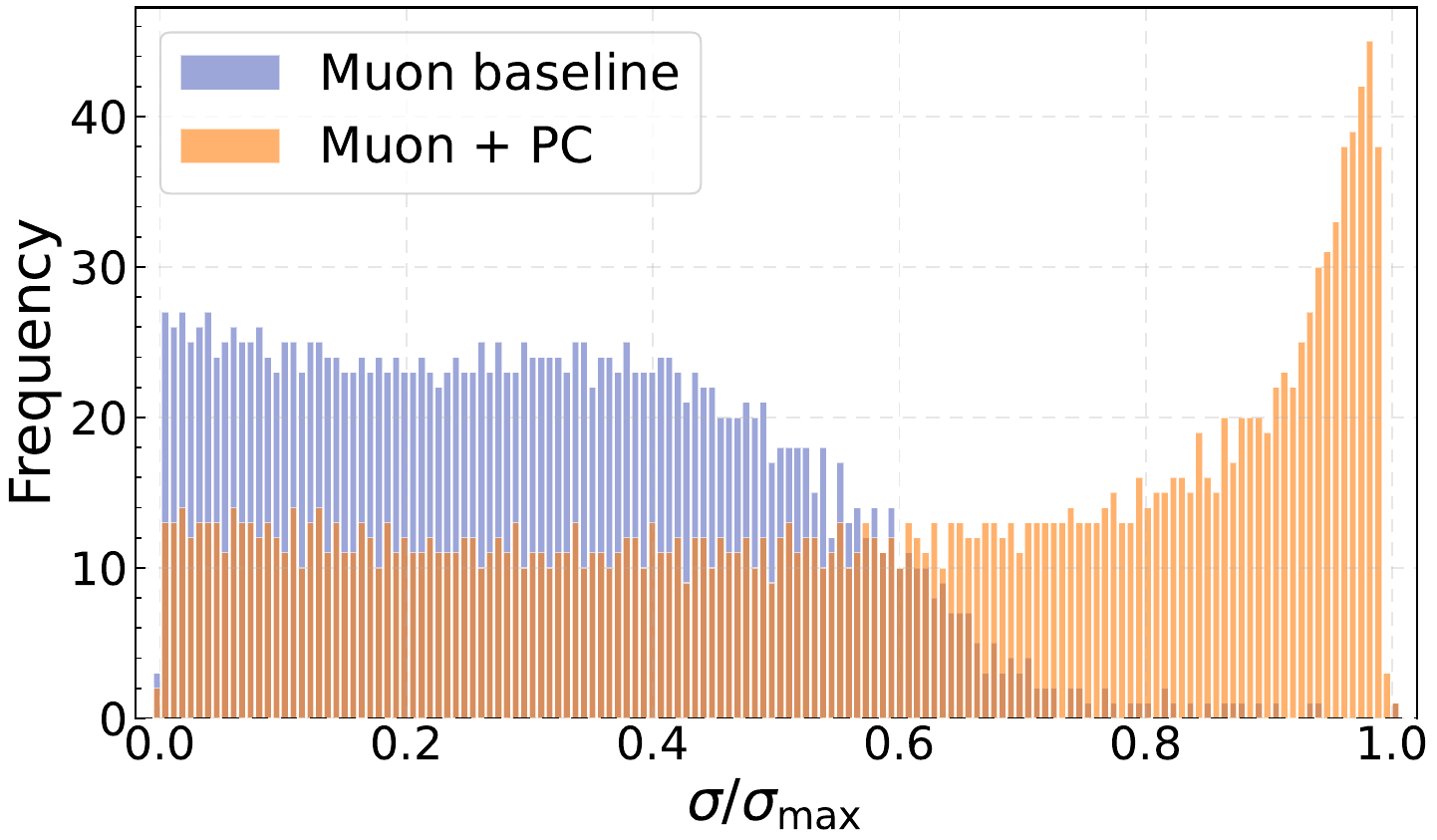}
        \caption{Layer 18: $W_{\rm O}$}
        \label{fig:muon-spec-l18-o}
    \end{subfigure}\hfill
    \begin{subfigure}[t]{0.24\textwidth}
        \centering
        \includegraphics[width=\linewidth]{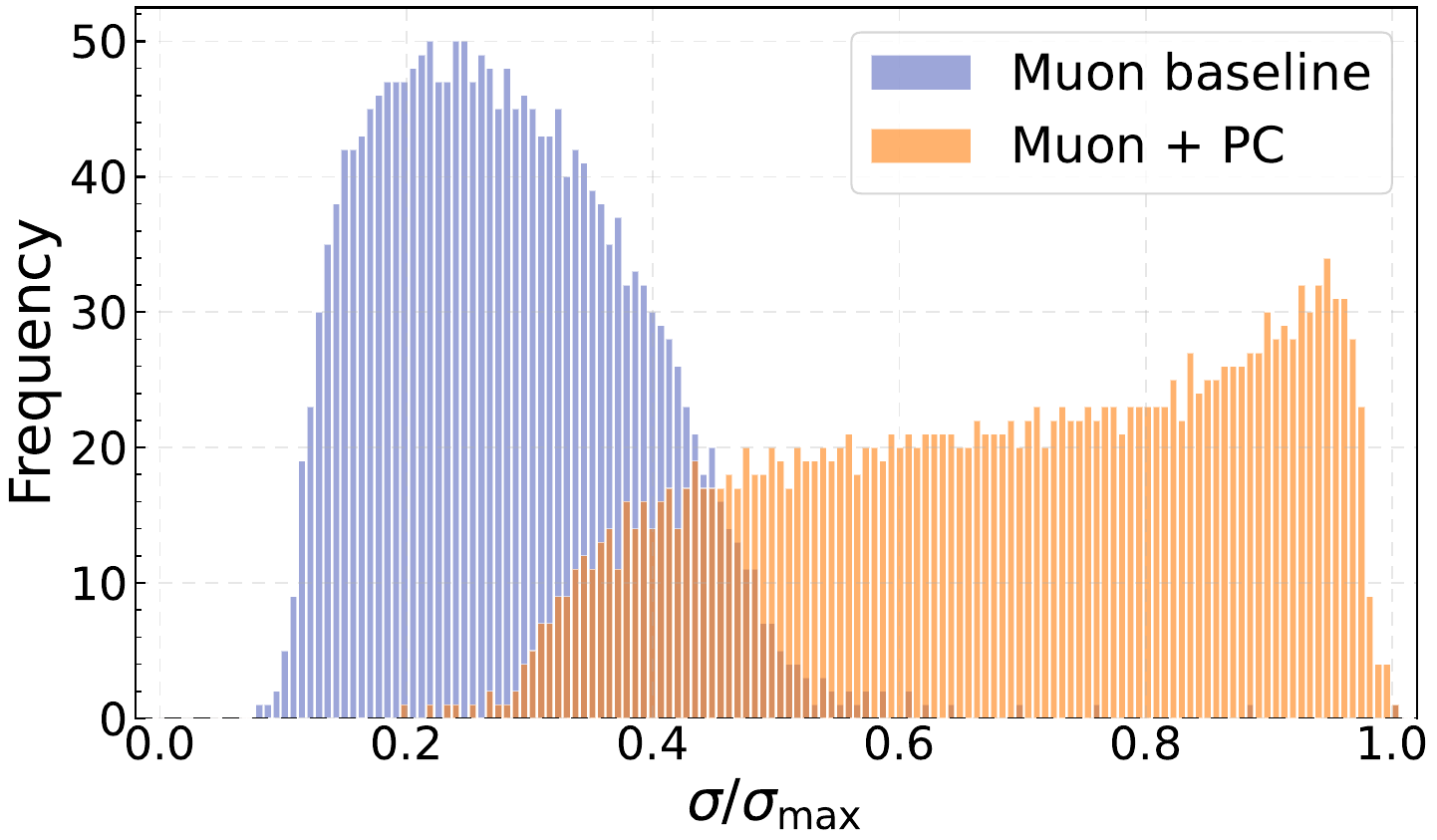}
        \caption{Layer 18: $W_{\rm gate}$}
        \label{fig:muon-spec-l18-w1}
    \end{subfigure}\hfill
    \begin{subfigure}[t]{0.24\textwidth}
        \centering
        \includegraphics[width=\linewidth]{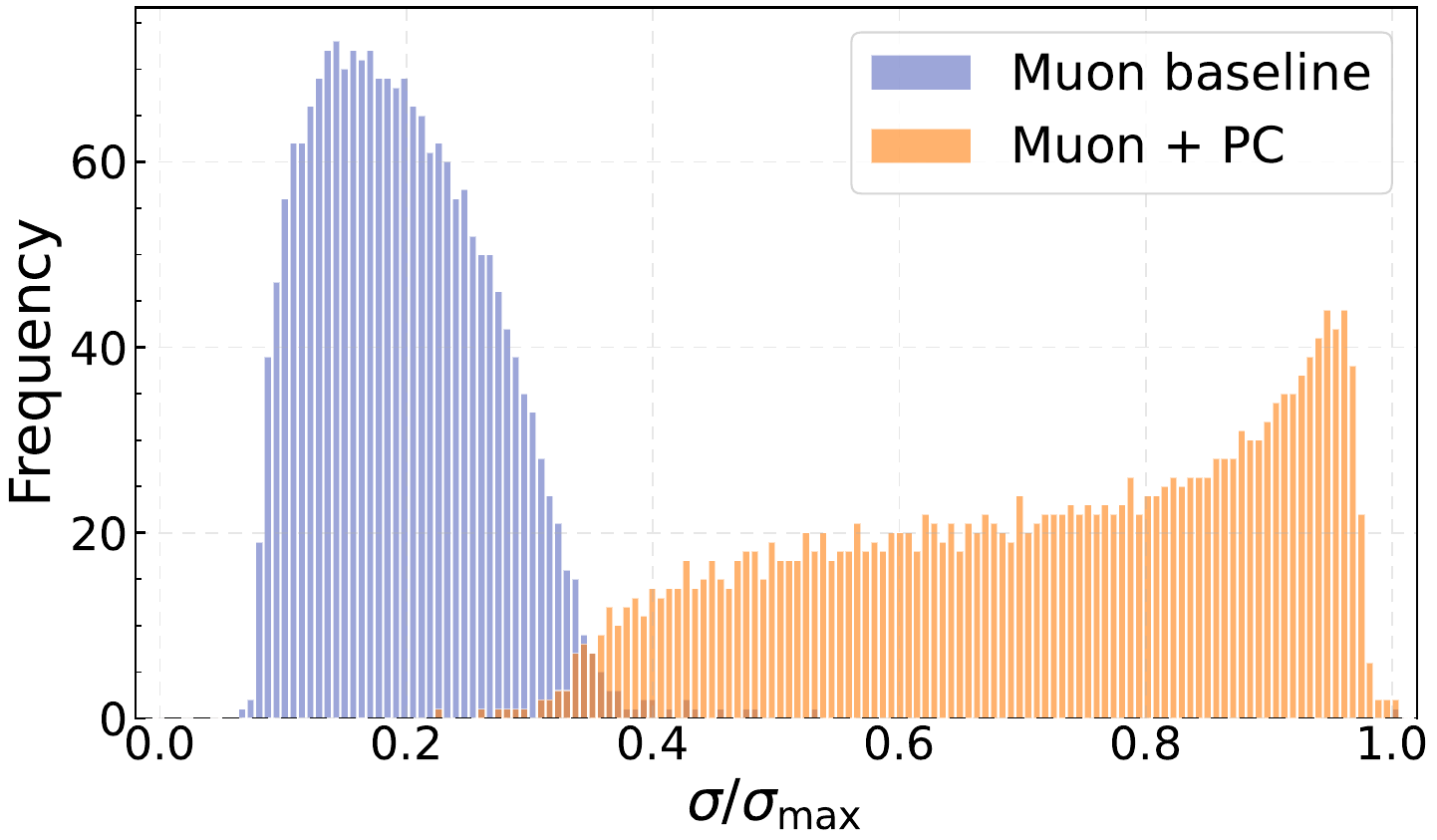}
        \caption{Layer 18: $W_{\rm up}$}
        \label{fig:muon-spec-l18-w3}
    \end{subfigure}\hfill
    \begin{subfigure}[t]{0.24\textwidth}
        \centering
        \includegraphics[width=\linewidth]{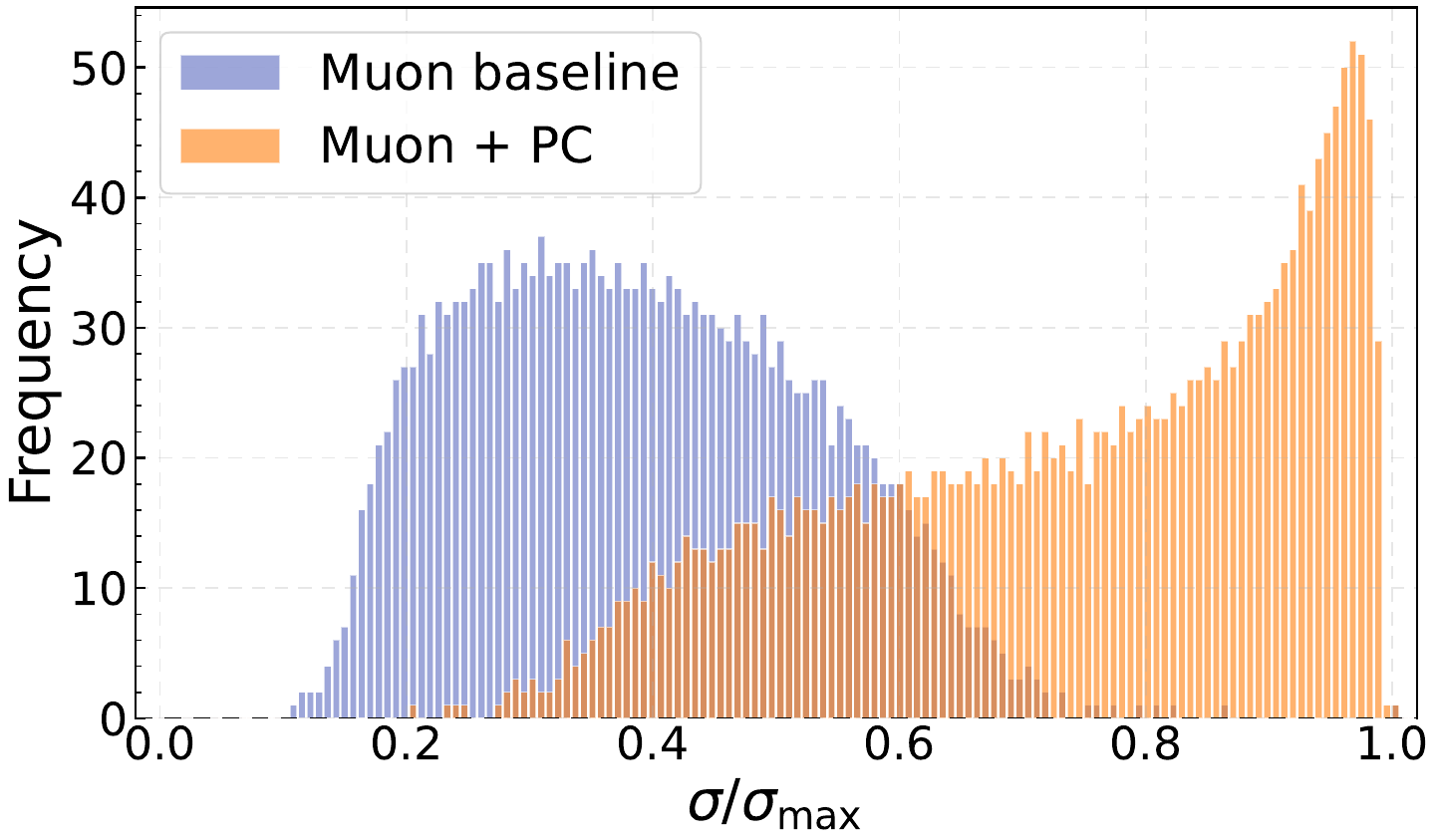}
        \caption{Layer 18: $W_{\rm down}$}
        \label{fig:muon-spec-l18-w2}
    \end{subfigure}

    \caption{\textbf{Singular-value histograms at the final checkpoint (Muon, Llama-1B).}
    Setup as in Figure~\ref{fig:pc-sv-hist-adamw} (three depths, all \texttt{PC\_blocks}, per-matrix max-normalization to $[0,1]$).
    Across depths and blocks, PC shifts the bulk of the spectrum upward toward $\sigma_{\max}$, narrowing the relative spectral spread.}
    \label{fig:pc-sv-hist-muon}
\end{figure}

\subsection{Supplementary Discussions}
\label{subapp:supp_disc}
In this section, we provide further discussions on the similarities between Muon and PC in terms of their computational templates and their approaches to spectral conditioning.

\paragraph{A shared computational template.}
Despite these conceptual differences, both methods share a computational template that makes them GPU-friendly.
Concretely, both can be expressed using matrix-polynomial transforms of the form $\mathcal{T}(X)=q(XX^\top)X$, i.e., a specific polynomial in $XX^\top$ applied to $X$.
This structure reduces to a small number of matrix multiplications and avoids explicit spectral decompositions, while still inducing a controllable map on singular values.

Notably, Muon employs a \emph{25th}-degree polynomial (comprising 5 iterations of degree-5 polynomials) to precondition the matrix, while PC layer uses a \emph{9th}-degree polynomial ($\texttt{pc\_level}=4$) under AdamW and a \emph{5th}-degree polynomial ($\texttt{pc\_level}=2$) when combined with Muon. A detailed FLOPs analysis of the PC layer overhead is provided in Appendix~\ref{subapp:pc_cost}.

\paragraph{Implicit vs.\ explicit control of the \emph{weight} spectrum.}
Although Muon acts on \emph{updates} rather than weights, it can still be viewed as exerting an \emph{implicit} form of weight-spectrum control through the optimizer dynamics.
Recent work suggests that Muon with decoupled weight decay corresponds to an implicit spectral constraint (notably affecting the top singular value) \citep{chen2025muon}, and empirical studies report that Muon-trained weights tend to be more isotropic, with higher effective rank / spectral entropy \citep{wang2025muon,liu2025muon}.
In contrast, PC provides \emph{explicit} weight-spectrum control by directly shaping the singular values of selected weight blocks toward a healthy range throughout training. 

Since Muon already exerts implicit spectral control on the updates, the \emph{additional} spectrum shaping required from PC under Muon might be smaller. As a result, a lower \texttt{pc\_level} (i.e., a lower polynomial degree) already suffices to shape the spectrum without overly interfering with the model's expressiveness under Muon. Higher \texttt{pc\_level} values, on the other hand, may risk compromising the model's ability to represent complex patterns by over-constraining the weight spectrum. This explains why the optimal \texttt{pc\_level} for Muon is smaller than that for AdamW.

 \

\end{document}